\documentclass{article} % For LaTeX2e
\usepackage{microtype}
\usepackage{graphicx}
\usepackage{subfigure}
\usepackage{booktabs} % for professional tables

\usepackage{hyperref}

\usepackage{mathrsfs}  

\usepackage[accepted]{icml2023}

\usepackage{amsmath}
\usepackage{amssymb}
\usepackage{mathtools}
\usepackage{amsthm}

\usepackage[capitalize,noabbrev]{cleveref}

\theoremstyle{plain}
\newtheorem{theorem}{Theorem}[section]

\newtheorem{lemma}[theorem]{Lemma}
\newtheorem{corollary}[theorem]{Corollary}
\theoremstyle{definition}

\theoremstyle{remark}

\usepackage[textsize=tiny]{todonotes}
\usepackage{wrapfig}

\usepackage{algorithm}
\usepackage{algorithmic}
\usepackage{stfloats}
\usepackage[capitalize,noabbrev]{cleveref}
\usepackage[T1]{fontenc}    % use 8-bit T1 fonts
\usepackage{url}            % simple URL typesetting
\usepackage{booktabs}       % professional-quality tables
\usepackage{amsfonts}       % blackboard math symbols
\usepackage{nicefrac}       % compact symbols for 1/2, etc.
 \usepackage{microtype}      % microtypography
\usepackage{xcolor}         % colors
\usepackage{enumitem}
\usepackage{cleveref}
\usepackage{comment}
\usepackage{dsfont}
\usepackage{bbm}

\input{def.set}

\icmltitlerunning{Are Neurons Actually Collapsed? On the Fine-Grained Structure in Neural Representations}

\begin{document}

\twocolumn[
\icmltitle{Are Neurons Actually Collapsed? On the Fine-Grained Structure in\\Neural Representations}

% It is OKAY to include author information, even for blind
% submissions: the style file will automatically remove it for you
% unless you've provided the [accepted] option to the icml2023
% package.

% List of affiliations: The first argument should be a (short)
% identifier you will use later to specify author affiliations
% Academic affiliations should list Department, University, City, Region, Country
% Industry affiliations should list Company, City, Region, Country

% You can specify symbols, otherwise they are numbered in order.
% Ideally, you should not use this facility. Affiliations will be numbered
% in order of appearance and this is the preferred way.
% \icmlsetsymbol{equal}{*}

\begin{icmlauthorlist}
\icmlauthor{Yongyi Yang}{umich}
\icmlauthor{Jacob Steinhardt}{ucb}
\icmlauthor{Wei Hu}{umich}
\end{icmlauthorlist}
\icmlaffiliation{umich}{University of Michigan, Ann Arbor, Michigan, USA}
\icmlaffiliation{ucb}{UC Berkeley, Berkeley, California, USA}

\icmlcorrespondingauthor{Wei Hu}{vvh@umich.edu}

% You may provide any keywords that you
% find helpful for describing your paper; these are used to populate
% the "keywords" metadata in the PDF but will not be shown in the document
\icmlkeywords{Machine Learning, ICML}

\vskip 0.3in
]

% this must go after the closing bracket ] following \twocolumn[ ...

% This command actually creates the footnote in the first column
% listing the affiliations and the copyright notice.
% The command takes one argument, which is text to display at the start of the footnote.
% The \icmlEqualContribution command is standard text for equal contribution.
% Remove it (just {}) if you do not need this facility.

%\printAffiliationsAndNotice{}  % leave blank if no need to mention equal contribution
\printAffiliationsAndNotice{} % otherwise use the standard text.

\begin{abstract}
Recent work has observed an intriguing ``Neural Collapse'' phenomenon in well-trained neural networks, where the last-layer representations of training samples with the same label collapse into each other. This appears to suggest that the last-layer representations are completely determined by the labels, and do not depend on the intrinsic structure of input distribution. We provide evidence that this is not a complete description, and that the apparent collapse hides important fine-grained structure in the representations.
Specifically, even when representations apparently collapse, the small amount of remaining variation can still faithfully and accurately captures the intrinsic structure of input distribution.
% within that label and is enough there is a mismatch between the inputs and the labels (e.g. when a coarser or finer labeling is provided), the learned last-layer representations often naturally exhibit a clustered structure according to the input distribution. %, and such structure can even persist all the way through the end of training when Neural Collapse has happened as driven by the labels. 
As an example, if we train on CIFAR-10 using only 5 coarse-grained labels (by combining two classes into one super-class) until convergence, we can reconstruct the original 10-class labels from the learned representations via unsupervised clustering. The reconstructed labels achieve $93\%$ accuracy on the CIFAR-10 test set, nearly matching the normal CIFAR-10 accuracy for the same architecture.
We also provide an initial theoretical result showing the fine-grained representation structure in a simplified synthetic setting.
Our results show concretely how the structure of input data can play a significant role in determining the fine-grained structure of neural representations, going beyond what Neural Collapse predicts.
\end{abstract}

\section{Introduction}

\begin{comment}
intro:
\begin{itemize}
    \item DL's success is attributed to representation learning. Understanding the properties of neural representations is important.
    \item One recent line of work discovered NC
    \item Other work found that representations for subclasses tend to be separable
    \item These two observations seem contradictory. How to reconcile?
    \item A plausible explanation is that NC is in the terminal phase, while the clustering structure appears in the middle of training before NC happens
    \item We find that they can both happen, there's no contradiction, there's microstructure even when NC happens

    \item does not say anything about structure within class
    \item existing theoretical explanations of NC are based on Unconstrained Feature Model, which completely ignores the input
    \item we provide convincing empirical evidence that input data play an important role...
    \item describe experimental setup: coarse CIFAR-10/100
    \item we find that representations preserve structure of the original labels. show picture
    \item cluster and train
    \item additional experiment on fine CIFAR
\end{itemize}
\end{comment}

Much of the success of deep neural networks has, arguably, been attributed to their ability to learn useful \emph{representations}, or \emph{features}, of the data~\citep{rumelhart1985learning}. Although neural networks are often trained to optimize a single objective function with no explicit requirements on the inner representations, there is ample evidence suggesting that these learned representations contain rich information about the input data~\citep{levy-goldberg-2014-linguistic,olah2017feature}. As a result, formally characterizing and understanding the structural properties of neural representations is of great theoretical and practical interest, and can provide insights on how deep learning works and how to make better use of these representations. %\wh{cite something?}

One intriguing phenomenon recently discovered by~\citet{neuralcollapse-original} is \emph{Neural Collapse}, which identifies structural properties of last-layer representations during the terminal phase of training (i.e. after zero training error is reached). The simplest of these properties is that the last-layer representations for training samples with the same label collapse into a single point, which is referred to as ``variability collapse (NC1).''
This is surprising, since the collapsed structure is not necessary to achieve small training or test error, yet it arises consistently in standard architectures trained on standard classification datasets.

A series of recent papers were able to theoretically explain Neural Collapse under a simplified model called the \emph{unconstrained feature model} or \emph{layer-peeled model} (see \Cref{sec:related} for a list of references). In this model, the last-layer representation of each training sample is treated as a free optimization variable and therefore the training loss essentially has the form of a matrix factorization. Under a variety of different setups, it was proved that the solution to this simplified problem should satisfy Neural Collapse. Although Neural Collapse is relatively well understood in this simplified model, this model completely ignores the role of the input data because the loss function is independent of the input data. Conceptually, this suggests that \textbf{Neural Collapse is only determined by the labels} and may happen regardless of the input data distribution. \citet{zhu2021geometric} provided further empirical support of this claim via a random labeling experiment.
 
On the other hand, it is conceivable that the \textbf{intrinsic structure of the input distribution \emph{should} play a role} in determining the structure of neural net representations. For example, if a class contains a heterogeneous set of input data (such as different subclasses), it is possible that their heterogeneity is also respected in their feature representations~\citep{sohoni2020no}.
%For instance, \citet{sohoni2020no} observed that, when a class contains multiple finer-grained subclasses, the last-layer representations of different subclasses are often separated into different clusters, even though they are given the same label in training. 
However, this appears to contradict Neural Collapse, because Neural Collapse would predict that all the representations collapse into each other as long as they have the same class label. This dilemma motivates us to study the following main question in this paper:
\begin{center}
    \emph{How can we reconcile the roles of the \textbf{intrinsic structure of input distribution} vs. the \textbf{explicit structure of the labels} in determining the last-layer representations in neural networks?}
\end{center}

\begin{figure}[t]
\centering
\includegraphics[width=\linewidth]{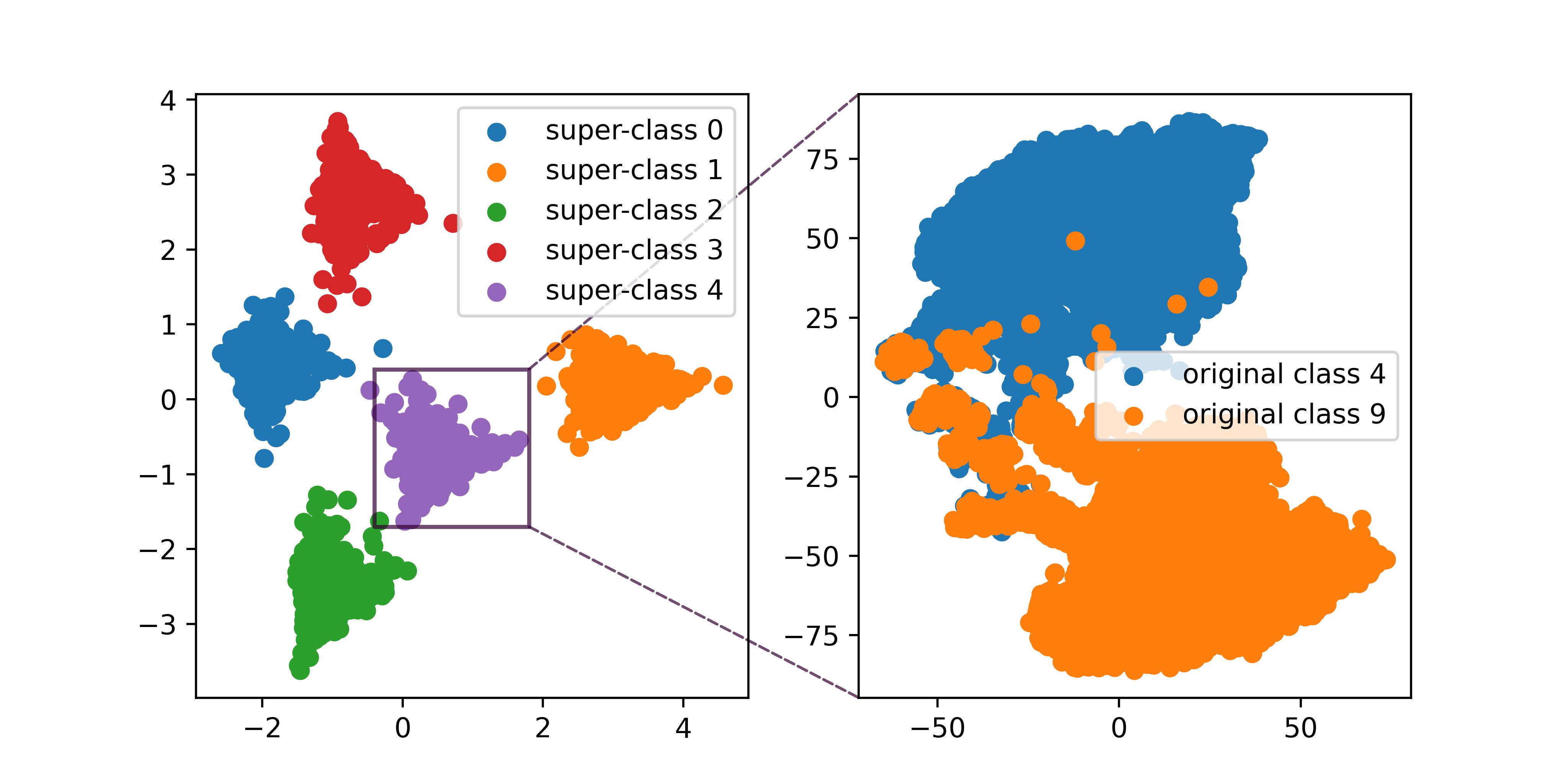}
    \caption{Fine-grained clustering structure of the last-layer representations of ResNet-18 trained on Coarse CIFAR-10 (5 super-classes). Left figure: PCA visualization for all training samples. Right figure: t-SNE visualization for all training samples in super-class 4 (which consists of original classes 4 and 9).}
    \label{fig:intro-visualize}
\end{figure}

\paragraph{Our methodology and findings.}
To study the above question, we design experiments to manually create a \emph{mismatch} between the intrinsic structure of the input distribution and the explicit labels provided for training in standard classification datasets and measure how the last-layer representations behave in response to our interventions. This allows us to isolate the effect of the input distribution from the effect of labels.
As an illustrative example, for the CIFAR-10 dataset (a 10-class classification task), we alter its labels in two different ways, resulting in a coarsely-labeled and a finely-labeled version:
\begin{itemize}
\item Coarse CIFAR-10: combine every two class labels into one and obtain a 5-class task (see \cref{fig:dataset_info} for an illustration);
\item Fine CIFAR-10: split every class label randomly into two labels and obtain a 20-class task.
\end{itemize}
We train standard network architectures (e.g. ResNet, DenseNet) using SGD on these altered datasets. Our main findings are summarized below.

First, both the intrinsic structure of the input distribution and the explicit labels provided in training clearly affect the structure of the last-layer representations. The effect of input distribution emerges earlier in training, while the effect of labels appears at a later stage.
For example, for both Coarse CIFAR-10 and Fine CIFAR-10, at some point the representations naturally form 10 clusters according to the original CIFAR-10 labels (which comes from the intrinsic input structure), even though 5 or 20 different labels are provided for training. Later in training (after 100\% training accuracy is reached), the representations collapse into 5 or 20 clusters driven by the explicit labels provided, as predicted by Neural Collapse. %\yyy{double check this}

Second, even after Neural Collapse has occurred according to the explicit label information, the seemingly collapsed representations corresponding to each label can still exhibit \emph{fine-grained structures} determined by the input distribution. As an illustration, \Cref{fig:intro-visualize} visualizes the representations from the last epoch of training a ResNet-18 on Coarse CIFAR-10. While globally there are 5 separated clusters as predicted by Neural Collapse, if we zoom in on each cluster, it clearly consists of two subclusters which correspond to the original CIFAR-10 classes. We also find that this phenomenon persists even after a very long training period (e.g. 1,000 epochs), indicating that the effect of input distribution is not destroyed by that of the labels, at least not within a normal training budget.
%\wh{change to fine grained and mention fine?}

To further validate our finding that significant input information is present in the last-layer representations despite Neural Collapse, we perform a simple \emph{Cluster-and-Linear-Probe (CLP)} procedure on the representations from ResNet-18 trained on Coarse CIFAR-10, in which we use an unsupervised clustering method to reconstruct the original labels, and then train a linear classifier on top of these representations using the reconstructed labels. We find that CLP can achieve $>93\%$ accuracy on the original CIFAR-10 test set, matching the standard accuracy of ResNet-18, even though only 5 coarse labels are provided the entire time.
%\wh{add a concluding sentence?}

\paragraph{Theoretical result in a synthetic setting.}
To complement our findings, we provide a theoretical explanation of the fine-grained representation structure in a simplified synthetic setting --- a one-hidden-layer neural network trained on coarsely labeled Gaussian mixture data. We prove that such a network trained by gradient descent produces separable hidden-layer representations for different clusters even if they are given the same label for training.

\paragraph{Takeaway.}
While Neural Collapse is an intriguing phenomenon that consistently happens, we provide concrete evidence showing that it is not the most comprehensive description of the behavior of last-layer representations in practice, as it fails to capture the possible fine-grained properties determined by the intrinsic structure of the input distribution.

\section{Related Work} \label{sec:related}

The Neural Collapse phenomenon was originally discovered by \citet{neuralcollapse-original}, and has led to a series of further investigations.

A number of papers~\citet{fang2021exploring, lu2020neural, wojtowytsch2020emergence, mixon2022neural, zhu2021geometric, ji2021unconstrained,han2021neural, zhou2022optimization, tirer2022extended, yaras2022neural} studied a simplified \emph{unconstrained feature model}, also known as \emph{layer-peeled model}, and showed that Neural Collapse provably happens under a variety of settings. This model treats the last-layer representations of all training samples as free optimization variables. By doing this, the loss function no longer depends on the input data, and therefore this line of work is unable to capture any effect of the input distribution on the structure of the representations.
\citet{ergen2021revealing, tirer2022extended, weinan2022emergence} considered more complicated models but still did not incorporate the role of the input distribution.

\citet{hui2022limitations} studied the connection of Neural Collapse to generalization and concluded that Neural Collapse occurs only on the training set, not on the test set.
\citet{galanti2021role} found that Neural Collapse does generalize to test samples as well as new classes, and used this observation to study transfer learning and few-shot learning.

\citet{sohoni2020no} observed that the last-layer representations of different subclasses within the same class are often separated into different clusters, and used this observation to design an algorithm for improving group robustness. The fine-grained representation phenomenon we observe is in a qualitatively different regime from that of \citet{sohoni2020no}. First of all, we focus on the Neural Collapse regime and find that fine-grained representation structure can co-exist with Neural Collapse. Furthermore, \citet{sohoni2020no} looked at settings in which different subclasses have very different accuracies and explicitly attributed the representation separability phenomenon to this performance difference. On the other hand, we find that representation separability can happen even when there is no performance gap between different subclasses.

%First, the fine-grained representation phenomenon we observe is qualitatively different from that of Sohoni et al. (2020). Sohoni et al. considered settings where different sub-classes in a super-class have a large accuracy gap (e.g. when there are imbalanced sub-classes), and they explicitly attributed the reason behind the fine-grained representation phenomenon to the accuracy gap. For example, Section 3.2 in their paper states "the larger the accuracy gap, the more separable the subclasses are." On the other hand, we find that an accuracy gap is not at all necessary for the fine-grained structure to emerge -- it can emerge even in benign and balanced datasets like CIFAR-10. Concretely, we measure the per-class test accuracy of ResNet-18 trained on Coarse CIFAR-10 under config #1 (see Section 4 of our paper) and find that most of the classes have similar accuracies:

\section{Preliminaries and Setup}

%This section introduces the notation used throughout this paper, as well as a more detailed explanation of Neural Collapse and a clarification of the experiment setup.

Consider a classification dataset $\mathcal D = \{(\bx_k , y_k)\}_{k=1}^n$, where $(\bx_k , y_k) \in \mathbb R^{d'} \times [C]$ is a pair of input features and label, $n$ is the number of samples, $d'$ is the input dimension, and $C$ is the number of classes. Here $[C] = \{1,\ldots,C\}$.

For a given neural network, we denote its last-layer representations corresponding to the dataset $\mathcal D$ by $H \in \mathbb R^{n \times d}$, i.e. the hidden representation before the final linear transformation, where $d$ is the last-layer dimensionality. For an original class $c \in [C]$, we denote the number of samples in class $c$ by $n_c$, and the last-layer representation of $k$-th sample in class $c$ by ${\boldsymbol{h}}^{(c)}_k$.

\subsection{Preliminaries of Neural Collapse}\label{sec:preliminaries-of-NC}

Neural Collapse~\citep{neuralcollapse-original} characterizes 4 phenomena, named NC1-NC4. Here we introduce NC1 and NC2 which concern the structure of the last-layer representations.

NC1, or variability collapse, asserts that the variance of last-layer representations of samples within the same class vanishes as training proceeds. Formally, it can be measured by $\mathsf{NC_1} = \frac{1}{C}\mathrm{Tr} \left({\Sigma}_W {\Sigma}_B^{\dagger}\right)$ \citep{neuralcollapse-original,zhu2021geometric}, which is observed to go to $0$. Here ${{\Sigma}}_W$ and ${{\Sigma}}_B$ are defined as\begin{equation}{\Sigma}_W  = \frac{1}{C} \sum_{c \in [C]} \frac{1}{n_c} \sum_{i=1}^{n_c} \left(\bh_{i}^{(c)} - \boldsymbol{\mu}_c\right)\left(\bh_{i}^{(c)} - \boldsymbol{\mu}_c\right)^\top  \label{eq:NC1}\end{equation}and\begin{equation}
{\Sigma}_B = \frac{1}{C} \sum_{c\in[C]} \left({\boldsymbol{\mu}}_c - {\boldsymbol{\mu}}_G\right)\left({\boldsymbol{\mu}}_c - {\boldsymbol{\mu}}_G\right)^\top,
\end{equation}
where ${\boldsymbol{\mu}}_c=\frac{1}{n_c}\sum_{k=1}^{n_c}{\boldsymbol{h}}_k^{(c)}$ are the class means and ${\boldsymbol{\mu}}_G = \frac{1}{n}\sum_{c=1}^C \sum_{k=1}^{n_c}{\boldsymbol{h}}_k^{(c)}$ is the global mean.

NC2 predicts that the class means form a special structure, namely, their normalized covariance converges to the Simplex Equiangular Tight Frame (ETF). This can be characterize by\begin{equation}\mathsf{NC_2} \overset{\text{def}}{=} \left\| \frac{MM^\top}{\left\|MM^\top\right\|_\mathcal F} - \frac{1}{\sqrt{C-1}}\left(I - \frac{1}{C}{\boldsymbol{1}}_C {\boldsymbol{1}}_C^\top \right)\right\|_\mathcal F \to 0\end{equation}
during training, where $M \in \R^{C\times d}$ is the stack of centralized class-means, whose $c$-th row is ${\boldsymbol{\mu}}_c - {\boldsymbol{\mu}}_G$, and $\mathbbm 1_C \in \mathbbm R^C$ is the all-one vector, and $I$ is the identity matrix.

\subsection{Experiment Setup}\label{sec:experiment-setup}

In our experiment, we explore the role of input distribution and labels through assigning coarser or finer labels to each sample, and then explore the structure of last-layer representation of a model trained on the dataset with coarse or fine labels and see to what extent the information of original labels are preserved.

The coarse labels are created in the following way. Choose a number $\tilde C$ divides $C$, and create coarse labels by \begin{equation}
\tilde y_k = y_k \mathop{ \mathrm{ mod } } \tilde C,
\end{equation}
which merges the classes whose indices have the same remainder w.r.t. $\tilde C$ and thus creates $\tilde C$ super-classes. Since the original indices of classes generally have no special meanings, this process should act similarly to randomly merging classes.\footnote{We adopt this deterministic process for simplicity and reproducibility. However, we do provide additional results with random merging in \cref{sec:complete-rand-cifar-10}.} We say the samples with the same coarse label belong to the same super-class, and call the dataset $\tilde {\mathcal D} = \{(x_k , \tilde y_k)\}_{k=1}^n$ the \emph{coarse dataset}, which we use to train the model. Figure \ref{fig:dataset_info} provides an illustration of the coarse labels on CIFAR-10 with $\tilde C = 5$, which we call Coarse CIFAR-10.

\begin{figure*}[htbp]
    \centering
    \includegraphics[width=\linewidth]{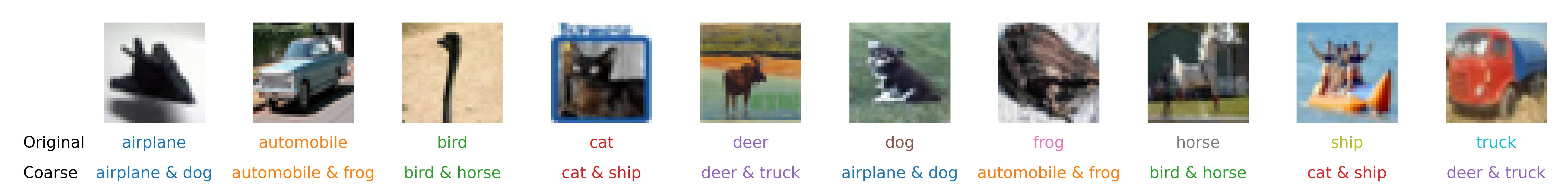}
    \caption{Illustration of Coarse CIFAR-10.}
    \label{fig:dataset_info}
\end{figure*}

To create fine labels, we randomly split each class into two sub-classes. Specifically, the fine labels are created by\begin{equation}\hat y_k = y_k + \beta C,\end{equation} 
where $\hat y_k$ is the fine label of sample $k$ and $C$ is the number of original classes and $\beta$ is a Bernoulli Variable. This process result in a dataset $\hat {\mathcal D} = \left\{\left(x_k , \hat y_k\right)\right\}_{k=1}^n$ with $2C$ classes. Same as before, we call $\hat {\mathcal D}$ the fine dataset.

\section{Exploring the Fine-Grained Representation Structure with Coarse CIFAR-10}\label{sec:experiment-on-coarse-cifar-10}

In this section, we experiment with coarsely labeled datasets, using Coarse CIFAR-10 as an illustrative example. Specifically, the model is trained on the training set of Coarse CIFAR-10 for a certain number of steps that is sufficient for the model to converge. We then take the last-layer representations of the model throughout training and explore their structure.

%The methods we use to explore the microstructures include: \begin{enumerate}\item Class Distance: Measure the average distance between all samples of two \textit{original} classes, specifically, we calculate \begin{equation} D_{i,j} = \frac{1}{n_in_j}  \sum_{u=1}^{n_i}\sum_{v=1}^{n_j} \left\|\bh^{(i)}_u - \bh^{(j)}_v\right\|_{2}^2,\end{equation}
% for all $i,j \in [C]$, and show a heatmap of the matrix $D$. 
% \item Visualization: Project the last-layer representations to a plane and directly look at the distribution of original labels among a coarse label. 
% \item Cluster-and-Train: Use a cluster method on the representation trained on coarse label to generate "fake" original labels and train a linear classifier on them, and evaluate the model on real original label. \wh{instead of ``fake'', say ``reconstructed''}
% \end{enumerate} 

In order to make an exhaustive observation, the experiments are performed using different learning rates and weight-decay rates. Specifically, we choose the initial learning rate in $\{10^{-1},10^{-2},10^{-3}\}$ (we apply a standard learning rate decay schedule) and the weight-decay rate in $\{5\times 10^{-3} , 5\times 10^{-4}, 5\times 10^{-5}\}$ and run all 9 possible combinations of them. The experiments are also conducted on multiple datasets and network architectures. Due to space limit, we defer complete results to Appendix (see \cref{sec:complete-cifar-5,sec:complete-rand-cifar-100,sec:complete-rand-cifar-10}). In this section we focus on ResNet-18 on Coarse CIFAR-10, where the original number of classes is $C = 10$ and the number of coarse labels is $\tilde C = 5$. In this section, we only report results for one group of representative hyper-parameter combinations: learning rate is $0.1$ and weight-decay rate is $5\times 10^{-4}$. Results for other hyper-parameters are presented in \cref{sec:complete-cifar-5,sec:complete-cifar-20,sec:complete-rand-cifar-100,sec:complete-rand-cifar-10,sec:densenet,sec:vgg}. %Training error reaches $0$ in this setting.

First, we verify that Neural Collapse does happen, i.e. the representations converge to $5$ clusters, and the class means form a Simplex ETF structure. Specifically, we measure $\mathsf{NC_1}$ and $\mathsf{NC_2}$ defined in \Cref{sec:preliminaries-of-NC}, with $C$ replaced by $\tilde C$ since we are calculating it on the coarsely labeled dataset. The results are shown in \Cref{fig:nc_verify}, which matches previous results in \citet{neuralcollapse-original,zhu2021geometric}, which verify Neural Collapse happens.

\begin{figure}[htbp]
    \centering
    \includegraphics[scale=0.40]{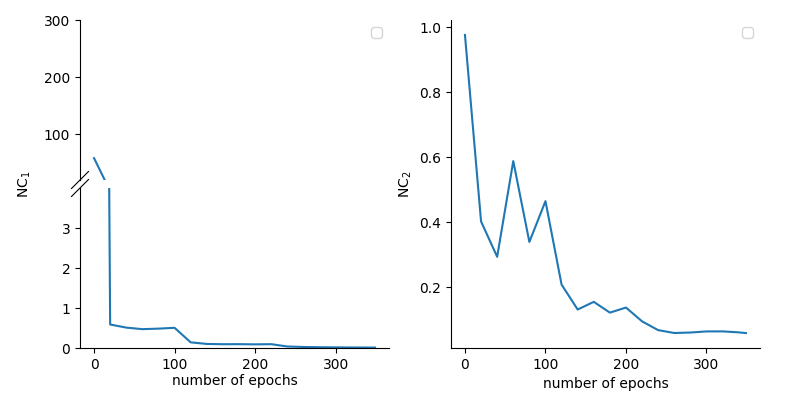}
    \caption{The value of $\mathrm{NC}_1$ and $\mathrm{NC}_2$ w.r.t. number of training epochs.}
    \label{fig:nc_verify}
\end{figure}

\subsection{Class Distance}\label{sec:class-distance}

Now, we look at the average square Euclidian distance of last-layer representations between each two \textit{original} classes. Formally, we calculate a class distance matrix $D \in \mathbb R^{C \times C}$, whose entries are \begin{equation} D_{i,j} = \frac{1}{n_in_j}  \sum_{u=1}^{n_i}\sum_{v=1}^{n_j} \left\|\bh^{(i)}_u - \bh^{(j)}_v\right\|_{2}^2,\label{eqn:mean-square-distance}\end{equation}
for all $i,j \in [C]$, where ${\boldsymbol{h}}_k^{(c)}$ represents the last-layer representation of the $k$-th sample of super-class $u$. 

\begin{figure}[htbp]
    \centering
    \includegraphics[width=\linewidth]{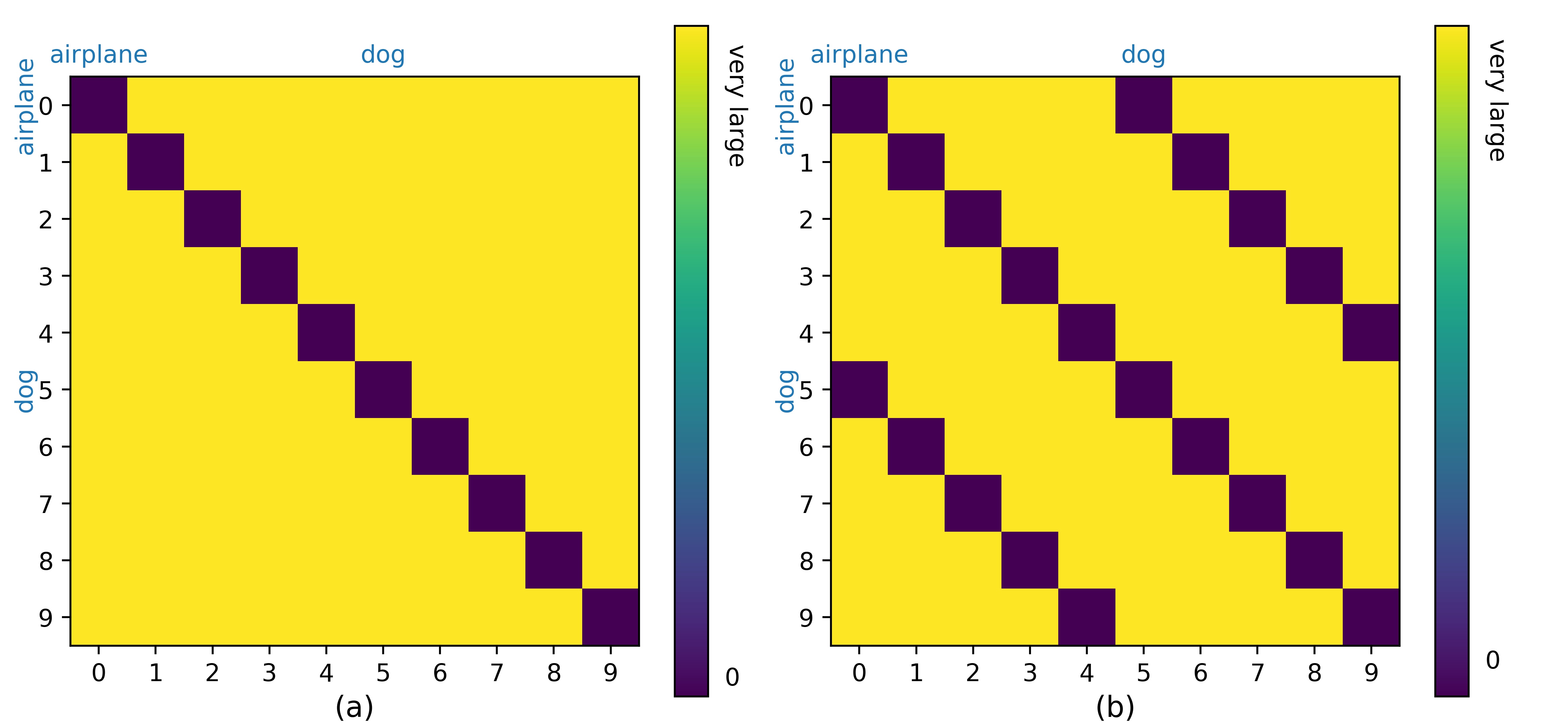}
    \caption{An illustration of predicted class distance matrix heatmaps of Coarse CIFAR-10, each row and column represents an original class. (a): If input distribution dominates the last-layer representations. (b): If Neural Collapse dominates the last-layer representations.}
    \label{fig:distance-intro}
\end{figure}

Since the model is trained on the coarse dataset, Neural Collapse asserts that for every original class pair $i,j$ in the same super-class (including the case of $i=j$), the class distance $D_{i,j}$ should be very small. In Coarse CIFAR-10, this will result in three darks lines (darker color represents lower value) in the heatmap of $D$ since each super class contains two original classes, as illustrated in \Cref{fig:distance-intro} (b). For example, in Coarse CIFAR-10 the original class "airplane" and "dog" both belong to the super class "airplane \& dog", therefore per Neural Collapse's prediction, their last-layer representations would collapse to each other, making the average square distance extremely small compared to other entries. In contrast, if the last-layer representations perfectly reflect the distribution of input, i.e. original classes, the class distance matrix should be a diagonal matrix as shown in \Cref{fig:distance-intro} (a), because the last-layer representation of samples in each original class only collapse to the class-mean of this original class.

\Cref{fig:distance-result-group-2} displays the heatmaps of class distance matrix $D$ at different stages in training, which shows that:
(i) There are indeed three dark lines that show up, but they do not show up simultaneously. In particular, the central diagonal line -- representing the samples in the same original classes -- emerges earlier in training.
(ii) Even in the final stage of training when the training error is zero, the three lines are not of the same degree of darkness. The central line is clearly darker, indicating a smaller distance within the original classes. %, and this is especially apparent in \Cref{fig:distance-result-group-2}.

Those observations suggest that the actual behaviour of the last-layer representations is between the cases predicted in \Cref{fig:distance-intro} (a) and (b): the input distribution and training labels both have an impact on the distribution of the last-layer representations, both of them can be present even after reaching zero training error for a long time, and the impact of input distribution emerges earlier in training. These observations suggest both the existence of Neural Collapse and the inadequacy of Neural Collapse to completely describe the behaviour of last-layer representations.

\begin{figure}[htbp]
\centering
    %\subfigure[epoch = 20]{\includegraphics[scale=0.20]{figures/distance-cifar-20-2-120.png}}
    \subfigure[epoch = 20]{\includegraphics[scale=0.22]{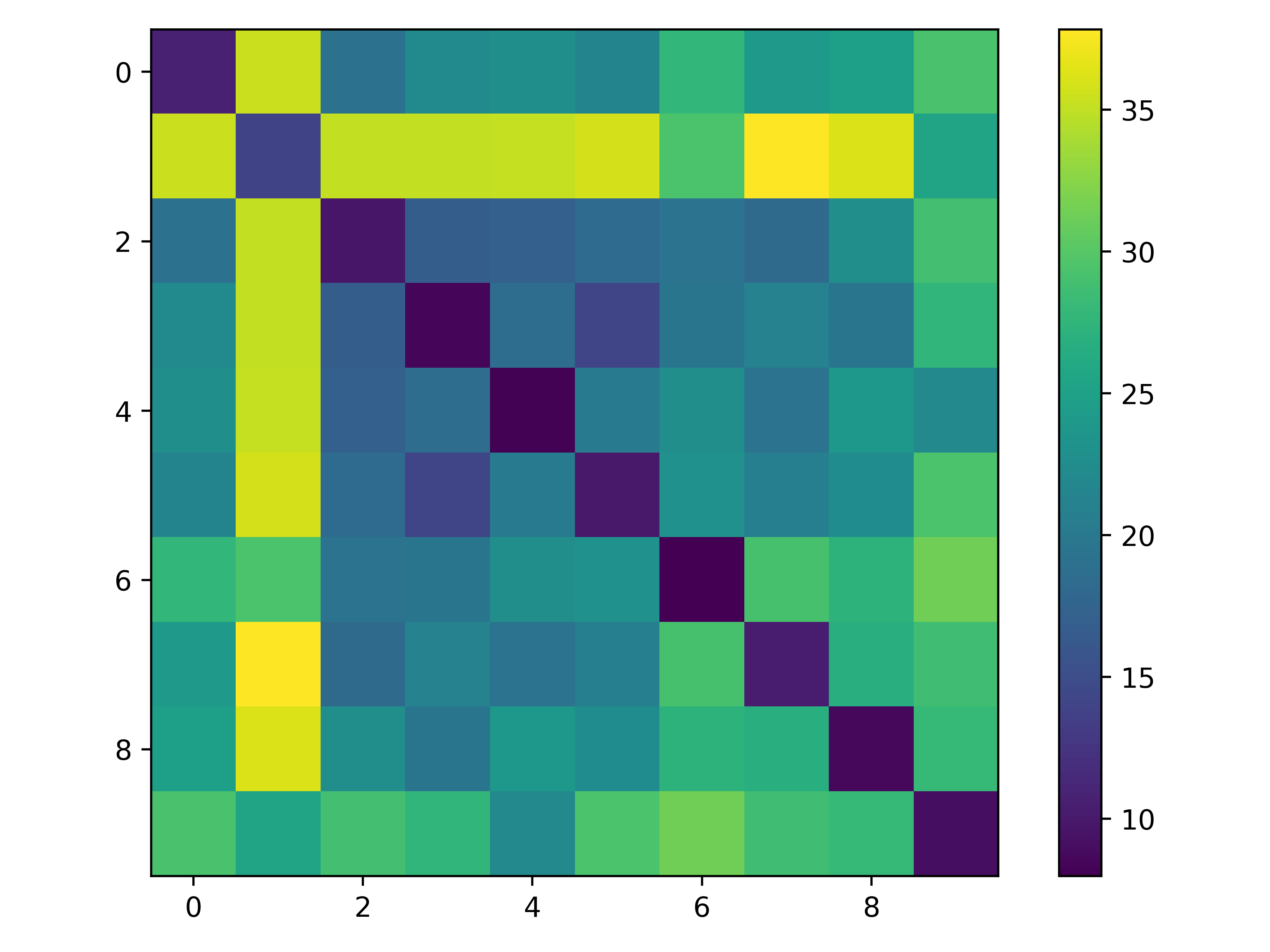}}
    \subfigure[epoch = 120]{\includegraphics[scale=0.22]{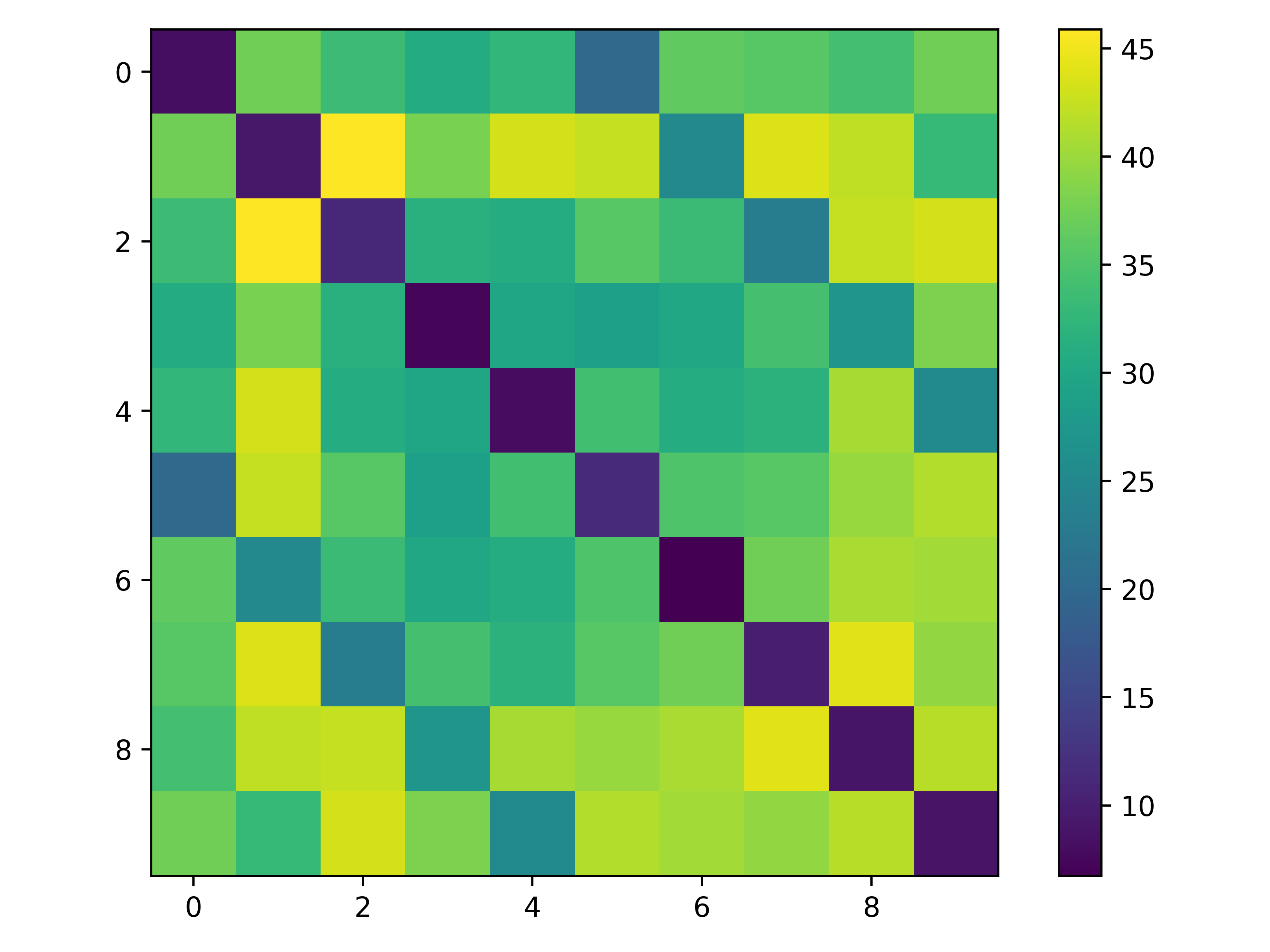}}
    \subfigure[epoch = 240]{\includegraphics[scale=0.22]{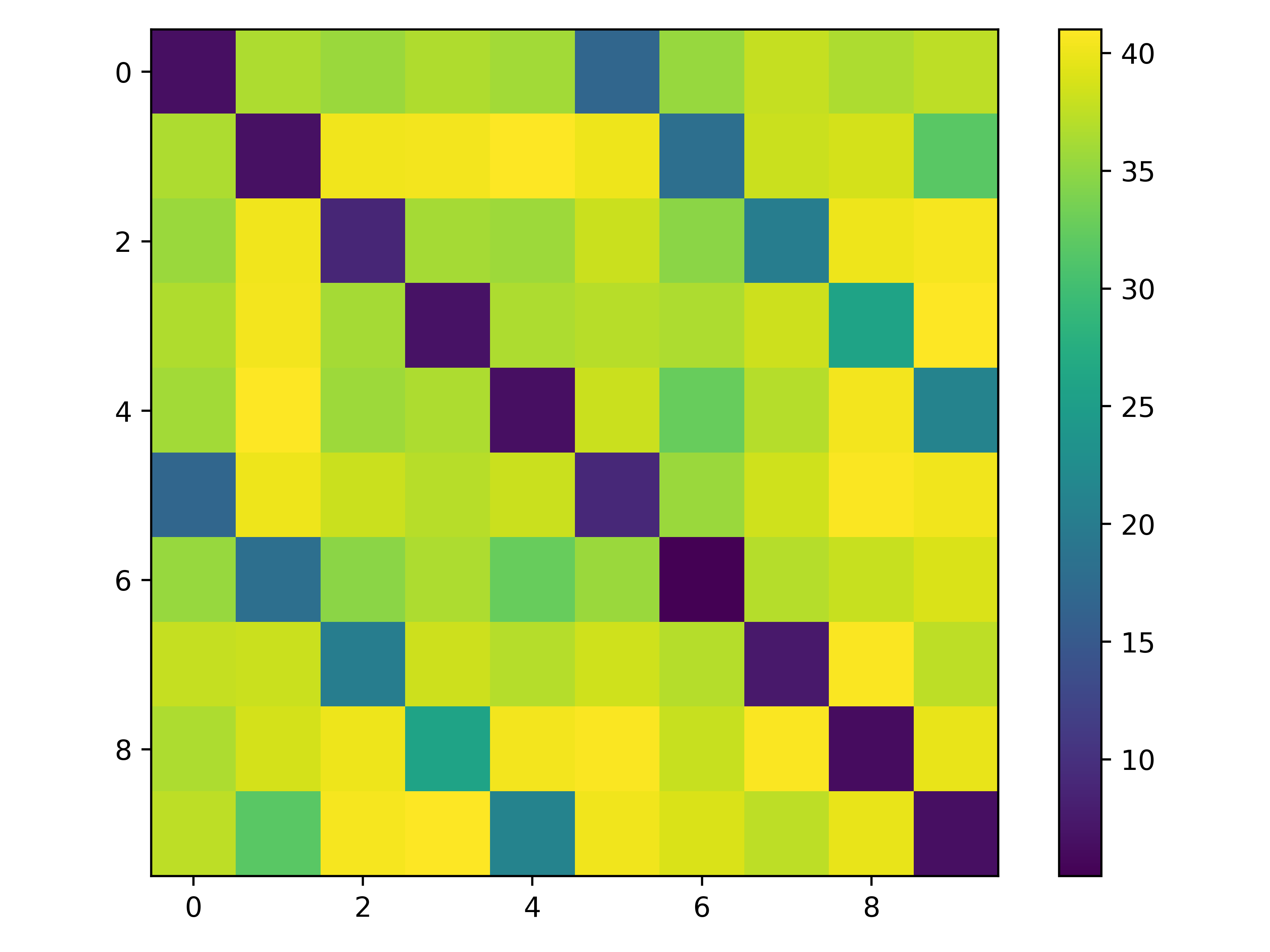}}
    \subfigure[epoch = 350]{\includegraphics[scale=0.22]{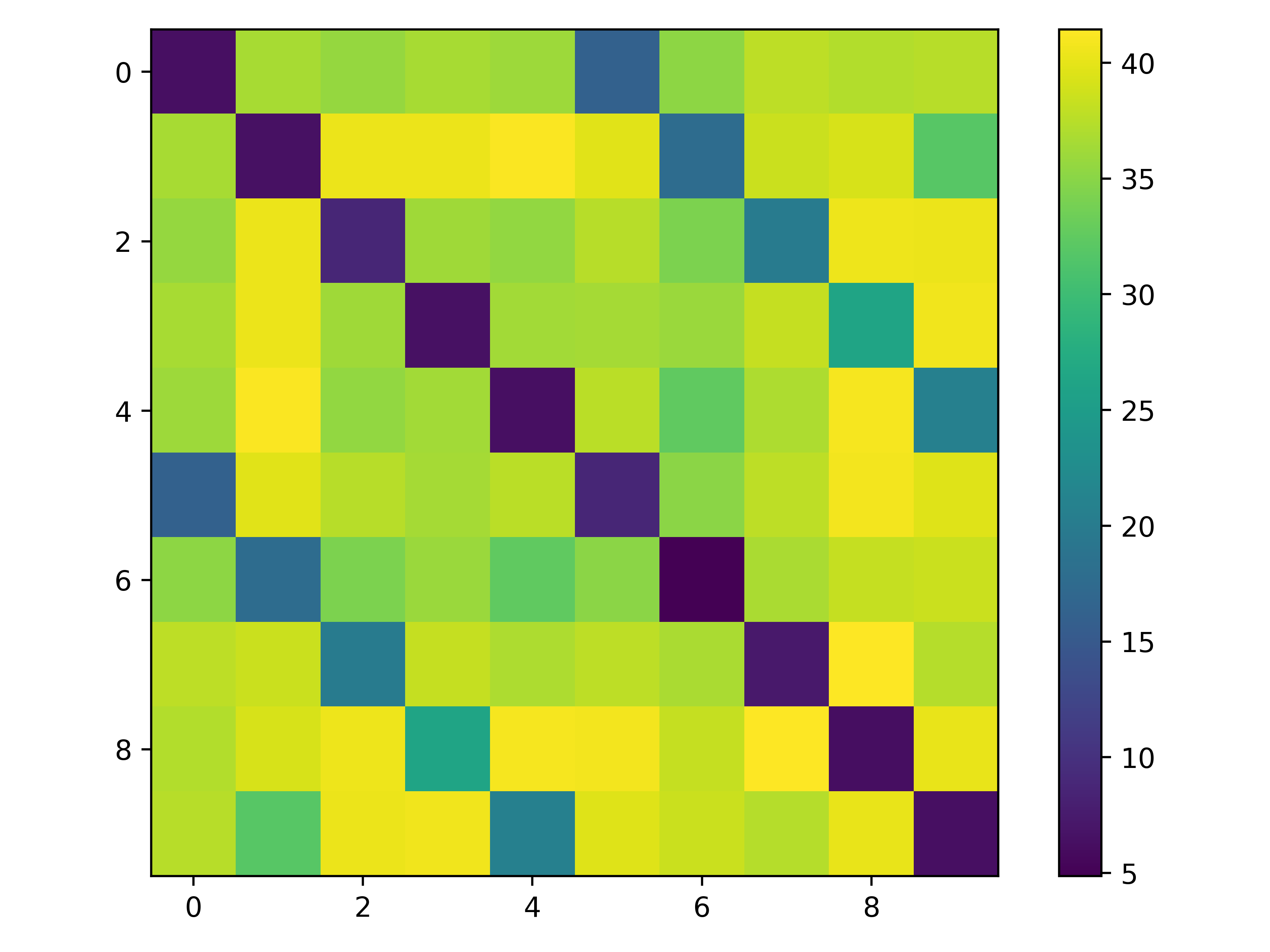}}
    \caption{The heatmap of class distance matrix.} %From left to right: number of epoch $=20,120,240$ and $350$.}
    \label{fig:distance-result-group-2}
\end{figure}

%  While we coarsen the labels by merging the labels with the same modulus of $5$, the figures are plotted under original classes. Ideally, as neural collapse predicted, the distance between $0$ and $5$ should be near $0$ while their distance to other classes should be much larger, which should act as three dark lines in the figures, which is true in most cases. However, notice the grid at top-right and middle-right, although they eventually achieves $100$\% as well, the two dark lines on the side is distinguishably lighter than the central line, which means, the average distance inside original class $0$, is much smaller than their distance to samples with original class $5$, despite they are labeled the same during training. Moreover, if we look at the result during training, the phenomenon is even clearer: the middle-top grid displays discrepancy of the three lines during training, while the discrepancy becomes undistinguishable at the end of training.

\subsection{Visualization}\label{sec:visualization}

In this subsection, we take a closer look at the last-layer representations of the model at the end of training by reducing the dimensionality of the last-layer representations to $2$ through t-SNE \citep{t-sne} and visualize them. Specifically, we visualize each super-class separately, but color the samples whose original labels are different with different colors. 

The visualization results are displayed in \Cref{fig:tsne-cifar-5-349-group2}, from which we observe a distinguishable difference between different original classes --- their representations form well-separated clusters in the 2-dimensional space. This suggests that the input distribution information, i.e. the original label information, is well preserved in the last-layer representations. %, albeit it seems they are not distinguishable through distance matrix in \Cref{sec:class-distance}.

\begin{figure*}[htbp]
    \centering
    \includegraphics[width=\linewidth]{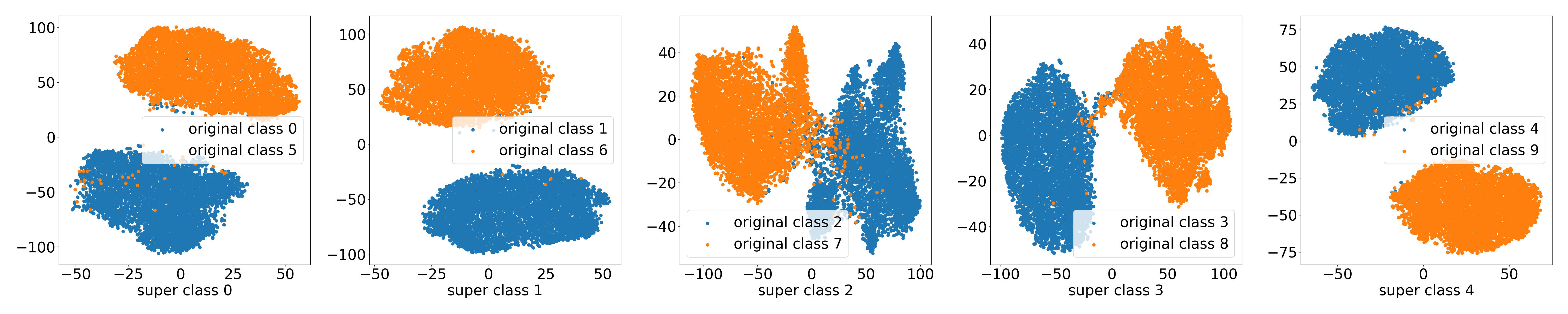}
    \caption{The t-SNE visualization of the last-layer representations of ResNet-18 trained on Coarse CIFAR-10. Each plot corresponds to a super-class.}
    \label{fig:tsne-cifar-5-349-group2}
\end{figure*}

\paragraph{Training extremely long.} In order to explore if the fine-grained structures are still preserved even after a extremely long time of training, we further train the model with 1,000 epochs. The heatmap of the distance matrix is presented in \Cref{fig:cifar-10-999-distance}. We also produce the t-SNE visualizations, but only include the result of the first super-class in \Cref{fig:tsne-cifar-10-999-1} due to space limitation.

\begin{figure}[htbp]
\centering
\begin{minipage}{0.48\linewidth}
\centering
    \includegraphics[scale=0.20]{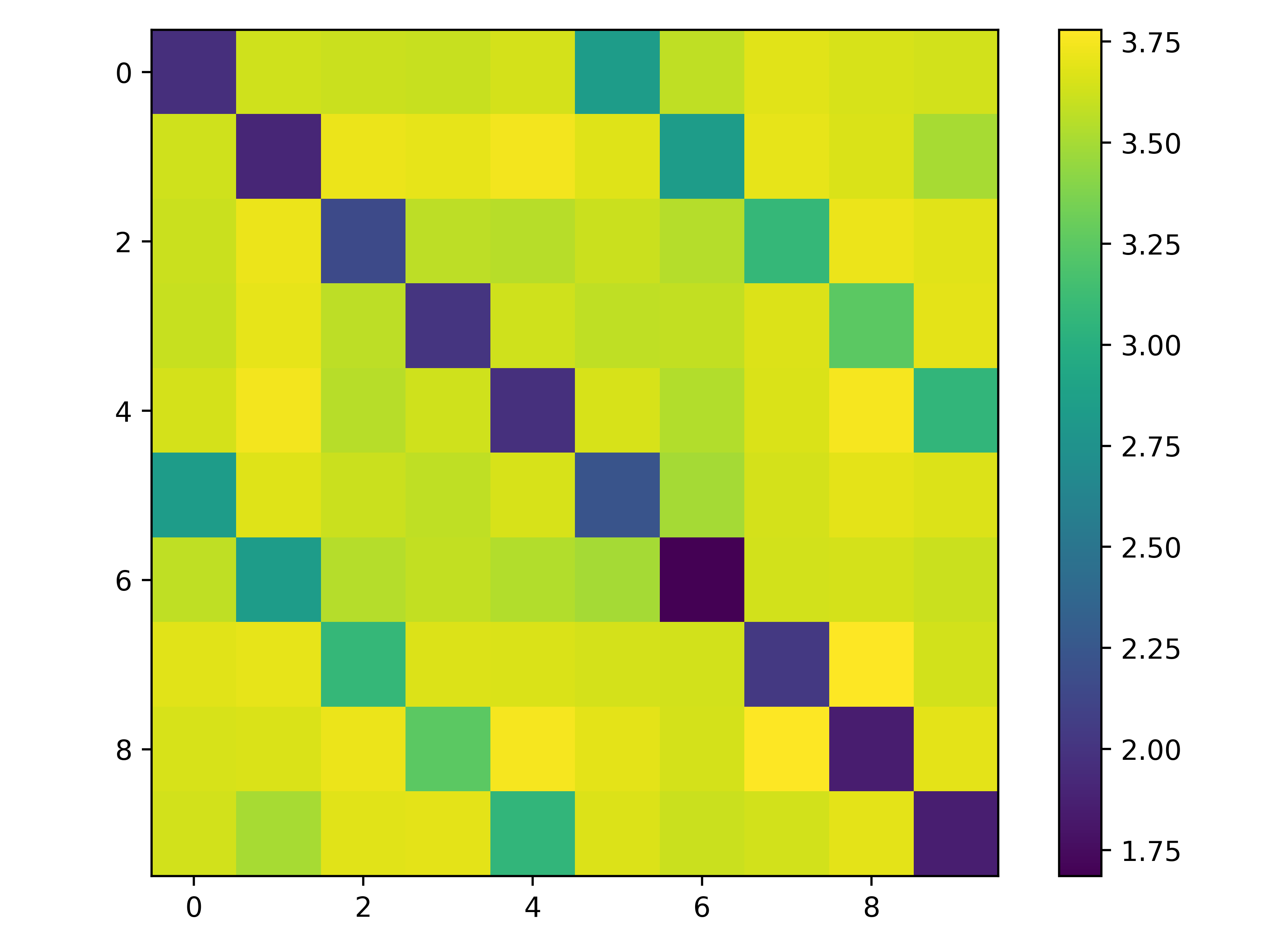}
    \caption{The heatmap of distance matrix of ResNet-18 trained on Coarse CIFAR-10 for 1,000 epochs.}\label{fig:cifar-10-999-distance}
\end{minipage}
\hfill
\begin{minipage}{0.48\linewidth}
\centering
    \includegraphics[scale=0.15]{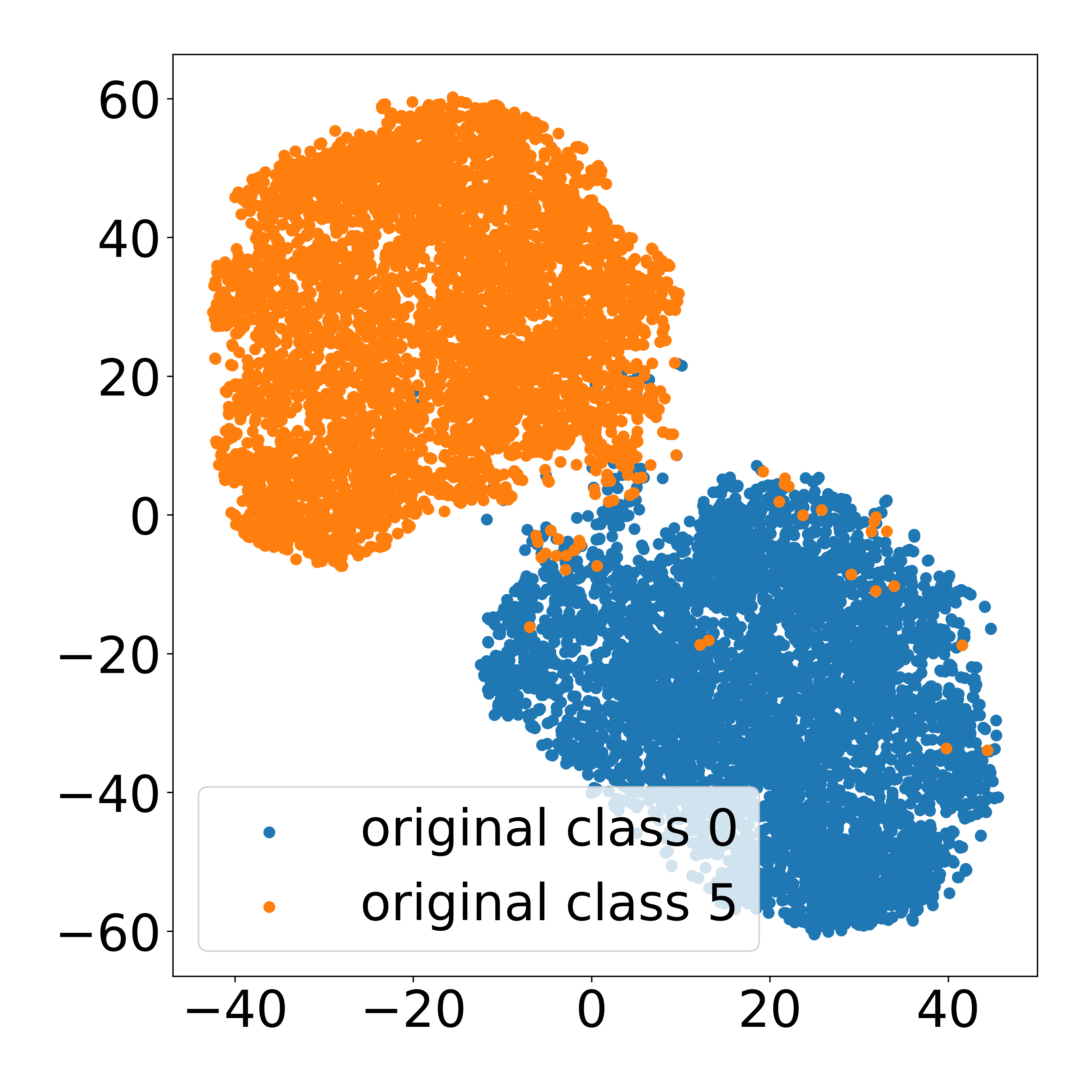}
    \caption{The t-SNE visualization of the last-layer representations of the first super-class of ResNet-18 trained on Coarse CIFAR-10 for 1,000 epochs.}\label{fig:tsne-cifar-10-999-1}
\end{minipage}
\end{figure}

\subsection{Learning CIFAR-10 from 5 Coarse Labels}

As the results in \Cref{sec:visualization} suggest, even after the training accuracy has reached $100\%$ for a long time, the samples within each super-class still exhibit a clear structure per their original class, and those structures act as clusters after reducing dimensionality. Inspired by this observation, we perform a Cluster-and-Linear-Probe (CLP) test to quantify to what extent the original class information is preserved in the last-layer representations. In CLP, we use the representations learned on Coarse CIFAR-10 to reconstruct $10$ labels and run a linear probe on these representations using the reconstructed labels. Specifically, we first use t-SNE to reduce the dimensionality to $2$ and then use $k$-means to find $2$ clusters in the dimensionality-reduced representations within each super-class. We use the clusters as reconstructed labels to do a linear probe. In linear probe, we train a linear classifier on top of the previously learned representations on the training set with reconstructed labels and evaluate the learned linear classifier on the original test set.  Notice that because we do not know the mapping of reconstructed classes to true original classes, we permute each possible mapping and report the highest performance. We also train a linear probe with original training labels as a comparison. The result is shown in \Cref{fig:cluster-group2}.

The performance of CLP on the original test set is comparable to linear probe trained on true original labels or even to models originally trained on CIFAR-10. Notice that the representation $H$ is obtained from the model trained with coarse labels, and the label reconstruction only uses information of $H$ and the number of original classes. This means we can achieve very high performance on the original test set even if we only have access to coarse labels. This result further confirms that the input distribution plays an important role in the last-layer representations.

\begin{figure}[htbp]
    \centering
    \centering
    \includegraphics[width=\linewidth]{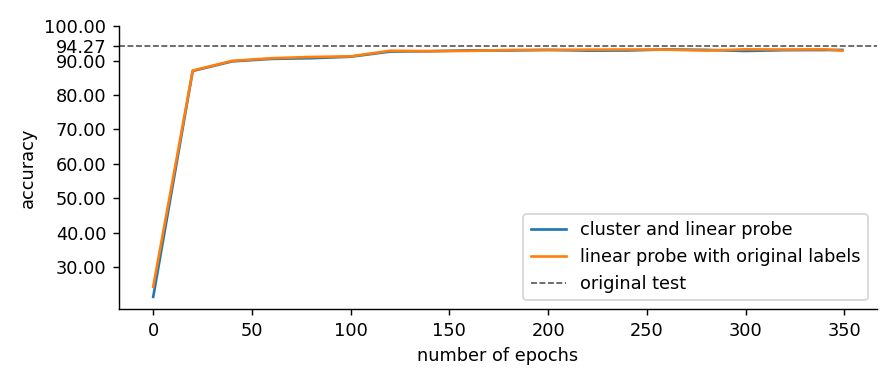}
    %\vspace*{-1.5em}
    \caption{The CLP result. ``original test'' is the highest test set accuracy achieved by ResNet18 trained on original CIFAR-10 with the same training hyper-parameters.}
    \label{fig:cluster-group2}
\end{figure}

% Notice that the representation $H$ is obtained through the model trained with coarse labels, and then the linear probe is trained on reconstructed original labels and never accessed true original labels, also notice that the process we use to construct original label does not use any information of the original labels except the number of classes $C$. Therefore, the whole training process of the CLP only uses the $5$-class coarse label information. However, it still achieves comparable performance on original dataset. The extra original label information can only come from the input. This further confirms that significant information about the input distribution is preserved in the last-layer representations despite Neural Collapse. \yyy{by conclusion to begining}

%\yyy{further confirm} finding exhibits how input distribution can impact last-layer representations on trained models.

%\wh{representations learned by training on Coarse CIFAR-10 contain all the information needed to achieve high accuracy on CIFAR-10}

\section{How Does Semantic Similarity Affect the Fine-Grained Structure?}\label{sec:semantic-similarity}

In our experiments with Coarse CIFAR-10, each coarse label is obtained by combining two classes regardless of the semantics. The fact that the neural network can separate the two classes in its representation space implies that the network recognizes these two classes as semantically different (even though they are given the same coarse label).
%In this case, the neural network shares a similar notion of semantics to us humans
In this section, we explore the following question: If the sub-classes in a super-class have semantic similarity, will the representations still exhibit a fine-grained structure to distinguish them? %This question raises from this thought: In experiments before, we measure how much the last-layer representations are impacted by original labels which are not explicitly provided to illustrate the role of input distribution, because the only source of information for original labels is input distribution. However, theoretically every two training samples are different, and when we say two samples are of the same (original) class, it means they have the same or similar semantics to us humans. In this section, we ask if it is also true for trained neural networks, i.e. if the sub-classes within one super-class are natural and have semantic similarity, can neural networks trained with only super-class information still distinguish the sub-classes as they do in Coarse CIFAR-10?
Intuitively, if the coarse label provided is ``natural'' and consists of semantically similar sub-classes, it is possible that the neural network will not distinguish between them and just produce truly collapsed representations.
\begin{figure}[htbp]
    \begin{minipage}{0.45\linewidth}
    \centering
    \includegraphics[scale=0.11]{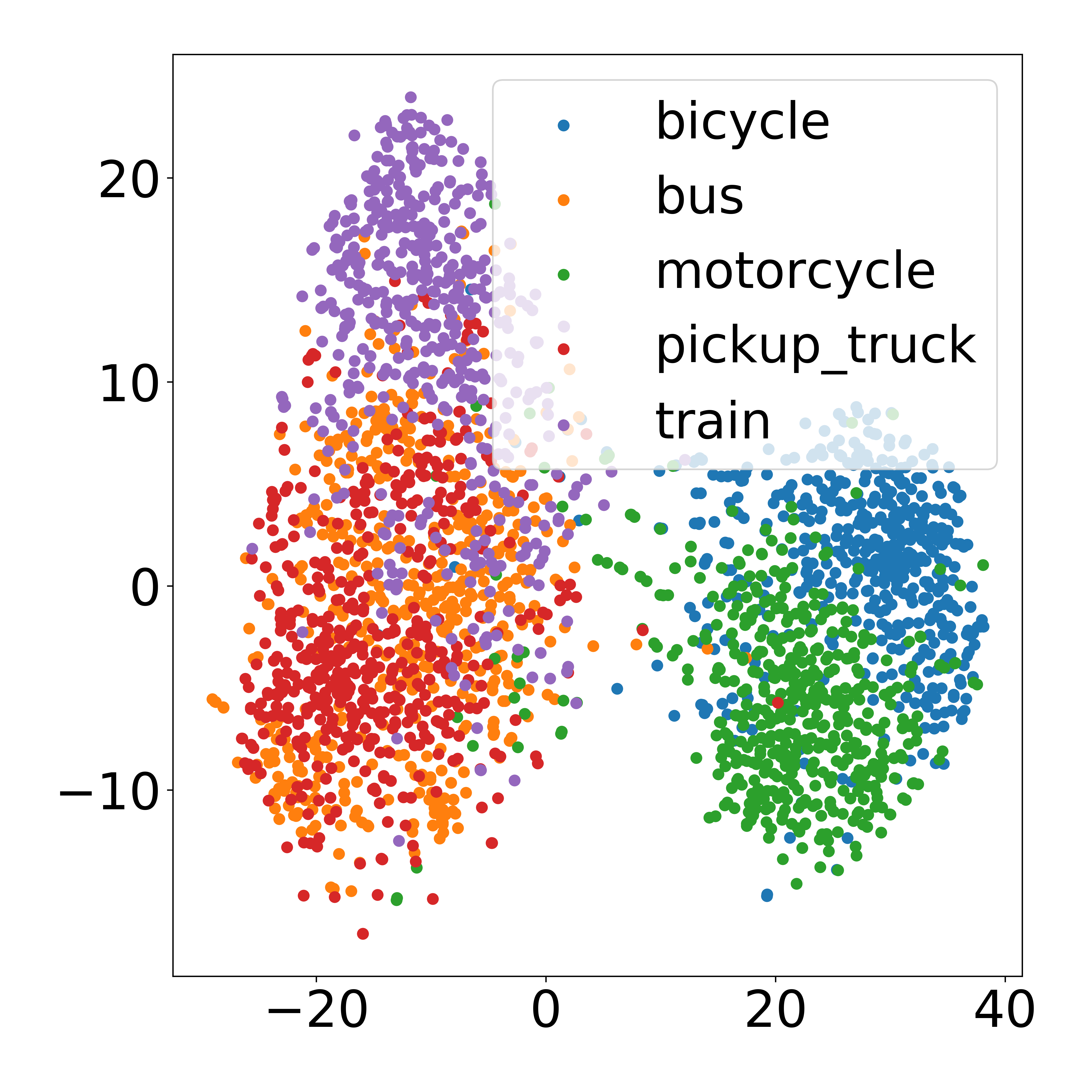}
    \caption{The t-SNE visualization of the last-layer representations of super-class ``vehicles 1'' of ResNet-18 trained on CIFAR-100 with original super-classes.}
    \label{fig:tsne-cifar-100-sup-vehicles}
    \end{minipage}
\hfill
    \begin{minipage}{0.45\linewidth}
    \centering
    \includegraphics[scale=0.11]{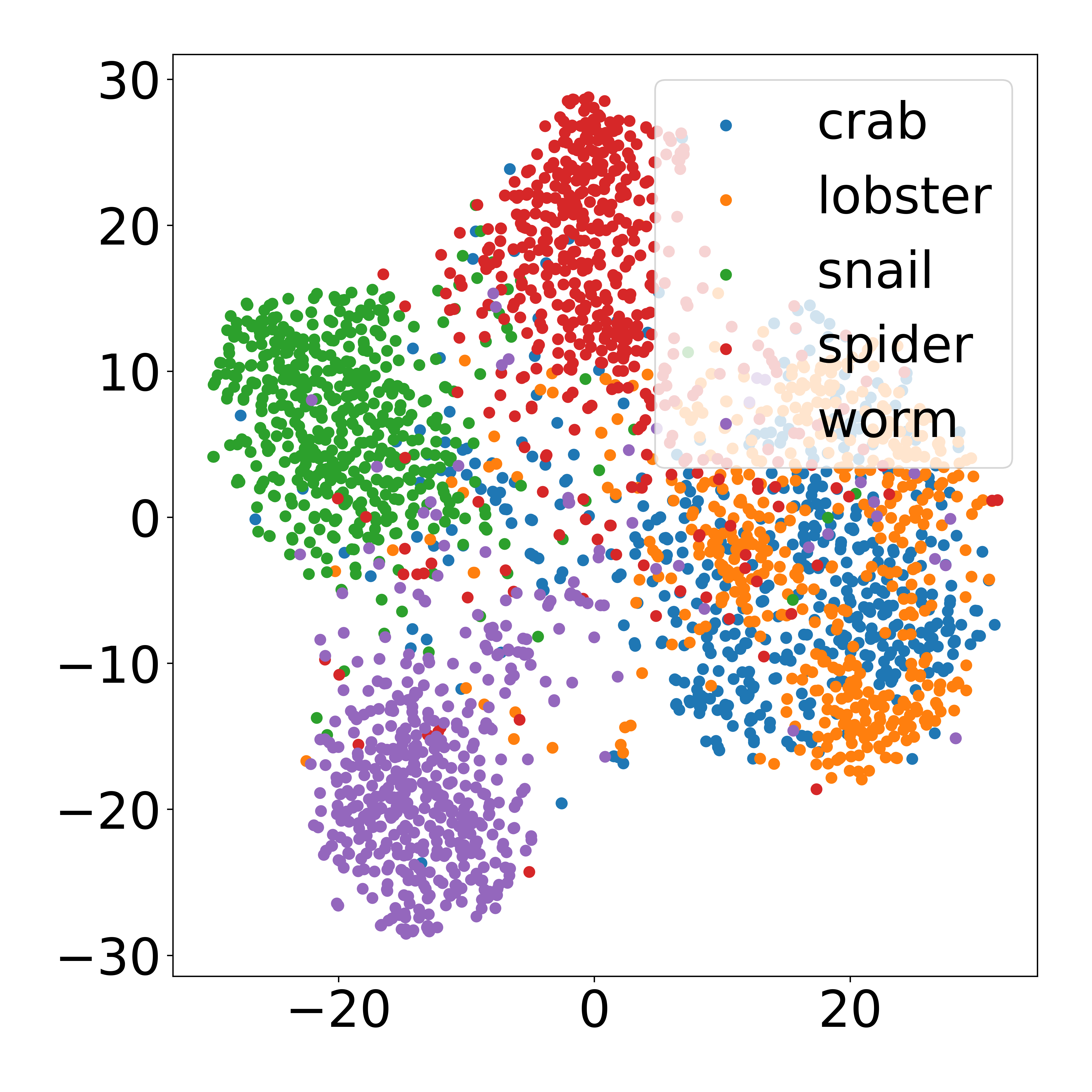}
    \caption{The t-SNE visualization of the last-layer representations of super-class ``non-insect invertebrates'' of ResNet-18 trained on CIFAR-100 with original super-classes.}
    \label{fig:tsne-cifar-100-sup-animals}
    \end{minipage}
\end{figure}

\begin{figure}[htbp]
    \begin{minipage}{0.45\linewidth}
    \centering
    \includegraphics[scale=0.11]{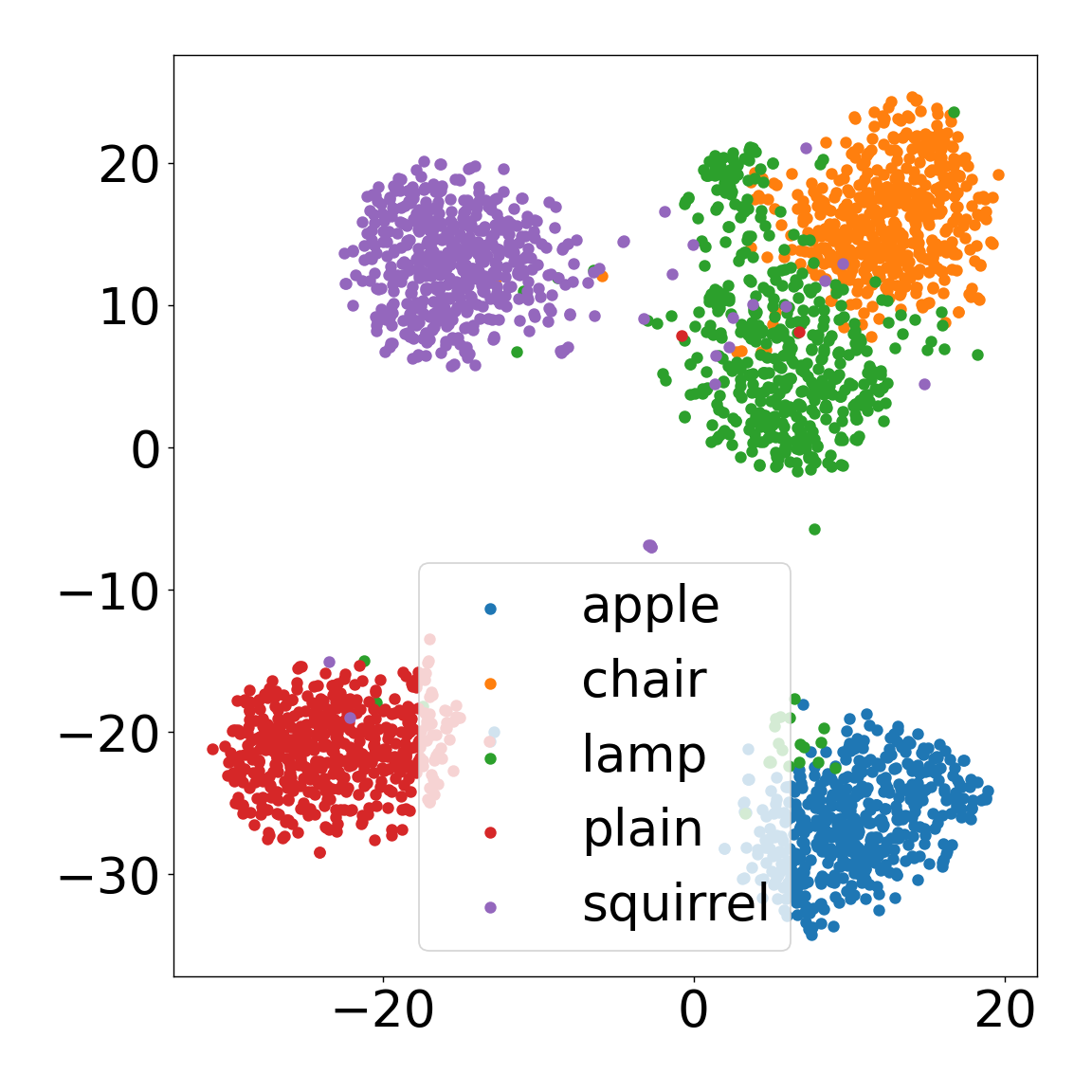}
    \caption{The t-SNE visualization of the last-layer representations of super-class 1 of ResNet-18 trained on CIFAR-100 with random super-classes.}
    \label{fig:tsne-cifar-100-rand-0}
    \end{minipage}
\hfill
    \begin{minipage}{0.45\linewidth}
    \centering
    \includegraphics[scale=0.11]{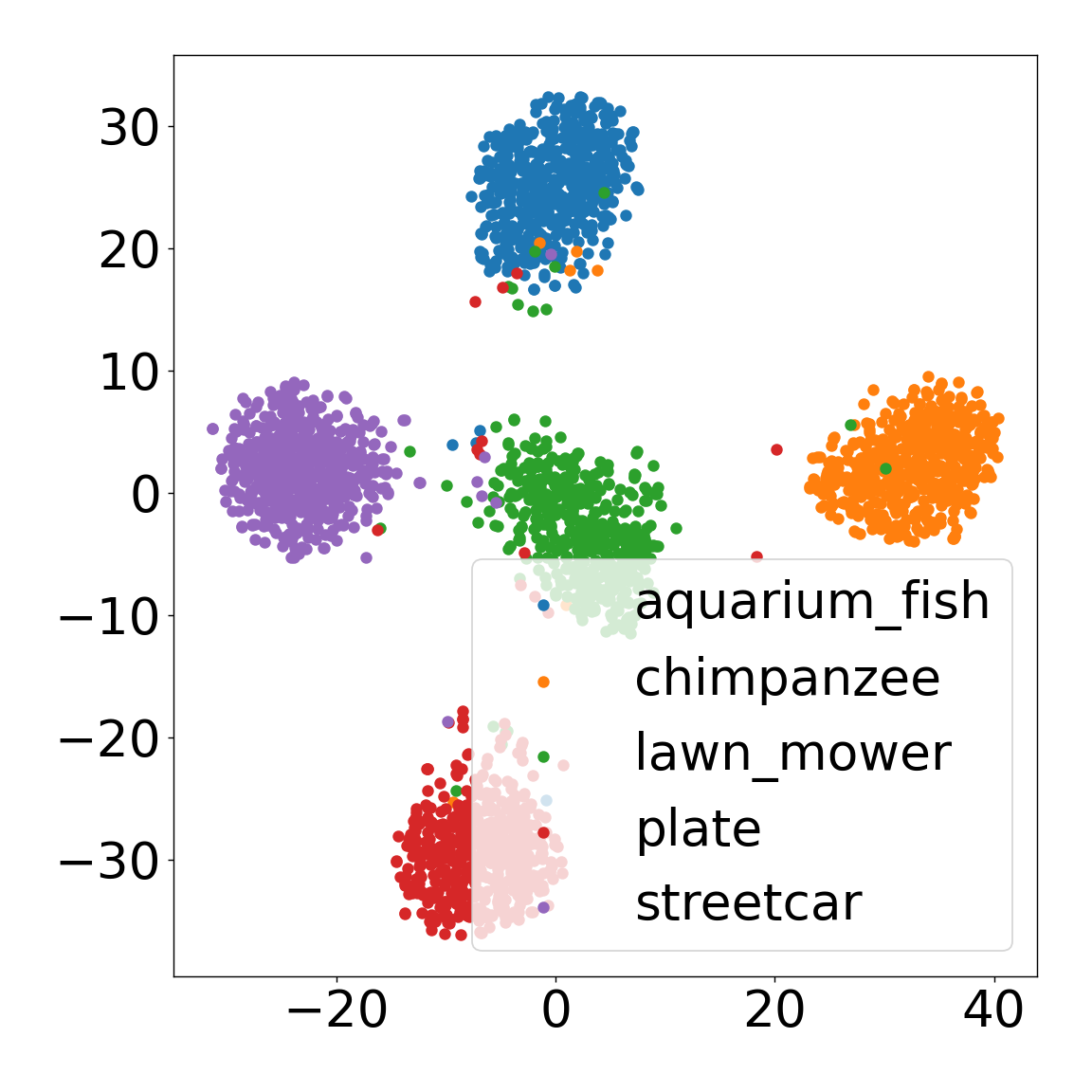}
    \caption{The t-SNE visualization of the last-layer representations of super-class 2 of ResNet-18 trained on CIFAR-100 with random super-classes.}
    \label{fig:tsne-cifar-100-rand-1}
    \end{minipage}
\end{figure}

We take a step towards answering this question by looking at ResNet-18 trained on CIFAR-100 using the official 20 super-classes (each super-class contains 5 sub-classes) as labels. Unlike randomly merging classes as we did in \Cref{sec:experiment-on-coarse-cifar-10}, the official super-classes of CIFAR-100 are natural, merging classes with similar semantics (for example, ``beaver'' and ``dolphin'' both belong to ``aquatic mammals''). This offers a perfect testbed for our question.

We find that ResNet-18 indeed is not able to distinguish all sub-classes in its representation space, but can still produce separable representations if some sub-classes within a super-class are sufficiently different. Interestingly, the notion of semantic similarity of ResNet-18 turns out to agree well with that of humans.
\Cref{fig:tsne-cifar-100-sup-animals,fig:tsne-cifar-100-sup-vehicles} show the t-SNE visualizations of representations from two super-classes.
From the visualizations, although there are not as clear clusters as for Coarse CIFAR-10, the representations do exhibit visible separations between certain sub-classes.
In \Cref{fig:tsne-cifar-100-sup-vehicles}, ``bicycles'' and ``motorcycles'' are entangled together, while they are separated from ``bus'', ``pickup truck'', and ``train'', which is human-interpretable.
In \Cref{fig:tsne-cifar-100-sup-animals}, ``crab'' and ``lobster'' are mixed together, which are both aquatic and belong to malacostraca, while the other three are not and have more differentiative representations.

In comparison, when the CIFAR-100 classes are randomly merged into 20 super-classes, we find that the sub-class representations are much better separated (\Cref{fig:tsne-cifar-100-rand-0,fig:tsne-cifar-100-rand-1}). This is because randomly merged super-classes no longer have semantic similarity in their sub-classes.

These results confirm the intuition that the fine-grained structure in last-layer representations is affected by, or even based on, the semantic similarity between the inputs.

%\vspace{-0.7em}
\section{Fine-Grained Representation Structure on Fine CIFAR-10}\label{sec:a-refined-dataset}

In this section, we consider a finely-labeled dataset. We construct a fine version of CIFAR-10 with the process described in \Cref{sec:experiment-setup}, and call it Fine CIFAR-10.
\Cref{fig:distance-result-group-1-fine} presents the class distance matrices, arranged by the number of training epochs. As before, we only provide the results for a specific training hyper-parameter setting here and defer the full results to \cref{sec:complete-cifar-20}.

It can be observed that at the early stage of training, there are three dark lines, which indicates the last-layer representations are converging towards $10$ clusters instead of $20$. At the end of training, this $10$-class relationship is still preserved, although with a lighter color. Therefore, both the input distribution and the label information still have a strong influence on the representation structure when fine labels are used for training.

%Those observations are consistent with the observations made in \Cref{sec:experiment-on-coarse-cifar-10} and supports our conclusion.

\begin{figure}[htbp]
\centering
\subfigure[epoch = 20]{\includegraphics[scale=0.22]{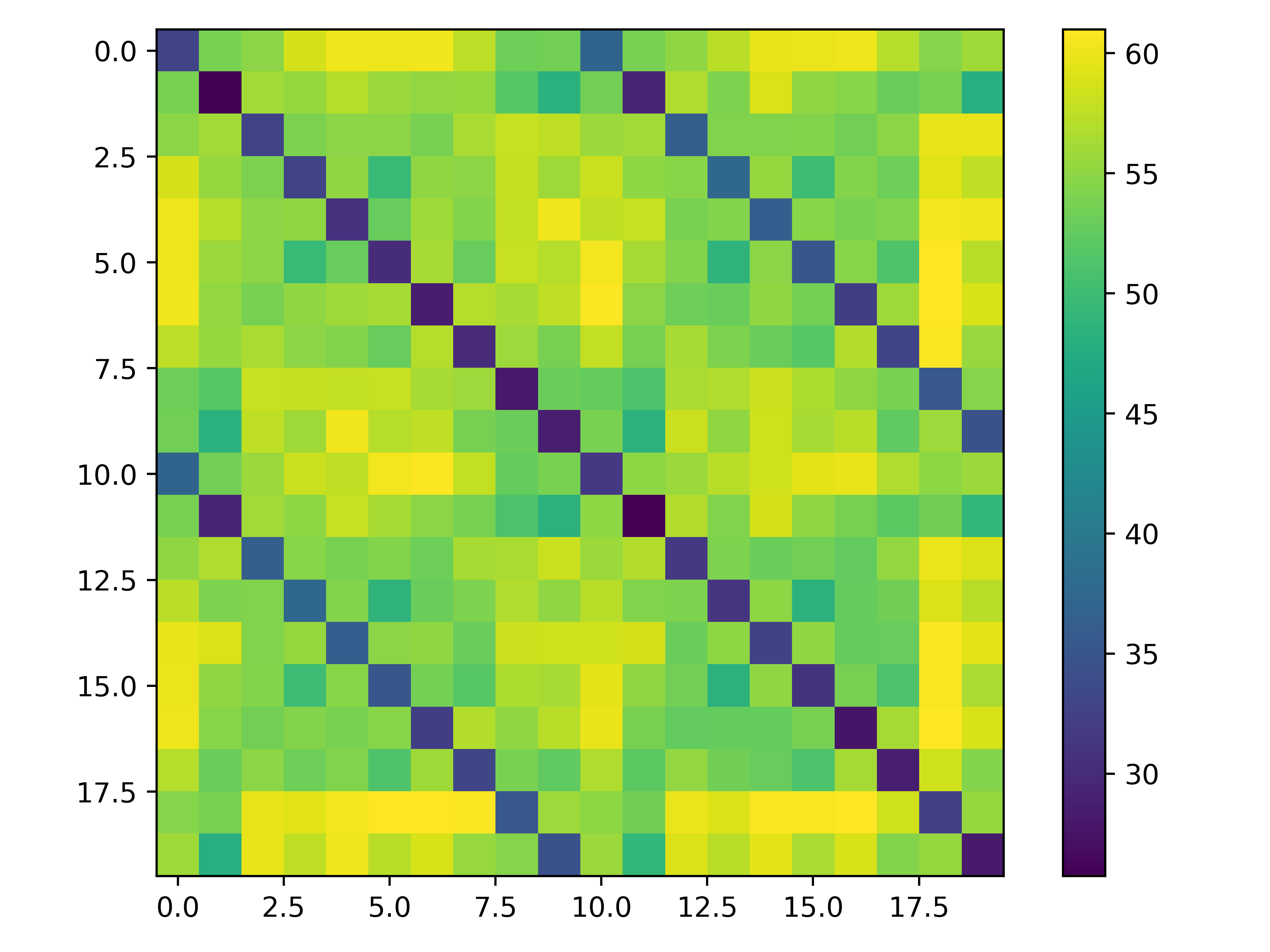}}
    \subfigure[epoch = 120]{\includegraphics[scale=0.22]{figures/distance-cifar-20-2-120.png}}
    \subfigure[epoch = 240]{\includegraphics[scale=0.22]{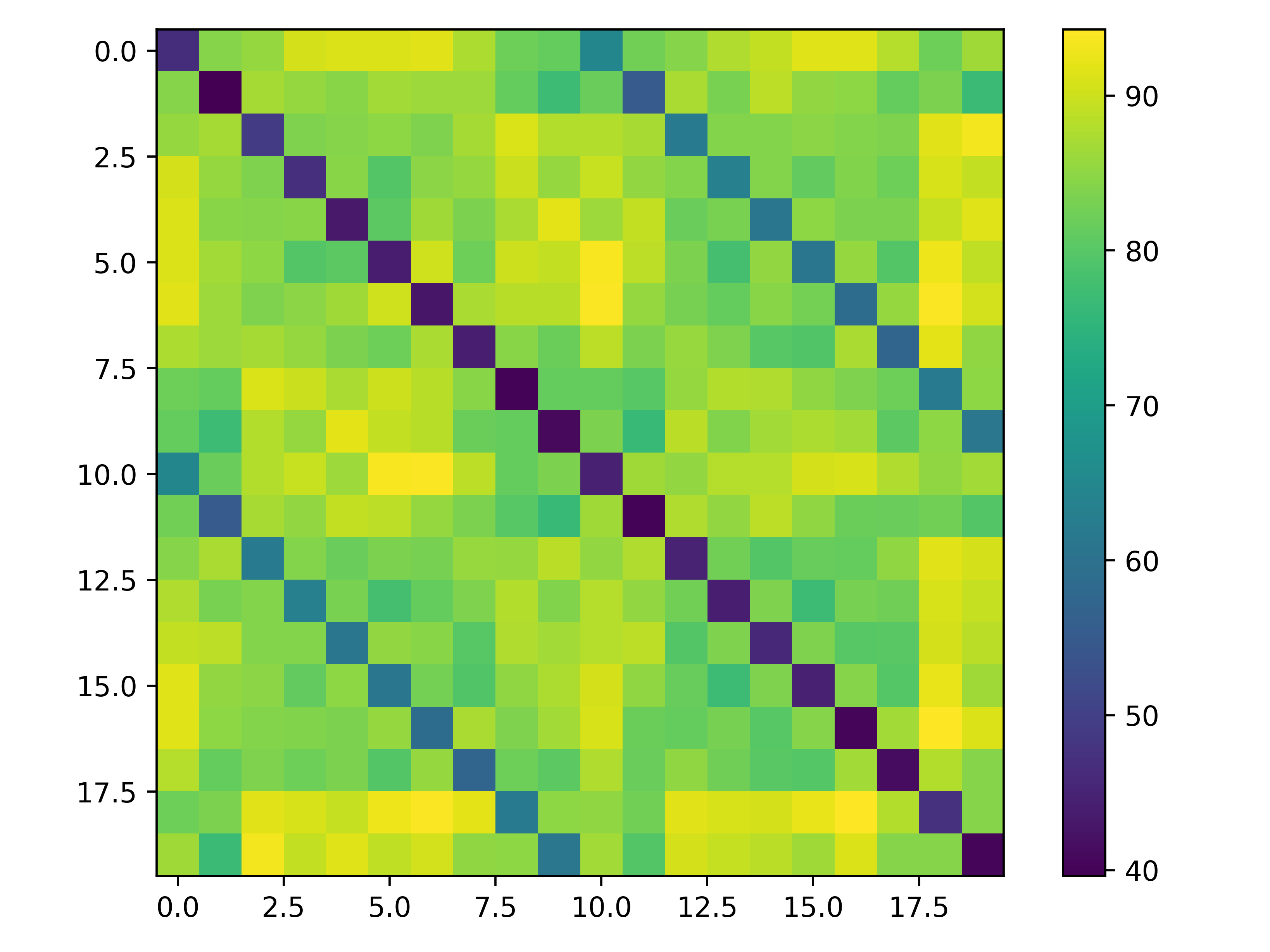}}
    \subfigure[epoch = 350]{\includegraphics[scale=0.22]{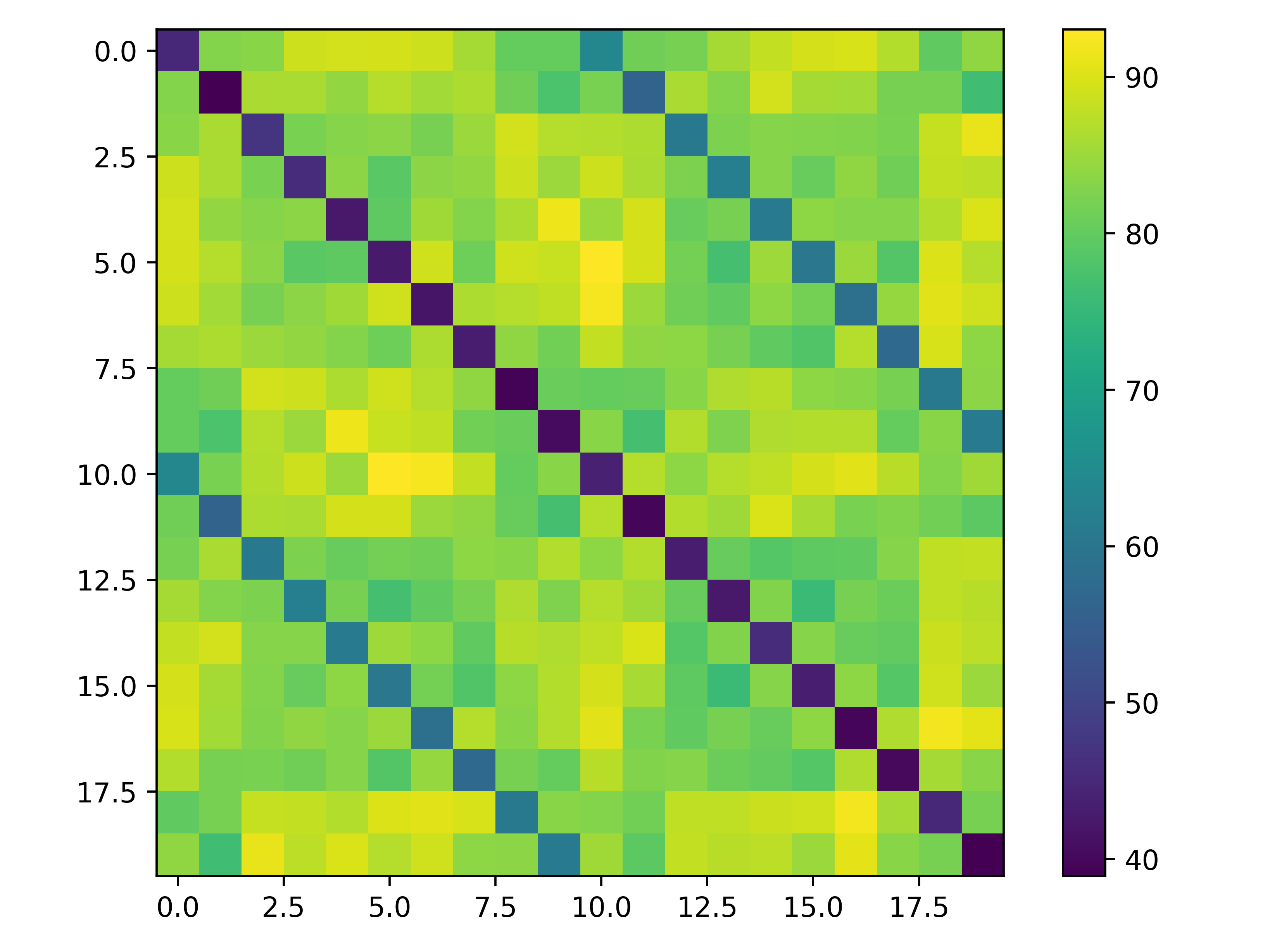}}
    \caption{The heatmap of class distance matrix on Fine CIFAR-10.}
    \label{fig:distance-result-group-1-fine}
\end{figure}

%\vspace{-0.7em}
\section{Theoretical Result in a Synthetic Setting} \label{sec:theory}

In this section, we provide a theoretical result to show the fine-grained representation structure for a coarsely labeled dataset, supporting our empirical observations in previous sections.
In particular, we consider a one-hidden-layer neural network trained by gradient descent on Gaussian mixture data. We describe our setting below.
%theoretical analysis on the dynamics of a two-layer neural network, explaining why input can impact the last-layer representation. Following previous work \citet{???}\yyy{data aug}. 

\paragraph{Data generation.}
We consider input data generated from a mixture of 4 separated Gaussian distributions in $\R^d$. We denote the four clusters as $\mathscr C_1$, $\mathscr C_2$, $\mathscr C_3$, and $\mathscr C_4$. We give coarse label $+1$ to inputs from $\mathscr C_1$ and $\mathscr C_2$, and give coarse label $-1$ to inputs from $\mathscr C_3$ and $\mathscr C_4$.
We denote the $n$ samples as $\{(\boldsymbol{x}_i, y_i)\}_{i=1}^n$.
We assume a sample ${\boldsymbol{x}}_i$ where $i \in \mathscr C_p$ ($p \in \{1,2,3,4\}$) is generated according to \begin{equation}{\boldsymbol{x}}_i = {\boldsymbol{\mu}}^{(p)} + {\boldsymbol{\xi}}_i,\end{equation}
where ${\boldsymbol{\xi}}_i \sim \mathcal N(0,\kappa^2{\boldsymbol{I}})$. %If ${\boldsymbol{x}}_i$ belongs to $C_p$, let ${\boldsymbol{\mu}}_i = {\boldsymbol{\mu}}^{(p)}$.
We assume that the means ${\boldsymbol{\mu}}^{(p)}$'s are pairwise orthogonal and %with L2 norm $\tau$, i.e. for $p \neq q$, we have ${\boldsymbol{\mu}}^{(p)\top}{\boldsymbol{\mu}}^{(q)} = 0$ and 
$\|{\boldsymbol{\mu}}^{(p)}\|_2 = \tau$. For convenience, assume each cluster $\mathscr C_p$ has the same number of samples $|\mathscr C_p| = n/4$. 

\paragraph{Neural network.}
We consider training a one-hidden-layer network with $m$ hidden neurons. The first-layer weight matrix is $W = ({\boldsymbol{w}}_1,{\boldsymbol{w}}_2,\cdots,{\boldsymbol{w}}_m)\in \R^{d\times m}$ and is trained by gradient descent. The second-layer weights are fixed to be all ones (which satisfies the ETF structure predicted by Neural Collapse in this setting). The output of the network is
\begin{equation*}
f({\boldsymbol{x}}, W) = \mathbf 1^\top h({\boldsymbol{x}}, W) = \sum_{r=1}^m \sigma\left({\boldsymbol{w}}_r^\top \boldsymbol x\right),
\end{equation*}
where $h(\boldsymbol x, W) = \left( \sigma({\boldsymbol{w}}_r^\top \boldsymbol x) \right)_{r=1}^m \in \R^m$ is the hidden-layer representation, and
\begin{equation*}
\sigma(z) = \begin{cases}\frac{1}{3}z^3 & \text{if $|z| \leq 1$}\\ z - \frac{2}{3} & \text{if $z \geq 1$} \\ z + \frac{2}{3} & \text{if $z \leq -1$}\end{cases}
\end{equation*}
is the activation function. This activation is a smoothed version of symmetrized ReLU; it and its variants have been adopted in a line of theoretical work (e.g. \citet{allen2020towards,zou2021understanding, shen2022data}).

%If ${\boldsymbol{x}}_i \in C_p$, let \begin{align}h({\boldsymbol{x}}_i) = \begin{bmatrix} & \sigma\left({\boldsymbol{w}}_1^\top {\boldsymbol{\mu}}^{(p)} + {\boldsymbol{w}}_1^\top {\boldsymbol{\xi}}_i\right) \\ & \sigma\left({\boldsymbol{w}}_2^\top {\boldsymbol{\mu}}^{(p)} + {\boldsymbol{w}}_2^\top {\boldsymbol{\xi}}_i\right) \\  &\vdots \\ & \sigma\left({\boldsymbol{w}}_m^\top {\boldsymbol{\mu}}^{(p)} + {\boldsymbol{w}}_m^\top {\boldsymbol{\xi}}_i\right) \end{bmatrix}\end{align} and \begin{equation}f({\boldsymbol{x}}_i) = \mathbf 1^\top h({\boldsymbol{x}}_i) = \sum_{r=1}^m \sigma\left({\boldsymbol{w}}_r(t)^\top {\boldsymbol{\mu}}^{(p)} + {\boldsymbol{w}}_r(t)^\top {\boldsymbol{\xi}}_i\right),\end{equation}
%where we use \begin{equation}
%\sigma(z) = \begin{cases}\frac{1}{3}z^3 & \text{if $|z| \leq 1$}\\ z - \frac{2}{3} & \text{if $z \geq 1$} \\ z + \frac{2}{3} & \text{if $z \leq -1$}\end{cases}
%\end{equation}
%as activation function.

\paragraph{Loss function and training algorithm.}
We train the network by gradient descent on the training loss
$$L(W) = \frac1n \sum_{i=1}^n \ell( f({\boldsymbol{x}}_i, W), y_i)$$
where $\ell(\hat y, y) = \frac{-y \hat y}{2}$ is the unhinged loss \citep{unhinged}. 
Our result can be generalized to the logistic loss since when $f({\boldsymbol{x}}_i, W)$ is small, the unhinged loss can be viewed as an approximation of logistic loss \citep{shen2022data}.
We initialize the first-layer weights i.i.d. from $\mathcal N(0, \omega^2)$ and update them using gradient descent with learning rate $\eta$.

%Suppose we train the model using loss function $\ell(f({\boldsymbol{x}}),y) = \log\left(1 + \exp\left[-yf({\boldsymbol{x}}) \right]\right)$ whose gradient\yyy{jacobian?} is \begin{equation}\frac{\partial \ell\left(f({\boldsymbol{x}}),y\right)}{\partial f({\boldsymbol{x}})} = \frac{-y \exp\left[-yf({\boldsymbol{x}}) \right]}{1 + \exp\left[-yf({\boldsymbol{x}}) \right]},\end{equation}
%which is nearly $-\frac{y}{2}$ when $f({\boldsymbol{x}})$ is closed to $0$ at beginning\yyy{?}. Define \begin{equation}L(W) = \frac 1n \sum_{r=1}^m \sum_{i=1}^n \ell(f({\boldsymbol{x}}_i;{\boldsymbol{w}}_r(t)); y_i).\end{equation} 

Our main theorem shows that after training, the hidden-layer representations for $\mathscr C_1$ and $\mathscr C_2$ will form two separate clusters, even though they are given the same label. (Similar result holds for $\mathscr C_3$ and $\mathscr C_4$ by symmetry.) 
In particular, for any three samples $i_1, {i_2} \in \mathscr 
 C_1$ and ${i_3} \in \mathscr C_2$, we show that
\begin{equation*}
\left\|h({\boldsymbol{x}}_{i_1}) - h({\boldsymbol{x}}_{i_2})\right\|_2 \ll \|h({\boldsymbol{x}}_{i_1}) - h({\boldsymbol{x}}_{i_3})\|_2.
\end{equation*}

\begin{theorem} \label{thm:main}
Consider the synthetic setting describe above and let $\mathfrak  c  = \frac{8\kappa \sqrt{d}}{\tau}$.
Suppose that the following conditions hold regarding the Gaussian mean length $\tau$, Gaussian variance $\kappa^2$, weight initialization variance $\omega^2$, input dimension $d$, number of samples $n$, and number of neurons $m$:
%Suppose the data is generated and ${\boldsymbol{W}}$ is initialized as described before, ${\boldsymbol{W}}$ is updated through gradient descent with step size $\eta$, and the following conditions holds:
%\begin{enumerate}
%    \item $\gamma > 72 \log\sqrt{\left(\frac{6}{\delta - 0.7}\right)}$ \label{stat:final-cond-gamma}
%    \item $n \geq 288 \gamma^2$; \label{stat:final-cond-n}
%    \item $n^2 \leq \left(\frac{\tau\omega}{18}\right) \times \frac{d}{\mathcal L^2}\log\frac{6n}{\delta - 0.7}$; \label{stat:final-cond-n-and-d}
%    \item $\mathcal L \leq \frac{ (\omega \tau)^2 \sqrt{d}}{64^2 \gamma}$; \label{stat:final-cond-L}
%    \item $\eta \leq \frac{16}{\omega \tau^2}$, \label{stat:final-cond-eta}
%    \item $(6n + 6)\exp(-d/64) + 6\exp(-n) < 0.1$, \label{stat:final-cond-delta}
%    \item $\omega \tau \leq 36$, \label{stat:final-cond-omega-tau}
%\end{enumerate}
\begin{enumerate}
    \item $n^\frac{1}{2}d^{-\frac{1}{4}} \ll \mathfrak  c \ll n^{\frac{1}{2}} d^{-\frac{1}{6}}$; %\label{stat:final-cond-gamma}
    \item $d^{1/3} \gg n$; %\label{stat:final-cond-n}
    \item $d^{-\frac{1}{4}}n^{\frac{1}{2}}\mathfrak  c^{-1}  \ll \tau \omega \ll \log^{-\tfrac{1}{2}}(m)$. %\label{stat:final-cond-n-and-d}
    %\item $m \gg 1$,
    % \item $(\omega \tau)^2 \gg \frac{\gamma}{\sqrt{d}}$; %$ \mathcal L \leq \frac{ (\omega \tau)^2 \sqrt{d}}{64^2 \gamma}$; %\label{stat:final-cond-L}
    %\item $\eta \leq \frac{16}{\omega \tau^2}$; %\label{stat:final-cond-eta}
    % \item $\omega \tau \ll m^{-1/6}$; %\label{stat:final-cond-omega-tau}
    % \item $n\gamma^2 \tau \omega \gg 1$.
\end{enumerate}
For learning rate $\eta = O \left(\min\left\{\mathfrak  c^3 \tau^{-4}, \mathfrak  c^2 \omega \tau^{-3}, \mathfrak  c^2 \omega \tau\right\}\right)$ and number of iterations $T = \Theta(\frac{1}{\eta \omega \tau^3})$, with high probability, the hidden-layer representation map $h(\boldsymbol{x}) = h(\boldsymbol{x}, W(T))$ satisfies that for all $i_1, {i_2} \in \mathscr 
 C_1$ and ${i_3} \in \mathscr C_2$, we have
\begin{equation*}
\left\|h({\boldsymbol{x}}_{i_1}) - h({\boldsymbol{x}}_{i_2})\right\|_2 \ll \|h({\boldsymbol{x}}_{i_1}) - h({\boldsymbol{x}}_{i_3})\|_2.
\end{equation*}
%Then with probability at least $1-e^{-m/80}$, we have $\frac{\left\|h({\boldsymbol{x}}_{i_1}) - h\left({\boldsymbol{x}}_{i_3}\right)\right\|}{\left\|h({\boldsymbol{x}}_{i_1}) - h\left({\boldsymbol{x}}_{i_2}\right)\right\|} = O(n\gamma^2 \tau \omega)$.
\end{theorem}

An example set of parameters that satisfy the above conditions is:
\begin{align*}
\kappa = 1, \tau = d^{0.52}, \omega = d^{-0.53}, m = \log d, n = d^{0.32} .
\end{align*}

The proof of \Cref{thm:main} is given in \cref{sec:the-proof-of-theory}. The main step is to prove that after training, the neurons $\boldsymbol w_r$ will be better correlated with the class means $\boldsymbol{\mu}^{(p)}$ than with the individual sample noise $\boldsymbol{\xi}_i$. Therefore, the network will produce more similar representations for samples from the same cluster than for samples from different clusters.

\paragraph{Empirical verification.}

To verify our theoretical analysis and gain more understanding of the fine-grained structure of last-layer representation we provide some synthetic experiment results under the setting of classifying mixture of Gaussian using a 2-layer MLP in \cref{sec:additional-synthetic}, which is analogous (although not exactly the same) to the setting analyzed in our theory. Notice that this synthetic experiment is not only helpful to verify the theory, but also able to let us perform controlled experiments by varying different characteristics of the data distribution, architecture, and algorithmic components, to let us better understand how those hyper-parameters play a role in the final-layer representation, and serve as a starting point for a more thorough understanding of the last-layer representation behavior of neural networks.

%\vspace{-0.7em}
\section{Discussion}
%\vspace{-0.7em}

In this paper, we initiated the study of the role of the intrinsic structure of the input data distribution on the last-layer representations of neural networks, and in particular, how to reconcile it with the Neural Collapse phenomenon, which is only driven by the explicit labels provided in the training procedure. Through a series of experiments, we provide concrete evidence that the representations can exhibit clear fine-grained structure despite their apparent collapse.
While Neural Collapse is an intriguing phenomenon and deserves further studies to understand its cause and consequences, our work calls for more scientific investigations of the structure of neural representations that go beyond Neural Collapse.

We note that the fine-grained representation structure we observed depends on the inductive biases of the network architecture and the training algorithm, and may not appear universally. In our experiments on Coarse CIFAR-10, we observe the fine-grained structure for ResNet and DenseNet, but not for VGG (see \cref{sec:vgg,sec:densenet} for extended results). We also note that certain choices of learning rate and weight-decay rate lead to stronger fine-grained structure than others. We leave a thorough investigation of such subtlety for future work.

\section*{Acknowledgements}

WH would like to thank Ruoqi Shen for helpful discussions and contributions in the early stage of this project.
JS acknowledges support from NSF DMS-2031899.

\bibliography{references}

\begin{thebibliography}{26}
\providecommand{\natexlab}[1]{#1}
\providecommand{\url}[1]{\texttt{#1}}
\expandafter\ifx\csname urlstyle\endcsname\relax
  \providecommand{\doi}[1]{doi: #1}\else
  \providecommand{\doi}{doi: \begingroup \urlstyle{rm}\Url}\fi

\bibitem[Allen-Zhu \& Li(2020)Allen-Zhu and Li]{allen2020towards}
Allen-Zhu, Z. and Li, Y.
\newblock Towards understanding ensemble, knowledge distillation and
  self-distillation in deep learning.
\newblock \emph{arXiv preprint arXiv:2012.09816}, 2020.

\bibitem[Chung \& Lu(2006)Chung and Lu]{concen-survey}
Chung, F. and Lu, L.
\newblock Concentration inequalities and martingale inequalities: a survey.
\newblock \emph{Internet mathematics}, 3\penalty0 (1):\penalty0 79--127, 2006.

\bibitem[Ergen \& Pilanci(2021)Ergen and Pilanci]{ergen2021revealing}
Ergen, T. and Pilanci, M.
\newblock Revealing the structure of deep neural networks via convex duality.
\newblock In \emph{International Conference on Machine Learning}, pp.\
  3004--3014. PMLR, 2021.

\bibitem[Fang et~al.(2021)Fang, He, Long, and Su]{fang2021exploring}
Fang, C., He, H., Long, Q., and Su, W.~J.
\newblock Exploring deep neural networks via layer-peeled model: Minority
  collapse in imbalanced training.
\newblock \emph{Proceedings of the National Academy of Sciences}, 118\penalty0
  (43):\penalty0 e2103091118, 2021.

\bibitem[Galanti et~al.(2021)Galanti, Gy{\"o}rgy, and Hutter]{galanti2021role}
Galanti, T., Gy{\"o}rgy, A., and Hutter, M.
\newblock On the role of neural collapse in transfer learning.
\newblock \emph{arXiv preprint arXiv:2112.15121}, 2021.

\bibitem[Han et~al.(2021)Han, Papyan, and Donoho]{han2021neural}
Han, X., Papyan, V., and Donoho, D.~L.
\newblock Neural collapse under mse loss: Proximity to and dynamics on the
  central path.
\newblock \emph{arXiv preprint arXiv:2106.02073}, 2021.

\bibitem[Hui et~al.(2022)Hui, Belkin, and Nakkiran]{hui2022limitations}
Hui, L., Belkin, M., and Nakkiran, P.
\newblock Limitations of neural collapse for understanding generalization in
  deep learning.
\newblock \emph{arXiv preprint arXiv:2202.08384}, 2022.

\bibitem[Ji et~al.(2021)Ji, Lu, Zhang, Deng, and Su]{ji2021unconstrained}
Ji, W., Lu, Y., Zhang, Y., Deng, Z., and Su, W.~J.
\newblock An unconstrained layer-peeled perspective on neural collapse.
\newblock \emph{arXiv preprint arXiv:2110.02796}, 2021.

\bibitem[Levy \& Goldberg(2014)Levy and
  Goldberg]{levy-goldberg-2014-linguistic}
Levy, O. and Goldberg, Y.
\newblock Linguistic regularities in sparse and explicit word representations.
\newblock In \emph{Proceedings of the Eighteenth Conference on Computational
  Natural Language Learning}, pp.\  171--180, Ann Arbor, Michigan, June 2014.
  Association for Computational Linguistics.
\newblock \doi{10.3115/v1/W14-1618}.
\newblock URL \url{https://aclanthology.org/W14-1618}.

\bibitem[Lu \& Steinerberger(2020)Lu and Steinerberger]{lu2020neural}
Lu, J. and Steinerberger, S.
\newblock Neural collapse with cross-entropy loss.
\newblock \emph{arXiv preprint arXiv:2012.08465}, 2020.

\bibitem[McKenna()]{gaussian-abs-sum-tail-bound}
McKenna, R.
\newblock Tail bounds on the sum of half normal random variables.
\newblock
  \url{http://www.ryanhmckenna.com/2021/12/tail-bounds-on-sum-of-half-normal.html}.

\bibitem[Mixon et~al.(2022)Mixon, Parshall, and Pi]{mixon2022neural}
Mixon, D.~G., Parshall, H., and Pi, J.
\newblock Neural collapse with unconstrained features.
\newblock \emph{Sampling Theory, Signal Processing, and Data Analysis},
  20\penalty0 (2):\penalty0 1--13, 2022.

\bibitem[Olah et~al.(2017)Olah, Mordvintsev, and Schubert]{olah2017feature}
Olah, C., Mordvintsev, A., and Schubert, L.
\newblock Feature visualization.
\newblock \emph{Distill}, 2017.
\newblock \doi{10.23915/distill.00007}.
\newblock https://distill.pub/2017/feature-visualization.

\bibitem[Papyan et~al.(2020)Papyan, Han, and Donoho]{neuralcollapse-original}
Papyan, V., Han, X., and Donoho, D.~L.
\newblock Prevalence of neural collapse during the terminal phase of deep
  learning training.
\newblock \emph{Proceedings of the National Academy of Sciences}, 117\penalty0
  (40):\penalty0 24652--24663, 2020.

\bibitem[Rumelhart et~al.(1985)Rumelhart, Hinton, and
  Williams]{rumelhart1985learning}
Rumelhart, D.~E., Hinton, G.~E., and Williams, R.~J.
\newblock Learning internal representations by error propagation.
\newblock Technical report, California Univ San Diego La Jolla Inst for
  Cognitive Science, 1985.

\bibitem[Shen et~al.(2022)Shen, Bubeck, and Gunasekar]{shen2022data}
Shen, R., Bubeck, S., and Gunasekar, S.
\newblock Data augmentation as feature manipulation.
\newblock In \emph{International Conference on Machine Learning}, pp.\
  19773--19808. PMLR, 2022.

\bibitem[Sohoni et~al.(2020)Sohoni, Dunnmon, Angus, Gu, and
  R{\'e}]{sohoni2020no}
Sohoni, N., Dunnmon, J., Angus, G., Gu, A., and R{\'e}, C.
\newblock No subclass left behind: Fine-grained robustness in coarse-grained
  classification problems.
\newblock \emph{Advances in Neural Information Processing Systems},
  33:\penalty0 19339--19352, 2020.

\bibitem[Tirer \& Bruna(2022)Tirer and Bruna]{tirer2022extended}
Tirer, T. and Bruna, J.
\newblock Extended unconstrained features model for exploring deep neural
  collapse.
\newblock \emph{arXiv preprint arXiv:2202.08087}, 2022.

\bibitem[Van~der Maaten \& Hinton(2008)Van~der Maaten and Hinton]{t-sne}
Van~der Maaten, L. and Hinton, G.
\newblock Visualizing data using t-sne.
\newblock \emph{Journal of machine learning research}, 9\penalty0 (11), 2008.

\bibitem[van Rooyen et~al.(2015)van Rooyen, Menon, and Williamson]{unhinged}
van Rooyen, B., Menon, A.~K., and Williamson, R.~C.
\newblock Learning with symmetric label noise: The importance of being
  unhinged.
\newblock In Cortes, C., Lawrence, N.~D., Lee, D.~D., Sugiyama, M., and
  Garnett, R. (eds.), \emph{Advances in Neural Information Processing Systems
  28: Annual Conference on Neural Information Processing Systems 2015, December
  7-12, 2015, Montreal, Quebec, Canada}, pp.\  10--18, 2015.

\bibitem[Weinan \& Wojtowytsch(2022)Weinan and
  Wojtowytsch]{weinan2022emergence}
Weinan, E. and Wojtowytsch, S.
\newblock On the emergence of simplex symmetry in the final and penultimate
  layers of neural network classifiers.
\newblock In \emph{Mathematical and Scientific Machine Learning}, pp.\
  270--290. PMLR, 2022.

\bibitem[Wojtowytsch et~al.(2020)]{wojtowytsch2020emergence}
Wojtowytsch, S. et~al.
\newblock On the emergence of simplex symmetry in the final and penultimate
  layers of neural network classifiers.
\newblock \emph{arXiv preprint arXiv:2012.05420}, 2020.

\bibitem[Yaras et~al.(2022)Yaras, Wang, Zhu, Balzano, and Qu]{yaras2022neural}
Yaras, C., Wang, P., Zhu, Z., Balzano, L., and Qu, Q.
\newblock Neural collapse with normalized features: A geometric analysis over
  the riemannian manifold.
\newblock \emph{arXiv preprint arXiv:2209.09211}, 2022.

\bibitem[Zhou et~al.(2022)Zhou, Li, Ding, You, Qu, and
  Zhu]{zhou2022optimization}
Zhou, J., Li, X., Ding, T., You, C., Qu, Q., and Zhu, Z.
\newblock On the optimization landscape of neural collapse under mse loss:
  Global optimality with unconstrained features.
\newblock \emph{arXiv preprint arXiv:2203.01238}, 2022.

\bibitem[Zhu et~al.(2021)Zhu, Ding, Zhou, Li, You, Sulam, and
  Qu]{zhu2021geometric}
Zhu, Z., Ding, T., Zhou, J., Li, X., You, C., Sulam, J., and Qu, Q.
\newblock A geometric analysis of neural collapse with unconstrained features.
\newblock \emph{Advances in Neural Information Processing Systems},
  34:\penalty0 29820--29834, 2021.

\bibitem[Zou et~al.(2021)Zou, Cao, Li, and Gu]{zou2021understanding}
Zou, D., Cao, Y., Li, Y., and Gu, Q.
\newblock Understanding the generalization of adam in learning neural networks
  with proper regularization.
\newblock \emph{arXiv preprint arXiv:2108.11371}, 2021.

\end{thebibliography}
\bibliographystyle{icml2023}

\appendix
\onecolumn

\section{Proof of \cref{thm:main}}\label{sec:the-proof-of-theory}

Recall that we use unhinged loss $\ell(f({\boldsymbol{x}}),y) = -\frac{yf({\boldsymbol{x}})}{2}$ whose gradient is \begin{equation}\frac{\partial \ell\left(f({\boldsymbol{x}}),y\right)}{\partial f({\boldsymbol{x}})} = \frac{-y }{2},\end{equation}

and our goal is to prove \begin{equation}\left\|h({\boldsymbol{x}}_{i_1}) - h({\boldsymbol{x}}_{i_2})\right\| \ll \|h({\boldsymbol{x}}_{i_1}) - h({\boldsymbol{x}}_{i_3})\|.\end{equation}

\subsection{Concentration Lemmas Used}

We first introduce some concentration lemmas that we will use.

\begin{lemma}[\citet{gaussian-abs-sum-tail-bound}]\label{lem:concentration-sum-abs}
Suppose $\{\psi_k\}_{k=1}^n$ are a set of independent Gaussian variables such that $\psi_k \sim \mathcal N(0, \sigma_k^2)$. Then for any $t > 0$,
\begin{align}\mathbb P\left\{\sum_{k=1}^n |\psi_k| \geq t\right\} \leq \exp\left(-\frac{t^2}{2\sum_{k=1}^n \sigma_k^2} + n \log2\right).\end{align}
In other words, with probability at least $\exp(-n)$, the following inequality holds: \begin{equation}\mathbb \sum_{k=1}^n |\psi_k| \geq 3 \sqrt {n \sum_{i=1}^n \sigma_k^2}. \end{equation}
\end{lemma}

\begin{lemma}[Lemma 4 in \citet{shen2022data}]\label{lem:tail-bound-gaussian-product}
If  ${\boldsymbol{z}}_1 \sim \mathcal N(0,\sigma_1^2 {\boldsymbol{I}})$ and ${\boldsymbol{z}}_2 \sim \mathcal  N\left(0,\sigma_2^2 {\boldsymbol{I}}\right)$ are $d$-dimensional independent Gaussian vectors, then , we have  \begin{equation}
\mathbb P \left\{ \left|{\boldsymbol{z}}_1 ^\top {\boldsymbol{z}}_2\right| \geq 4 \sigma_1 \sigma_2 \sqrt{d \log(2/\delta)} \right\} \leq \delta
\end{equation}
\end{lemma}

\begin{lemma}[Corollary 3 in \citet{shen2022data}]\label{lem:tail-bound-gaussian-norm}
If  ${\boldsymbol{z}} \sim \mathcal N(0,\sigma^2 {\boldsymbol{I}})$ is a $d$-dimensional Gaussian vector, then for large enough $d$, and $\delta > 2e^{-d/64}$ we have  \begin{equation}
\frac{1}{2} \sigma^2 d \leq \|{\boldsymbol{z}}\|^2 \leq 2 \sigma^2 d
\end{equation}
\end{lemma}

\begin{lemma}[Union Bound]\label{lem:summation-tail}
If there are $n$ variables $\{z_k\}_{k=1}^n$ (not necessarily independent), and each $z_k$ satisfies \begin{equation}\mathbb P\left\{z_k \geq \epsilon_k(\delta)\right\} \leq \delta,\end{equation}
then \begin{equation}\mathbb P\left\{ \sum_{k=1}^n z_k \geq \sum_{k=1}^n\epsilon_k\left(\delta / n\right)  \right\} \leq \delta\end{equation}
\end{lemma}
\begin{proof}
We have that \begin{align}
\mathbb P\left\{ \sum_{k=1}^n z_k \leq \sum_{k=1}^n \epsilon_k\left(\delta / n\right)  \right\} & \geq \mathbb P\bigcap_{k=1}^n\left\{z_k \leq \epsilon_k(\delta/n)\right\}
\\ & \geq \sum_{k=1}^n \mathbb P\left\{z_k \leq \epsilon_k(\delta / n)\right\} - n + 1
\\ & \geq n - n \times \frac{\delta}{n} - n + 1
\\ & = 1 - \delta.
\end{align}
\end{proof}

\begin{lemma}[Lemma 5 in \citet{shen2022data}]\label{lem:bound-gaussian-max}
    If there are i.i.d samples $\{z_i\}_{i=1}^N$, where $z_i \sim \mathcal N(0,\sigma^2)$, then with probability at least $1-\delta$ we have \begin{align}
        \max_{i=1}^N |z_i| \leq \sigma\sqrt{2 \log\frac{2N }{\delta}}.
    \end{align}
\end{lemma}

\begin{lemma}[Anti-Concentration]\label{lem:gaussian-anti-concentrate}
If a random variable $z \sim \mathcal N(0, \sigma^2)$, then for any $\delta > 0$, \begin{equation}
\mathbb P\left\{-\delta < z < \delta\right\} < \frac{\delta}{\sigma},
\end{equation}
in other words, with probability higher than $1 - \delta$, we have $|z| \geq \sigma \delta$.
\end{lemma}\begin{proof}
    \begin{align}\mathbb P\left\{-\delta < z < \delta\right\} & = \int_{-\delta}^\delta \frac{1}{\sigma\sqrt{2\pi}} \exp\left(-\frac{x^2}{2\sigma^2} \right) \mathrm{d} x
    \\ & \leq \int_{-\delta}^\delta \frac{1}{\sigma\sqrt{2\pi}} \mathrm{d} x
    \\ & = \sigma \delta \frac{2}{\sqrt{2\pi}}
    \\ & < \sigma \delta.
    \end{align}
\end{proof}

\subsection{The Dynamics of Model Parameter}\label{sec:dynamic-model-parameter}

In this section, we consider the dynamics of the projection of the model parameters on the ``cluster mean'' direction and the ``noise'' direction. Specifically, let ${\boldsymbol{\mu}}_k = {\boldsymbol{\mu}}^{(p)}$ for $k \in {\mathscr C}_p$, we will investigate these two quantities: $\zeta_k(t) = {\boldsymbol{w}}_r(t)^\top {\boldsymbol{\mu}}_k$ and $\phi_k(t) = {\boldsymbol{w}}_r(t)^\top {\boldsymbol{\xi}}_i$. Notice that $\zeta_k$ is only dependent on the cluster mean of the cluster which ${\boldsymbol{x}}_k$ belongs, so if $k \in \mathscr C_s$, we denote $\zeta_k$ by $\zeta^{(s)}$. Notice that $\zeta_k$ and $\phi_k$ are actually depend on neuron index $r$. Throughout this section we fix a neuron index $r$ to perform the analysis.

Notice that since we have assumed $f({\boldsymbol{x}})$ is initialized very small, the activation function $\sigma$ will fall in the interval $[-1,1]$, and hence be a cubic function. Hereinafter unless explicitly mentioned, we will simply use $\sigma(z) = \frac{1}{3}z^3$ and therefore $\sigma'(x) = x^2$. This simplification will be rigorously justified in \cref{thm:main}.

For gradient descent with step size $\eta$, the update of ${\boldsymbol{w}}_r(t)$ in each step is $\Delta {\boldsymbol{w}}(t) = -\eta \frac{\partial L({{W}})}{\partial {\boldsymbol{w}}_r(t)}$. For other quantities (i.e. $\zeta, \phi$), we will use $\Delta$ to denote the update. For example $\Delta \zeta^{(s)}(t) = \Delta {\boldsymbol{w}}_r(t)^\top {\boldsymbol{\mu}}^{(s)}$.

The gradient of target function w.r.t.  ${\boldsymbol{w}}_r(t)$ is: \begin{align}
\frac{1}{\eta}\Delta{\boldsymbol{w}}_r(t) & = - \frac{\partial L(W)}{\partial {\boldsymbol{{\boldsymbol{W}}}}_r(t)}
 \\ & = -\frac{1}{n} \sum_{i=1}^n \ell'(f({\boldsymbol{x}}_i;{\boldsymbol{w}}_r(t)),y_i) \frac{\partial f({\boldsymbol{x}}_i,{\boldsymbol{w}}_r(t))}{\partial {\boldsymbol{w}}_r(t)}
\\ & = -\frac 1n \sum_{i=1}^n\ell'(f({\boldsymbol{x}}_i;{\boldsymbol{w}}_r(t)),y_i) \sigma'({\boldsymbol{x}}_i^\top {\boldsymbol{w}}_r(t)){\boldsymbol{x}}_i
\\ & = \frac{1}{2n} \sum_{i=1}^n y_i\sigma'({\boldsymbol{x}}_i^\top {\boldsymbol{w}}_r(t)){\boldsymbol{x}}_i
\\ & = \frac 1{2n} \sum_{i=1}^n y_i \sigma'\left({\boldsymbol{\mu}}_i^\top {\boldsymbol{w}}_r(t) + {\boldsymbol{\xi}}_i^\top {\boldsymbol{w}}_r(t)\right)({\boldsymbol{\mu}}_i + {\boldsymbol{\xi}}_i)
\\ & =  \frac{1}{2n}\sum_{p = 1}^4\sum_{{\boldsymbol{x}}_i \in {\mathscr C}_p} y_i\sigma'\left({\boldsymbol{w}}_r(t)^\top {\boldsymbol{\mu}}^{(p)} + {\boldsymbol{w}}_r(t)^\top {\boldsymbol{\xi}}_i\right)({\boldsymbol{\mu}}^{(p)} + {\boldsymbol{\xi}}_i).
\end{align}

Next, we consider the projection of ${\boldsymbol{w}}_r(t)$ onto the directions of ${\boldsymbol{\mu}}^{(s)}$ and ${\boldsymbol{\xi}}_k$ separately.

\begin{lemma}\label{lem:w_on_mu}
For $s \in \{1,2\}$, we have the following inequality holds with probability at least $1 - \exp(-n)$:
\begin{equation}\left|\frac{1}{\eta}\Delta\zeta^{(s)}(t) - \frac{\tau^2}{2n}\sum_{{\boldsymbol{x}}_i \in {\mathscr C}_s} \sigma'\left[\zeta^{(s)}(t) + \phi_i(t)\right] \right| \leq \frac 3{2}\mathcal L \tau \kappa.\end{equation}
\end{lemma}
\begin{proof}
Notice that since $s \in \{1,2\}$, we have $y_i = 1$ for $i \in {\mathscr C}_s$.
We have that \begin{align}2n \times \frac{1}{\eta}\Delta{\boldsymbol{w}}_r(t)^\top {\boldsymbol{\mu}}^{(s)}  = &\sum_{p=1}^4 \sum_{{\boldsymbol{x}}_i \in {\mathscr C}_p} y_i \sigma'\left({\boldsymbol{w}}_r(t)^\top {\boldsymbol{\mu}}^{(p)} + {\boldsymbol{w}}_r(t)^\top {\boldsymbol{\xi}}_i\right)\left({\boldsymbol{\mu}}^{(p)\top} {\boldsymbol{\mu}}^{(s)} + {\boldsymbol{\xi}}_i^\top {\boldsymbol{\mu}}^{(s)}\right) 
\\ = &\sum_{{\boldsymbol{x}}_i \in {\mathscr C}_s}  \sigma'\left({\boldsymbol{w}}_r(t)^\top {\boldsymbol{\mu}}^{(s)} + {\boldsymbol{\xi}}_i^\top {\boldsymbol{w}}_r(t)\right) \tau^2+ \sum_{i=1}^n y_i \sigma'\left({\boldsymbol{w}}_r(t)^\top {\boldsymbol{\mu}} ^{(s)} + {\boldsymbol{w}}_r(t)^\top {\boldsymbol{\xi}}_i\right) {\boldsymbol{\xi}}_i^\top {\boldsymbol{\mu}}^{(s)}
\\ \in & \sum_{{\boldsymbol{x}}_i \in {\mathscr C}_s} \sigma'\left({\boldsymbol{w}}_r(t)^\top {\boldsymbol{\mu}}^{(s)} + {\boldsymbol{\xi}}_i^\top {\boldsymbol{w}}_r(t)\right) \tau^2 \pm \left( \mathcal L \sum_{i=1}^n  \left|{\boldsymbol{\xi}}_i^\top {\boldsymbol{\mu}}^{(s)}\right|\right)
\\ =  & \sum_{{\boldsymbol{x}}_i \in {\mathscr C}_s} \sigma'\left({\boldsymbol{w}}_r(t)^\top {\boldsymbol{\mu}}^{(s)} + {\boldsymbol{\xi}}_i^\top {\boldsymbol{w}}_r(t)\right) \tau^2 \pm \mathcal L \beta \label{stat:_1},
\end{align}
where  $\beta =\sum_{i=1}^n  |{\boldsymbol{\xi}}_i^\top {\boldsymbol{\mu}}^{(s)}|$. Since ${\boldsymbol{\xi}}_i^\top {\boldsymbol{\mu}}^{(s)} \sim \mathcal N(0, \kappa^2\tau^2)$, by \cref{lem:concentration-sum-abs}, we have that \begin{equation}\mathbb P\left\{\beta \geq 3n\kappa \tau \right\} \leq \exp(-n) .\label{stat:_2}\end{equation}

Combining \cref{stat:_1,stat:_2}, we have that with probability at least $1 - \exp(-n)$, \begin{align}
2n \times \frac{1}{\eta}\Delta{\boldsymbol{w}}_r(t)^\top {\boldsymbol{\mu}}^{(s)} & \in  \tau^2 \sum_{{\boldsymbol{x}}_i \in {\mathscr C}_s} \sigma'\left({\boldsymbol{w}}_r(t)^\top {\boldsymbol{\mu}}^{(s)} + {\boldsymbol{\xi}}_i^\top {\boldsymbol{w}}_r(t)\right) \pm  3n\mathcal L \tau \kappa  ,
\end{align}
which proves the proposition.
\end{proof}

\begin{lemma}\label{lem:w_on_xi}
For $k \in \mathscr C_1 \cup \mathscr C_2$, if there exists a constant $c^{(n)} > 0$ such that $n \leq c^{(n)} \frac{\kappa d}{\tau}$, then we have \begin{equation}
\mathbb P\left\{ \frac{\kappa^2d}{4n} \sigma'(\phi_k(t) + \zeta_k(t)) -  \frac{5}{2}\kappa \tau \mathcal L \leq \frac{1}{\eta}\Delta \phi_k(t) \leq \frac{3\kappa^2d}{4n} \sigma'(\phi_k(t) + \zeta_k(t)) + \frac{5}{2}\kappa \tau \mathcal L \right\} \geq 1 - \delta,
\end{equation}
where $\delta = \exp(-n) + 2n \exp\left(-\frac{1}{\mathfrak  c^2}\right) + 2\exp(-d/64)$.
\end{lemma}
\begin{proof}

For the noise term ${\boldsymbol{\xi}}_k$, notice that since $k \in \mathscr C_1 \cup \mathscr C_2$, we have $y_k = 1$, and 
\begin{align}2n  \times \frac{1}{\eta}\Delta{\boldsymbol{w}}_r(t)^\top {\boldsymbol{\xi}}_k  = & \sigma'\left[{\boldsymbol{\mu}}_k^\top {\boldsymbol{w}}_r(t) + {\boldsymbol{\xi}}_k^\top {\boldsymbol{w}}_r(t)\right]\|{\boldsymbol{\xi}}_k\|^2+
\sigma'\left[{\boldsymbol{\mu}}_k^\top {\boldsymbol{w}}_r(t) + {\boldsymbol{\xi}}_k^\top {\boldsymbol{w}}_r(t)\right]{\boldsymbol{\xi}}_k^\top {\boldsymbol{\mu}}_k \\  & +   \sum_{p=1}^4\sum_{{\boldsymbol{x}}_i \in {\mathscr C}_p \atop i \neq k} y_i \sigma'({\boldsymbol{w}}_r(t)^\top {\boldsymbol{\mu}}^{(p)} + {\boldsymbol{w}}_r(t)^\top {\boldsymbol{\xi}}_i)\left({\boldsymbol{\xi}}_k^\top {\boldsymbol{\mu}}^{(p)} + {\boldsymbol{\xi}}_k^\top {\boldsymbol{\xi}}_i \right)
\\ \in & \sigma'\left[{\boldsymbol{\mu}}_k^\top {\boldsymbol{w}}_r(t) + {\boldsymbol{\xi}}_k^\top {\boldsymbol{w}}_r(t)\right]\|{\boldsymbol{\xi}}_k\|^2 \pm \left[\mathcal L\left|{\boldsymbol{\xi}}_k^\top {\boldsymbol{\mu}}_k\right| + \sum_{p=1}^4\sum_{{\boldsymbol{x}}_i \in {\mathscr C}_p \atop i \neq k} \mathcal L\left|{\boldsymbol{\xi}}_k^\top {\boldsymbol{\mu}}^{(p)} + {\boldsymbol{\xi}}_k^\top {\boldsymbol{\xi}}_i \right|\right]
\\ \subseteq & \sigma'\left[{\boldsymbol{\mu}}_k^\top {\boldsymbol{w}}_r(t) + {\boldsymbol{\xi}}_k^\top {\boldsymbol{w}}_r(t)\right]\|{\boldsymbol{\xi}}_k\|^2 \pm \left[\sum_{i=1}^n\mathcal L\left|{\boldsymbol{\xi}}_i^\top {\boldsymbol{\mu}}_i\right| + \sum_{1 \leq i \leq n \atop i \neq k}\mathcal L\left| {\boldsymbol{\xi}}_k^\top {\boldsymbol{\xi}}_i \right|\right]
\\ = & \sigma'\left[{\boldsymbol{\mu}}_k^\top {\boldsymbol{w}}_r(t) + {\boldsymbol{\xi}}_k^\top {\boldsymbol{w}}_r(t)\right]\|{\boldsymbol{\xi}}_k\|^2 \pm \mathcal L \beta
\end{align}
where \begin{align}\beta =  \sum_{i=1}^n \left|{\boldsymbol{\xi}}_k^\top {\boldsymbol{\mu}}_i\right| + \sum_{1 \leq i \leq n \atop i \neq k} \left|{\boldsymbol{\xi}}_k^\top {\boldsymbol{\xi}}_i\right|.\end{align}

Notice that $n \leq \frac{c^{(n)} \kappa d}{\tau}$. Let $\delta_1 = \exp(-n)$. From Lemma \ref{lem:concentration-sum-abs}, with probability at least $1-\delta_1$, we have \begin{equation}\mathbb \sum_{i=1}^n\left|{\boldsymbol{\xi}}_k^\top {\boldsymbol{\mu}}_i\right| \leq 3\tau\kappa n . \label{eq:delta-1}\end{equation}
From Lemma \ref{lem:tail-bound-gaussian-product} and Lemma \ref{lem:concentration-sum-abs}, we have that with probability at least $1 - \delta_2$, we have \begin{equation}\sum_{i \neq k} \left|{\boldsymbol{\xi}}_k^\top {\boldsymbol{\xi}}_i\right|  \leq 2n\kappa^2 \sqrt{d \log (2n/\delta_2)}.
\label{eq:delta-2}\end{equation}
Take $\delta_2 = 2n \exp\left(- \frac{1}{\mathfrak  c^2}\right)$, we get \begin{align}
\sum_{i \neq k} \left|{\boldsymbol{\xi}}_k^\top {\boldsymbol{\xi}}_i\right| & \leq 2n\kappa^2 \sqrt{d \log (2n/\delta_2)}.
\\ & \leq 2n\kappa^2 \sqrt{d \times \frac{\tau^2}{\kappa^2 d}}
\\ & = 2n \kappa \tau
\end{align}
To summery, with probability at least $1 - \delta_1 - \delta_2$, we have \begin{equation}\frac{\mathcal L}{2n}\beta \leq \frac{5}{2} \kappa \tau \mathcal L.\label{eq:beta-phi-bound}\end{equation}.

From Lemma \ref{lem:tail-bound-gaussian-norm}, with $\delta_3 > 2\exp(-d / 64)$, we have \begin{equation}
\frac{1}{2} \kappa^2d \leq \|{\boldsymbol{\xi}}_k\|^2 \leq \frac{3}{2} \kappa^2d. \label{eq:xi_k-norm}
\end{equation}

Combining \cref{eq:beta-phi-bound,eq:xi_k-norm}, the proposition is proved.

\end{proof}

To summarize, \cref{lem:w_on_mu} shows that with probability at least $1-\delta_1$, where $\delta_1 = -\exp(-n)$, we have\begin{equation}\zeta^{(s)}(t+1) - \zeta^{(s)}(t) \in \frac{\tau^2\eta}{2n}\sum_{{\boldsymbol{x}}_i \in {\mathscr C}_s} \sigma'\left[\zeta ^{(s)}(t) + \phi_i(t)\right] \pm \frac{\tau^2\eta}{2} \times \frac{3\mathcal L \kappa}{\tau} \label{eq:descrete-zeta}\end{equation}
and with probability at least $1 - \delta_2$, where $\delta_2 = \exp(-n) + 2n\exp\left(- 
 c^{-2}\right) + 2 \exp(-d/64)$, we have \begin{align}
\phi_k(t + 1) - \phi_k(t) & = \eta \frac{\mathrm{d} }{\mathrm{d} t} \phi_k(t) 
\\ & \in \eta \left[\frac{\kappa^2d}{4n} \sigma'(\phi_k(t) + \zeta_k(t)) - \frac{5}{2} \kappa \tau \mathcal L ,  \frac{3\kappa^2d}{4n} \sigma'(\phi_k(t) + \zeta_k(t)) + \frac{5 c^{(n)} \kappa^2 d \mathcal L}{n}\right]
\\ & \subseteq \eta \left[\frac{\kappa^2d}{4n} \sigma'(|\phi_k(t)| + |\zeta_k(t)|) - \frac{5 c^{(n)} \kappa^2 d \mathcal L}{n} ,  \frac{3\kappa^2d}{4n} \sigma'(|\phi_k(t)| + |\zeta_k(t)|) + \frac{5}{2} \kappa \tau \mathcal L \right].
\label{eq:descrete-phi}
\end{align}

\begin{lemma}
Let $s \in \{1,2\}$. Suppose $\zeta^{(s)}$ is initialized by $\zeta^{(s)}(0)$, and there exists constants $c^{(t)} \in \left(0,\frac{8}{1 + 8  \mathfrak  c}\right)$ and $C^{(\phi)} > 0$ such that the following conditions hold for $t_0 < c^{(t)} \left(\tau^2 \eta \left|\zeta^{(s)}(0)\right|\right)^{-1}$: \begin{enumerate}
    \item $\forall t \leq t_0$, if $\zeta^{(s)}(0) > 0$, we have $\left[\zeta^{(s)}(0)^{-1} - \left(
    \frac{1}{8} - \mathfrak  c\right)\eta \tau^2 t\right]^{-1} \leq \zeta^{(s)}(t) \leq \left[\zeta^{(s)}(0)^{-1} - \left(\frac{1}{8} + \mathfrak  c\right) \eta \tau^2 t\right]^{-1}$, while if $\zeta^{(s)}(0) < 0$, we have $\zeta^{(s)}(0) \leq \zeta^{(s)}(t) \leq \zeta^{(s)}(0) + \left(\frac{1}{8} + \mathfrak c\right)\eta \tau^2 t |\zeta^{(s)}(0)|^2$, \label{stat:cond-zeta-1}
    \item $\forall t \leq t_0, \forall i \leq n$, $|\phi_i(t)| \leq \frac{C^{(\phi)} \mathfrak  c^2}{n} \left|\zeta^{(s)}(t)\right | + \frac{\mathfrak  c}{8}|\zeta^{(s)}(0)|$; \label{stat:cond-zeta-2}
    \item $ \sqrt{ \frac{32 \mathcal Ln}{\mathfrak  c^2\sqrt{d}}} \leq |\zeta^{(s)}(0)| \leq \frac{\mathfrak  c\left(8 - {c^{(t)}}\right)}{16 \eta \tau^2}$; \label{stat:cond-zeta-3}
    \item $\mathfrak  c \leq \min\left\{\frac{nC^{(\phi)}}{4}, \left(  \frac{n}{48}\right)^{\frac{1}{3}},\frac{1}{8}\right\}$, \label{stat:cond-zeta-4}
\end{enumerate}
and $\zeta^{(s)}$ is updated as described in \cref{eq:descrete-zeta}, then we have:
\begin{enumerate}[label={\alph*)}]
    \item $\forall t \leq t_0, \left|\zeta^{(s)}(t)\right| > \frac{1}{2} \left|\zeta^{(s)}(0)\right|$; \label{stat:res_zeta_1}
    \item $ \left(\frac{1}{8} - \frac{\mathfrak  c}{2}\right) \eta \tau^2 \zeta^{(s)}(t_0)^2 \leq \zeta^{(s)}(t_0 + 1) - \zeta^{(s)}(t_0) \leq  \left(\frac{1}{8} + \frac{\mathfrak  c}{2}\right) \eta \tau^2 \zeta^{(s)}(t_0)^2$; \label{stat:zeta-res-2}
    \item If $\zeta^{(s)}(0) > 0$, we have \begin{align}    \left[\zeta^{(s)}(0)^{-1} - \left(
    \frac{1}{8} - \mathfrak  c\right)\eta \tau^2 (t_0 + 1)\right]^{-1} \leq \zeta^{(s)}(t_0 + 1) \leq \left[\zeta^{(s)}(0)^{-1} - \left(\frac{1}{8} + \mathfrak  c\right) \eta \tau^2 (t_0 + 1)\right]^{-1}
,\end{align}while if $\zeta^{(s)}(0) < 0$ we have $\zeta^{(s)}(0) \leq \zeta^{(s)}(t_0 + 1) \leq \zeta^{(s)}(0) + \eta \tau^2 (t_0 + 1) |\zeta^{(s)}(0)|^2$. \label{stat:res-zeta-3}
\end{enumerate}
\label{lem:bound-zeta}
\end{lemma}
\begin{proof}~ \begin{itemize}
    \item  First, notice that since $t < t_0 < c^{(t)}\left(\eta \tau^2 |\zeta^{(s)}(0)|\right)^{-1}$, we have $\eta \tau^2 t |\zeta^{(s)}(0)|^2 \leq c^{(t)} |\zeta^{(s)}(0)|$, and \begin{align}    \zeta^{(s)}(t) \leq \zeta^{(s)}(0) + \left(\frac{1}{8} + \mathfrak c\right)\eta\tau^2 t|\zeta^{(s)}(0)|^2 \leq \zeta^{(s)}(0) + c^{(t)} \left(\frac{1}{8} + \mathfrak c\right)|\zeta^{(s)}(0)| \leq 0.\end{align}
    If $\zeta^{(s)}(0) < 0$, then \begin{align}
       |\zeta^{(s)}(t)| \geq |\zeta^{(s)}(0)| - \eta\tau^2 t |\zeta^{(s)}(0)|^2 \geq \left(1 - c^{(t)}\right)|\zeta^{(s)}(0)|  \geq \frac{1}{2} |\zeta^{(s)}(0)|. \label{eq:zeta_k_t_geq_zeta_k_0_neg}
    \end{align}
     While if $\zeta^{(s)}(0) > 0$, then \begin{align}
        \left|\zeta^{(s)}(t)\right| \geq  \left[\left|\zeta^{(s)}(0)\right|^{-1} - \left(\frac{1}{8} - \mathfrak  c\right)\tau^2 \eta t\right]^{-1} \geq |\zeta^{(s)}(0)|. \label{eq:zeta_k_t_geq_zeta_k_0_pos}
    \end{align}
    \cref{eq:zeta_k_t_geq_zeta_k_0_neg} and \cref{eq:zeta_k_t_geq_zeta_k_0_pos} together proves Result \ref{stat:res_zeta_1}.
    \item  Notice that from Condition \ref{stat:cond-zeta-4} we have $\frac{\mathfrak  c^2 C^{(\phi)}}{n } \leq \frac{\mathfrak c}{4}$. From Condition \ref{stat:cond-zeta-2}, we have \begin{align}
        \zeta^{(s)}(t) + \phi_i(t) \geq \left| \zeta^{(s)}(t)\right| - \left|\phi_i(t)\right| \geq \left(1-\frac{\mathfrak c}{4}\right) \left| \zeta^{(s)}(t)\right|- \frac{\mathfrak  c}{8} |\zeta^{(s)}(0)| \geq \left(1 - \frac{\mathfrak  c}{2}\right)|\zeta^{(s)}(t)|
    \end{align}
    and \begin{align}
            \zeta^{(s)}(t) + \phi_i(t) \leq \left| \zeta^{(s)}(t)\right| + \left|\phi_i(t)\right| \leq \left(1+\frac{\mathfrak c}{4}\right) \left| \zeta^{(s)}(t)\right| + \frac{\mathfrak  c}{8} |\zeta^{(s)}(0)| \leq \left(1 + \frac{\mathfrak  c}{2}\right)|\zeta^{(s)}(t)|.
    \end{align}
    Therefore we have \begin{align}
       \frac{1}{n}\sum_{{\boldsymbol{x}}_i \in {\mathscr C}_s} \sigma'\left[\zeta^{(s)}(t) + \phi_i(t)\right] & = \frac{1}{n} \sum_{{\boldsymbol{x}}_i \in {\mathscr C}_s}\left[\zeta ^{(s)}(t) + \phi_i(t)\right]^2
       \\ & \in \frac{1}{n} \sum_{{\boldsymbol{x}}_i \in {\mathscr C}_s}  \left(1 \pm \frac{\mathfrak c}{2}\right)^2  \zeta^{(s)}(t)^2
       \\ & = \frac{1}{4} \left(1 \pm \frac{\mathfrak c}{2}\right)^2  \zeta^{(s)}(t)^2. 
       \\ & \subseteq  \frac{1}{4} (1 \pm 2 \mathfrak  c)  \zeta^{(s)}(t)^2. \label{eq:bound-mean-sigma-eta-plus-phi}
    \end{align}

    From Condition \ref{stat:cond-zeta-3}, Condition \ref{stat:cond-zeta-4} and Result \ref{stat:res_zeta_1}, we have \begin{align}
       \frac{3 \mathcal L  \kappa}{\tau}  \leq \frac{\mathfrak c}{8}\zeta^{(s)}(0)^2 \leq \frac{\mathfrak  c}{2} \zeta^{(s)}(t)^2 . \label{eq:bound-update-zeta-noise}
    \end{align}
    
    Subtracting \cref{eq:bound-mean-sigma-eta-plus-phi} and \cref{eq:bound-update-zeta-noise} into the update rule \cref{eq:descrete-zeta}, we have \begin{align}
        \zeta^{(s)}(t + 1) - \zeta^{(s)}(t) & \in \frac{\tau^2\eta}{2}\left(\frac{1}{4} \pm \frac{\mathfrak  c}{2}\right) \zeta^{(s)}(t)^2 \pm \frac{\tau^2 \eta \mathfrak  c}{4} \zeta^{(s)}(t)^2
        \\ & =  \tau^2 \eta\zeta^{(s)}(t)^2 \times \left(\frac{1}{8} \pm \frac{\mathfrak  c}{2}\right),
    \end{align}
    which proves Result \ref{stat:zeta-res-2}.
    \item For the greater-or-equal part of Result \ref{stat:res-zeta-3}, if $\zeta^{(s)}(0) < 0,$ then we have $\zeta^{(s)}(t_0+1)  \geq  \zeta^{(s)}(t_0) \geq \zeta^{(s)}(0)$. In the following we assume $\zeta^{(s)}(0) > 0$.
    From Result \ref{stat:zeta-res-2} we have \begin{align}
        \zeta^{(s)}(t_0 + 1) & \geq \zeta^{(s)}(t_0) + \left(\frac{1}{8} - \frac{\mathfrak  c}{2}\right) \eta \tau^2 \zeta^{(s)}(t_0)^2
        \\ & \geq \frac{1}{\zeta^{(s)}(0)^{-1} - \left(\frac{1}{8}- \mathfrak  c\right) \tau^2 \eta t_0} + \frac{\left(\frac{1}{8} - \frac{\mathfrak  c}{2}\right) \tau^2 \eta}{\left(\zeta^{(s)}(0)^{-1} - \left(\frac{1}{8} - \mathfrak  c\right)\tau^2 \eta t_0\right)^2}
        \\ & = \frac{\zeta^{(s)}(0)^{-1} - \left(\frac{1}{8} - \mathfrak  c\right) \tau^2 \eta t_0  + \left(\frac 18 - \frac{\mathfrak  c}{2}\right) \tau^2 \eta}{\left(\zeta^{(s)}(0)^{-1} - \left(\frac{1}{8} - \mathfrak  c\right)\tau^2 \eta t_0\right)^2}
        \\ & \overset{\text{(i)}}{\geq}   \frac{1}{\zeta^{(s)}(0)^{-1} - \left(\frac{1}{8} - \mathfrak  c\right) \tau^2 \eta (t_0 + 1)},
    \end{align}
    which proves the greater-or-equal part of Result \ref{stat:res-zeta-3}. To see (i), let $A = \zeta^{(s)}(0)^{-1} - \left(\frac{1}{8} - \mathfrak  c\right) \tau^2 \eta t_0$. Since $t_0 < c^{(t)} \left[\tau^2 \eta |\zeta^{(s)}(0)|\right]^{-1}$ and $\zeta^{(s)}(0) \leq \frac{\mathfrak  c\left(8 - {c^{(t)}}\right)}{16 \eta \tau^2}$, we have 
        \begin{align}     A \geq  \left[1 - c^{(t)}\left(\frac{1}{8} - \mathfrak  c\right)\right] \zeta^{(s)}(0)^{-1} \geq \left(1 - \frac{c^{(t)}}{8}\right)\zeta^{(s)}(0)^{-1} \geq \frac{2}{\mathfrak  c} \eta \tau^2, \label{eq:claim-bound-A}
        \end{align}    
    and  \begin{align}
        \left[A + \left(\frac{1}{8} - \mathfrak  c\right) \tau^2 \eta+ \frac{\mathfrak  c}{2}\tau^2 \eta\right]\left[A - \left(\frac{1}{8} - \mathfrak  c\right)\tau^2 \eta\right] - A^2 & = \frac{\mathfrak  c}{2} A \tau^2 \eta - \left(\frac{1}{8} - \mathfrak  c\right)^2 \tau^4 \eta^2 - \frac{\mathfrak  c}{2} \tau^4 \eta^2
        \\ & \geq \tau^4 \eta^2\left[ 1 - \left( \left[\frac{1}{8} - \mathfrak  c\right]^2 + \frac{\mathfrak  c}{2} \right) \right]
        \\ & \geq 0,
    \end{align}
    which proves the greater-or-equal part of Result \ref{stat:res-zeta-3}.

    \item For the less-or-equal part of Result \ref{stat:res-zeta-3}, if $\zeta^{(s)}(0) < 0$, simply notice that $|\zeta^{(s)}(t_0)| \leq |\zeta^{(s)}(0)|$, so  \begin{align}
    \zeta^{(s)}(t_0+ 1) \leq \zeta^{(s)}(0) + \left(\frac{1}{8} + {\mathfrak c}\right)\eta \tau^2 t_0 \zeta^{(s)}(0)^2 + \left(\frac{1}{8} + {\mathfrak c}\right)\eta \tau^2  \zeta^{(s)}(t_0)^2 \leq \zeta^{(s)}(0) + \left(\frac{1}{8} + {\mathfrak c}\right)\eta \tau^2 (t_0+1) \zeta^{(s)}(0)^2,
    \end{align}
    which proves the result. If $\zeta^{(s)}(0) > 0$, then we have  \begin{align}
        \zeta^{(s)}(t_0 + 1) & \leq \zeta^{(s)}(t_0) + \left(\frac{1}{8} + \frac{\mathfrak c}{2}\right) \eta \tau^2 \zeta^{(s)}(t_0)^2
        \\ & \leq \frac{1}{\zeta^{(s)}(0)^{-1} - \left(\frac{1}{8} + \mathfrak  c\right) \tau^2 \eta t_0} + \frac{\left(\frac{1}{8} + \frac{\mathfrak c}{2}\right) \tau^2 \eta}{\left(\zeta^{(s)}(0)^{-1} - \left(\frac{1}{8} + \mathfrak  c\right)\tau^2 \eta t_0\right)^2}
        \\ & = \frac{\zeta^{(s)}(0)^{-1} - \left(\frac{1}{8} + \mathfrak  c\right) \tau^2 \eta t_0 + \left(\frac{1}{8} + \frac{\mathfrak c}{2}\right) \tau^2 \eta}{\left(\zeta^{(s)}(0)^{-1} - \left(\frac{1}{8} + \mathfrak  c\right)\tau^2 \eta t_0\right)^2}
        \\ & \leq \frac{\zeta^{(s)}(0)^{-1} - \left(\frac{1}{8} + \mathfrak  c\right) \tau^2 \eta (t_0-1) }{\left(\zeta^{(s)}(0)^{-1} - \left(\frac{1}{8} + \mathfrak  c\right)\tau^2 \eta t_0\right)^2}
        \\ & \overset{\text{(ii)}}{<}   \frac{1}{\zeta^{(s)}(0)^{-1} - \left(\frac{1}{8} + \mathfrak  c\right) \tau^2 \eta (t_0 + 1)},
    \end{align}
    which proves the less-or-equal part of Result \ref{stat:res-zeta-3}. To see (ii), notice that \begin{align}
        & \left[\zeta^{(s)}(0)^{-1} - \left(\frac{1}{8} + \mathfrak  c\right) \tau^2 \eta (t_0-1)\right]\left[\zeta^{(s)}(0)^{-1} - \left(\frac{1}{8} + \mathfrak  c\right) \tau^2 \eta (t_0+1)\right] 
        \\  = & \left[\zeta^{(s)}(0)^{-1} - \left(\frac{1}{8} + \mathfrak  c\right) \tau^2 \eta t_0\right]^2 - \left(\frac{1}{8} + \mathfrak  c\right)^2 \tau^4 \eta^2 
        \\ < & \left[\zeta^{(s)}(0)^{-1} - \left(\frac{1}{8} + \mathfrak  c\right) \tau^2 \eta t_0\right]^2.
    \end{align}
\end{itemize}
\end{proof}

Next, we will consider the dynamics of ${\boldsymbol{w}}_r$ projected to the direction of ${\boldsymbol{\xi}}_i$.

\begin{lemma}
Let $s \in \{1,2\}$ and $k \in \mathscr C_s$. Suppose $\zeta_k$ is initialized as $\zeta_k(0)$ and $\phi_k(k)$ is initialized as $\phi_k(0)$, and there exists constants $c^{(t)} \in \left(0,\frac{8}{1 + 8 \mathfrak  c}\right)$ and $8  \leq C^{(\phi)} \leq n$ such that the following conditions hold for $t_0 \leq c^{(t)}\left(\tau^2 \eta |\zeta_k(0)| \right)^{-1}$:\begin{enumerate}
    \item $\forall t \leq t_0$, if $\zeta_k(0) > 0$, we have $\left[\zeta_k(0)^{-1} - \left(
    \frac{1}{8} - \mathfrak  c\right)\eta \tau^2 t\right]^{-1} \leq \zeta_k(t) \leq \left[\zeta_k(0)^{-1} - \left(\frac{1}{8} + \mathfrak  c\right) \eta \tau^2 t\right]^{-1}$, while if $\zeta_k(0) < 0$, we have $\zeta_k(0) \leq \zeta_k(t) \leq \zeta_k(0) + \left(\frac{1}{8} + \mathfrak c\right)\eta \tau^2 t |\zeta_k(0)|^2$, \label{stat:cond-phi-1}
    \item $ t \leq t_0$, $|\phi_k(t)| \leq \frac{C^{(\phi)} \mathfrak  c^2}{n} \left|\zeta_k(t)\right | +\frac{\mathfrak  c}{8} |\zeta_k(0)|$, and  $|\phi_k(0)| \leq \frac{\mathfrak  c}{8} |\zeta_k(0)|$; \label{stat:cond-phi-2}
    \item $  \sqrt{ \frac{32 \mathcal Ln}{\mathfrak  c^2\sqrt{d}}} \leq |\zeta_k(0)| \leq \frac{\mathfrak  c\left(8 - {c^{(t)}}\right)}{16 \eta \tau^2}$; \label{stat:cond-phi-3}
    \item $\mathfrak  c \leq \min\left\{\frac{1}{4} - \frac{2}{C^{(\phi)}}, \frac{nC^{(\phi)}}{4}, \left(\frac{n}{48}\right)^{\frac{1}{3}}, \frac{1}{8}\right\}$, \label{stat:cond-phi-4}
        \end{enumerate}
and $\phi_k$ is updated through \cref{eq:descrete-phi}, $\zeta_k$ is updated through \cref{eq:descrete-zeta}, then
\begin{equation}
|\phi_k(t_0 + 1)| \leq \frac{C^{(\phi)} \mathfrak  c^2}{n} \left|\zeta_k(t_0 + 1)\right | +\frac{ \mathfrak  c }{8} |\zeta_k(0)|.
\end{equation}
\label{lem:bound-phi}
\end{lemma}
\begin{proof}~ 

\begin{itemize}
    \item From Lemma \ref{lem:bound-zeta}, we have:    
\begin{equation}
    \forall t \leq t_0, \left|\zeta_k(t)\right| > \frac{1}{2} \left|\zeta_k(0)\right| 
\end{equation}
and \begin{align}
\left(\frac{1}{8} - \frac{\mathfrak  c}{2}\right) \eta \tau^2 \zeta_k(t_0)^2 \leq \zeta_k(t_0 + 1) - \zeta_k(t_0) \leq  \left(\frac{1}{8} + \frac{\mathfrak  c}{2}\right) \eta \tau^2 \zeta_k(t_0)^2.\label{eq:phi-bound-update-of-zeta}
\end{align}
\item Let $s = \frac{5}{2}\kappa \tau \mathcal L$, from Condition \ref{stat:cond-phi-4} and the update rule \cref{eq:descrete-phi} we have \begin{align}
 \phi_k(t+1) - \phi_k(t) \in \eta \left[\frac{\kappa^2d}{4n} \sigma'(|\phi_k(t)| + |\zeta_k(t)|) - s ,  \frac{3\kappa^2d}{4n} \sigma'(|\phi_k(t)| + |\zeta_k(t)|) + s\right].
\end{align} 

Given Condition \ref{stat:cond-phi-2} and Condition \ref{stat:cond-phi-4}, we have \begin{equation}
|\phi_k(t)| + |\zeta_k(t) | \leq \left(1 + \frac{C^{(\phi)} \mathfrak  c^2}{n} + \frac{\mathfrak  c}{4}\right) |\zeta_k(t)|\leq 2 |\zeta_k(t)|. \label{eq:phi-1}
\end{equation}

Since $\sqrt{ \frac{32 \mathcal Ln}{\mathfrak  c^2\sqrt{d}}} \leq |\zeta_k(0)| \leq 2 |\zeta_k(t)|$,  we have \begin{align}
s = \frac{5}{2}\kappa\tau\mathcal L \leq \frac{5 \kappa ^2 d }{ n} \zeta_k(0)^2 \leq \frac{20 \kappa^2 d}{ n}\zeta_k(t)^2. \label{eq:phi-noise-term}
\end{align}

Combining \cref{eq:phi-1,eq:phi-noise-term}we have \begin{align}
\frac{3\kappa^2d}{4n} \sigma'(\phi_k(t) + \zeta_k(t)) + s \leq \frac{3\kappa^2 d}{2n}\zeta(t)^2 + \frac{20 \kappa^2 d}{n}\zeta(t)^2 \leq  \frac{30 \kappa^2 d}{n}\zeta(t)^2 .
\end{align}
On the other hand, we have \begin{align}
\frac{\kappa^2d}{4n} \sigma'(\phi_k(t) - \zeta_k(t)) - s \geq \frac{\kappa^2 d}{2n}\zeta(t)^2 - \frac{20 \kappa^2 d}{n}\zeta(t)^2 \geq - \frac{30  \kappa^2 d}{2n}\zeta(t)^2 .
\end{align}
In summary we have \begin{align}
    |\phi_k(t + 1) - \phi_k(t)| \leq \frac{30 \eta \kappa^2 d}{n}\zeta(t)^2.
\end{align}
\item Notice that $C^{(\phi)} \geq \frac{1}{\frac{1}{8} - \frac{\mathfrak  c }{2}}$. From \cref{eq:phi-bound-update-of-zeta}, we know that $\zeta_k(t_0)^2 \leq \frac{C^{(\phi)}}{\eta \tau^2} \left[\zeta(t_0 + 1) -\zeta(t_0) \right]$.  If $\zeta_k(0) > 0$, we have
\begin{align}
\left|\phi_k(t + 1)\right| & \leq \left|\phi_k(t)\right| + |\phi_k(t + 1) - \phi_k(t)|
\\ & \leq \frac{ \mathfrak  c }{8}  \zeta_k(0) + \frac{C^{(\phi)} \kappa^2 d}{n \tau^2} \zeta_k(t) + \frac{30 \eta \kappa^2 d}{n}\zeta_k(t)^2
\\ & \leq  \frac{ \mathfrak  c }{8}  \zeta_k(0) + \frac{C^{(\phi)} \kappa^2 d}{n \tau^2}\zeta_k(t) + \frac{30 \eta \kappa^2 d}{n} \times \frac{C^{(\phi)}}{\eta \tau^2}\left[\zeta(t_0 + 1) -\zeta(t_0) \right]
\\ & = \frac{ \mathfrak  c }{8} \zeta_k(0) +  \frac{C^{(\phi)} \kappa^2 d}{n\tau^2}\zeta_k(t + 1).
\\ & \leq \frac{ \mathfrak  c }{8} \zeta_k(0) +  \frac{C^{(\phi) } \mathfrak  c^2}{n}\zeta_k(t + 1).
\end{align}
While if $\zeta_k(0) < 0$, then we have $|\zeta_k(t)| \leq |\zeta_k(0)|$, and since $c ^{(t)} \leq C^{(\phi)}$, we have
\begin{align}
\left|\phi_k(t + 1)\right| & \leq |\phi_k(0)| + \sum_{i=1}^{t_0} \frac{ \eta \kappa^2 d}{n} \zeta_k(t)^2
\\ & \leq \frac{ \mathfrak  c }{8}  |\zeta_k(0)|  + t_0 \times \frac{ \eta \kappa^2 d}{n} \zeta_k(0)^2
\\ & \leq \frac{ \mathfrak  c }{8}  |\zeta_k(0)|  +  \frac{c^{(t)}  \kappa^2 d}{n\tau^2} |\zeta_k(0)|
\\ & \leq \frac{ \mathfrak  c }{8}  |\zeta_k(0)|  +  \frac{2 c^{(t)}  \kappa^2 d}{n\tau^2} |\zeta_k(t+1)|
\\ & \leq \frac{ \mathfrak  c }{8}  |\zeta_k(0)|  + \frac{C^{(\phi)} \mathfrak  c^2}{n}|\zeta_k(t+1)|.
\end{align}

\end{itemize}

\end{proof}

Combining Lemma \ref{lem:bound-zeta} and Lemma \ref{lem:bound-phi}, we have the following conclusion:
\begin{corollary}\label{cor:main}
Let $s \in \{1,2\}$. Suppose $\zeta^{(s)}$ is initialized as $\zeta^{(s)}(0)$ and $\phi_k(k)$ is initialized as $\phi_k(0)$ for any $k \in \mathscr C_s$, and there exists constants $c^{(t)} \in \left(0,\frac{8}{1 + 8  \mathfrak  c}\right)$ and $8 \leq C^{(\phi)} \leq n$ such that the following conditions hold for  $t_0 \leq c^{(t)}\left(\tau^2 \eta |\zeta^{(s)}(0)| \right)^{-1}$ and any $k \in \mathscr C_s$:\begin{enumerate}
    \item $ \sqrt{ \frac{32 \mathcal Ln}{\mathfrak  c^2\sqrt{d}}} \leq |\zeta^{(s)}(0)| \leq \frac{\mathfrak  c\left(8 - {c^{(t)}}\right)}{16 \eta \tau^2}$;
    \item $|\phi_k(0)| \leq   \frac{\mathfrak  c}{8} |\zeta^{(s)}(0)|$; 
    \item $\mathfrak  c \leq \min\left\{\frac{1}{4} - \frac{2}{C^{(\phi)}}, \frac{nC^{(\phi)}}{4},\left(\frac{n}{48}\right)^{\frac{1}{3}} , \frac{1}{8}\right\}$;
        \end{enumerate}
and $\phi_k$ is updated through \cref{eq:descrete-phi}, $\zeta^{(s)}$ is updated through \cref{eq:descrete-zeta}, then we have for all $t \leq t_0$ and $k \in \mathscr C_s$, we have the following conclusions:
\begin{enumerate}[label={\alph*)}]
    \item if $\zeta^{(s)}(0) > 0$, then $\left[\zeta^{(s)}(0)^{-1} - \left(
    \frac{1}{8} - \mathfrak  c\right)\eta \tau^2 t\right]^{-1} \leq \zeta^{(s)}(t) \leq \left[\zeta^{(s)}(0)^{-1} - \left(\frac{1}{8} + \mathfrak  c\right) \eta \tau^2 t\right]^{-1}$, while if $\zeta^{(s)}(0) < 0$, then $\zeta^{(s)}(0) \leq \zeta^{(s)}(t) \leq 0$;
    \item $|\phi_k(t)| \leq \frac{C^{(\phi)} \mathfrak  c^2}{n} \left|\zeta^{(s)}(t)\right | + \frac{\mathfrak  c}{8} |\zeta^{(s)}(0)|$.
\end{enumerate}
\end{corollary}

\subsection{Analysis of Initialization}

Next, we consider the initialization conditions in \cref{cor:main}. In this section, we will show that, either the initialization conditions in \cref{cor:main} will be satisfied at some point, or $\zeta$ and $\phi$ will stay small all the time.

Notice that $\zeta_k(t)$ is initialized as $\zeta_k(0) = {\boldsymbol{w}}(0)^\top {\boldsymbol{\mu}}_k \sim \mathcal N({\boldsymbol{0}}, \tau^2 \omega^2)$ and $\phi_k(t)$ is initialized as $\phi_k(0) = {\boldsymbol{w}}(0)^\top {\boldsymbol{\xi}}_k$ where ${\boldsymbol{w}}(0) \sim \mathcal N({\boldsymbol{0}}, \omega^2{\boldsymbol{I}})$ and ${\boldsymbol{\xi}}_k \sim \mathcal N\left({\boldsymbol{0}}, \kappa^2 {\boldsymbol{I}}\right)$. We have 
\begin{lemma}\label{lem:bound-zeta-phi-init}
    For $s \in \{1,2\}$, and any constant $P$ we have $\forall k \in \mathscr C_s, \phi_k(0) \leq P \mathfrak c \zeta^{(s)}(0)$ with probability higher than $1 - \delta$, where $\delta = \frac{n}{2}\exp(-d/64) + \frac{n}{2}\exp(-2P \mathfrak  c^{-1}) + \mathfrak  c^{1/2}$.
\end{lemma}
\begin{proof}
    Let event $\mathcal A = \{\frac{1}{2}\kappa^2 d \leq \|{\boldsymbol{\xi}}_k\|^2 \leq 2 \kappa^2 d\}$. From Lemma \ref{lem:tail-bound-gaussian-norm}, we have $\mathbb P(\mathcal A) \geq 1 - 2\exp(-d/64)$. Conditioned on $\mathcal A$, for any constant $S$ we have $\mathbb P\{ |\phi_k(0)| \leq 2 S \omega \kappa \sqrt{d}\} \geq 1 - 2\exp\left(-2S^2\right)$. From union bound we have $|\phi_k(0)| \leq 2 S \omega \kappa \sqrt{d}$ holds for all $k \in \mathscr C_s$ with probability at least $1-\frac{n}{2}\exp(-d/64)-\frac{n}{2}\exp\left(-2S^2\right)$.
    
    From Lemma \ref{lem:gaussian-anti-concentrate}, for any constant $T$ we have $\mathbb P\{ |\zeta_k(0)| \geq T \omega \tau\} \leq 1- T$.
    
    Let $S = \frac{P}{2} \mathfrak  c^{-1/2}$ and $T = \mathfrak  c^{1/2}$ we have with probability at least $1 -\frac{n}{2}\exp(-d/64) - \frac{n}{2}\exp(-2P \mathfrak  c^{-1}) - \mathfrak  c^{1/2}$, \begin{align}
       \forall k \in \mathscr C_s, \frac{|\phi_k(0)|}{|\zeta^{(s)}(0)|} \leq \frac{2 S \omega \kappa \sqrt{d}}{T \omega \tau} = P\mathfrak  c,
    \end{align}
    which proves the proposition.
    
\end{proof}

\begin{lemma}\label{lem:initialize-geq}
    For any $s \in \{1,2\}$, if $|\zeta^{(s)}(0)| \geq \sqrt{ \frac{18 \mathcal Ln}{\mathfrak  c^2\sqrt{d}}}$, then the following conclusions hold with probability at least $1 - \delta$, where $\delta = 2 \exp\left(-\frac{\mathfrak  c^2\left(1 - \frac{c^{(t)}}{8}\right)^2 }{8 \eta^2 \omega^2 \tau^6}\right) + \frac{n}{2}\exp(-d/64) + \frac{n}{2}\exp(-\frac{1}{4 \mathfrak c }) + \mathfrak  c^{1/2}$: \begin{enumerate}[label={\alph*)}]
        \item $|\zeta^{(s)}(0)| < \frac{\mathfrak  c\left(8 - {c^{(t)}}\right)}{16 \eta \tau^2}$;
        \item $\forall k \in \mathscr C_s, |\phi_k(0)| \leq \frac{\mathfrak  c}{8} |\zeta_k(0)|$.
    \end{enumerate}
\end{lemma}
\begin{proof}
    From Chernoff bound, we have $|\zeta^{(s)}(0)| \leq \frac{\mathfrak  c\left(8 - {c^{(t)}}\right)}{16 \eta \tau^2}$ with probability at least $1 - \delta_1$, where $\delta_1 = 2 \exp\left(-\frac{\mathfrak  c^2\left(1 - \frac{c^{(t)}}{8}\right)^2 }{8 \eta^2 \omega^2 \tau^6}\right)$. From Lemma \ref{lem:bound-zeta-phi-init}, we have $\forall k \in \mathscr C_s, |\phi_k(0)| \leq \frac{\mathfrak  c}{8} \times \sqrt{ \frac{18 \mathcal Ln}{\mathfrak  c^2\sqrt{d}}} \leq \frac{\mathfrak  c}{8} |\zeta_k(0)|$ holds with probability at least $1-\delta_2$, where $\delta_2 =  \frac{n}{2}\exp(-d/64) + \frac{n}{2}\exp(-\frac{1}{4 \mathfrak c }) + \mathfrak  c^{1/2}$. The proposition is proved through union bounding the probabilities of these two quantities.
\end{proof}

\begin{lemma}\label{lem:initialize-leq}
    Let $q = \frac{32 \mathcal Ln}{\mathfrak  c^2\sqrt{d}}$. For $s \in \{1,2\}$, if $|\zeta^{(s)}(0)| \leq \sqrt{q} $, and there exists a constant $c^{(t)} \in \left(0,\frac{8}{1 + 8  \mathfrak  c}\right)$ such that $t_0 \leq c^{(t)}\left(\eta \tau^2 \sqrt{q}\right)^{-1}$ we have $|\zeta^{(s)}(t_0-1)| \leq \sqrt{q}$, $|\zeta^{(s)}(t_0 )| \geq \sqrt{ q }$ and :
    \begin{enumerate}
        \item $\mathfrak  c \leq \frac{1}{8}$;
        \item $\eta \leq \min\left\{1 , \frac{\mathfrak  c
        \left(8 - c^{(t)}\right)}{16 \tau^2 \left(\sqrt{q} + \tau^2 q\right)}\right\}$, \label{stat:cond-init-eta}
    \end{enumerate}
    then the following conclusions hold with probability at least $1 - \delta$, where $\delta = \frac{n}{2}\exp(-d/64) + \frac{n}{2}\exp(-\frac{1}{4 \mathfrak c }) + \mathfrak  c^{1/2}$: \begin{enumerate}[label={\alph*)}]
        \item $|\zeta^{(s)}(t_0)| < \frac{\mathfrak  c\left(8 - {c^{(t)}}\right)}{16 \eta \tau^2}$;
        \item $\forall k \in \mathscr C_s, |\phi_i(t_0)| \leq \frac{\mathfrak  c}{8} \sqrt{q}$.
    \end{enumerate}
    Moreover, if $|\zeta^{(s)}(t)| < \sqrt{q}$ holds for all $t < c^{(t)}\left(\eta\tau^2 \sqrt{q}\right)^{-1}$, then we also have $\forall i \in \mathscr C_s, \phi_i(t) \leq \frac{\mathfrak  c}{8}\sqrt{q}$.
\end{lemma}
\begin{proof}
    If there does not exists a $t_0$ such that $\zeta^{(s)}(t_0 + 1)| \geq \sqrt{q}$, then we set $t_0 = c^{(t)}\left(\eta\tau^2 \sqrt{ q }\right)^{-1}$ in the following. With this notation we have $|\zeta^{(s)}(t)| \leq \sqrt{ q }$ for all $t < t_0$.
    \begin{itemize}
        \item Consider deduction on $\phi_k$. For a specific $t < t_0$ if for all $t' < t$ we have $|\phi_k(t')| \leq \frac{\mathfrak  c}{8} \times \sqrt{ q }$, then we have \begin{align}
            |\phi_k(t)| & \leq |\phi_k(0)| + \sum_{t' = 1}^{t-1} \left[\frac{3 \eta \kappa^2 d}{4n}\sigma'\left( |\phi_k(t')| + |\zeta^{(s)}(t')| \right) + \frac{5}{2} \eta \kappa \tau \mathcal L\right]
            \\ & \leq |\phi_k(0)| + \sum_{t' = 1}^{t-1} \left[\frac{3 \eta \kappa^2 d}{4n}\left[ \left(1 + \frac{\mathfrak  c}{8} \right) \sqrt{q }\right]^2 + \frac{5 \eta  \kappa^2 d q}{n}\right]
            \\ & \leq |\phi_k(0)| + t_0 \eta \left[\frac{3  \kappa^2 d q}{2n} + \frac{5 \kappa^2 d q}{n}\right]
            \\ & \leq |\phi_k(0)| + 7 c^{(t)} \mathfrak  c^4 \sqrt{q}
            \\ & \leq |\phi_k(0)| + \frac{7 \mathfrak  c}{64} \sqrt{ q }.
        \end{align} 
        From Lemma \ref{lem:bound-zeta-phi-init}, we have $\forall k \in \mathscr C_s, \phi_k(0) \leq \frac{\mathfrak  c}{64}\sqrt{q}$ holds with probability at least $1-\delta_1$, where $\delta_1 = \frac{n}{2}\exp(-d/64) + \frac{n}{2}\exp(-\frac{1}{32 \mathfrak  c}) + \mathfrak  c^{1/2}$. Through union bound, we have all $\frac{n}{4}$ satisfies this condition with probability at least $1 - \frac{n}{4}\delta_1$.
        \item Next, using \cref{eq:descrete-zeta} and Condition \ref{stat:cond-init-eta}, we have \begin{align}
            |\zeta^{(s)}(t_0)| & \leq |\zeta^{(s)}(t_0) - 1| +  \times \frac{\tau^2 \eta}{2n} \times \frac{n}{4} \times \left(\sqrt{q} + \frac{\mathfrak  c}{8} \sqrt{q}\right)^2 + \frac{3 \kappa \tau q \eta}{2}
            \\ & \leq \sqrt{q} + \frac{\eta \tau^2 q}{2} + \frac{3 \mathfrak  c \tau^2 q \eta}{2\sqrt{d}}
            \\ & \leq \frac{\mathfrak  c\left(8 - {c^{(t)}}\right)}{16 \eta \tau^2}.
        \end{align}
    \end{itemize}
    Consider the union bound of the two inequalities, the proposition is proved.
\end{proof}

\subsection{Bounding the Representation Distance}

In this section, we put all the results together and prove our main theorem. In this section we will consider all neurons. For $\zeta^{(s)}$ w.r.t. the $r$-th neuron, we denote it as $\zeta^{(s,r)}$, and for $\phi_k$ we denote it as $\phi_k^{(r)}$. Similar to before if $k \in \mathscr C_s$ we also write $\zeta^{(s,r)}$ as $\zeta_k^{(r)}$.

\begin{theorem}
Suppose the model and training process is described as before, and following conditions hold
\begin{enumerate}
    \item $n \leq d^{1/3}$;
    \item $\sqrt{\frac{32\mathcal L n}{d^{\frac{1}{2}}}}\leq \mathfrak c\leq \min\left\{ \sqrt{\frac{n}{d^{\frac{1}{3}}}} , \frac{1}{8}\right\}$
    \item $\eta \leq \min\left\{1 , \frac{\mathfrak  c^2 \times \left(8\mathfrak c + 4 \omega \tau \sqrt{\log(4m)} \right)}{4 \tau^2\left(\tau^2 + 1\right)}\right\}$;
    \item $\frac{2}{\tau} \times \sqrt{\frac{32 \mathcal Ln}{\mathfrak  c^2 \sqrt{d}}} \leq \omega \leq \frac{1}{4 \tau\sqrt{\log 4m}}$,
    % \item $\frac{2}{\tau} \times \sqrt{\frac{32 \mathcal Ln}{\mathfrak  c^2 \sqrt{d}}} \leq \omega \leq \frac{1}{4n^4 \tau\sqrt{\log 4m}}$,
\end{enumerate}
then for all $t < t_0 = \frac{1 - 4\omega\tau\sqrt{\log{4m}}}{\frac{1}{8} + \mathfrak  c}\times \left(2\eta \tau^3  \omega \sqrt{\log(4m)}\right)^{-1}$, we have \begin{align}
    \frac{\left\|h\left({\boldsymbol{x}}_{i_1}\right) - h\left({\boldsymbol{x}}_{i_3}\right)\right\|}{\left\|h\left({\boldsymbol{x}}_{i_1}\right) - h\left({\boldsymbol{x}}_{i_2}\right)\right\|} & \geq \sqrt{\frac{1}{24\mathcal L'}} \left[\frac{1}{\frac{\mathfrak  c }{2} + \frac{10}{\sqrt{8 \mathcal L} } \times \frac{\mathfrak  c^3 \sqrt{d}}{n^{1.5}}}\right]^3
    = \Omega\left(\min\left\{\mathfrak  c^{-3}, \mathfrak  c^{-9}n^\frac{9}{2}d^{-\frac{3}{2}}\right\}\right)
\end{align}
With probability at least $1-\delta$, where \begin{align}\delta = 4mt_0 \exp\left(-\mathfrak  c^{-2}\right) + mn \exp\left(-\mathfrak  c^{-1}\right) + m \mathfrak  c^\frac{1}{2} + 2m \exp\left[-\frac{\mathfrak  c^2 \times \left(8\mathfrak c + 4 \omega \tau \sqrt{\log(4m)} \right)}{16 \eta^2 \omega^2 \tau^6}\right] + 2m^{-1}.\end{align}
\end{theorem}
\begin{proof}
Let $q = \frac{32 \mathcal Ln}{\mathfrak  c^2 \sqrt{d}} \leq 1$ and $c^{(t)} = \frac{1 - 4\omega\tau\sqrt{\log{4m}}}{\frac{1}{8} + \mathfrak  c} \leq 8$, we have  $\eta \leq \frac{\mathfrak  c
        \left(8 - c^{(t)}\right)}{16 \tau^2 \left(\sqrt{q} + \tau^2 q\right)}$ and $\omega \tau \geq 2\sqrt{q}$. 

\begin{itemize}

\item First, notice that both Lemma \ref{lem:w_on_mu} and Lemma \ref{lem:w_on_xi} holds with high probability. Specifically, \cref{lem:w_on_mu} holds with probability at least $1- \exp(-n)$, and \cref{lem:w_on_xi} holds with probability at least $1-\delta_1'$, where $\delta_1' = \exp(-n) + 2n \exp\left(-\frac{1}{\mathfrak  c^2}\right) + 2\exp(-d/64) \leq 3n \exp\left(-\frac{1}{\mathfrak  c^2}\right)$. We have that \cref{cor:main} holds with probability at least $1 - mt_0 \delta_1$, where $\delta_1 = 4n\exp(-n)$.

\item From Lemma \ref{lem:bound-gaussian-max}, we have $\forall r \leq m,  \left|\zeta^{(1,r)}(0)\right| \leq 2 \omega\tau\sqrt{\log 4m}$ with probability at least $1-\delta_2$, where $\delta_2 = \frac{1}{m}$. Similarly the probability of $\forall r \leq m,  \left|\zeta^{(2,r)}(0)\right| \leq 2 \omega\tau\sqrt{\log 4m}$ is also at least $1-\delta_2$. In this case for any $s,r$ we have $t_0 \leq c^{(t)}\left(\eta \tau^2 \left|\zeta^{(s,r)}(0)\right|\right)^{-1}$.

If in addition \cref{cor:main} holds, then from \cref{lem:initialize-leq,lem:initialize-geq,cor:main}, for any $s \in \{1,2\}$ and $i \in \mathscr C_s$ we have $|\phi_i(t)| \leq \frac{C^{(\phi)} \mathfrak  c^2}{n} \left|\zeta^{(s)}(t)\right | +\frac{\mathfrak  c}{8} \sqrt{q}$ and $\zeta^{(s,r)}(t) \leq \left[1 - \left(\frac{1}{8} + \mathfrak 
 c\right)c^{(t)}\right]^{-1} \left|\zeta^{(1,r)}(0)\right|$ with probability at least $1-\delta_3$, where $\delta_3 = 2 \exp\left(-\frac{\mathfrak  c^2\left(1 - \frac{c^{(t)}}{8}\right)^2 }{8 \eta^2 \omega^2 \tau^6}\right) +  n\exp(-\frac{1}{4 \mathfrak c }) + \mathfrak  c^{1/2}$. In this case we have \begin{align}
\forall r \leq m, \forall s \in \{1,2\}, \forall t \leq t_0, \left|\zeta^{(s,r)}(t_0)\right| & \leq \left(\left|\zeta^{(s,r)}(0)\right|^{-1} - \left(\frac{1}{8} + \mathfrak  c\right)^{-1} \eta\tau^2 t\right)^{-1}
\\ & \leq \left[1 - c^{(t)}\left(\frac{1}{8} + \mathfrak  c\right)\right]^{-1} \left|\zeta^{(s,r)}(0)\right|
\\ & \leq \left[1 - c^{(t)}\left(\frac{1}{8} + \mathfrak  c\right)\right]^{-1} \times 2 \omega \tau \sqrt{\log 4m}
\\ & \leq \frac{1}{2}
\end{align}
with probability at least $1-2\delta_2 - m\delta_3$. 

With $|\zeta^{(s,r)}(t)| \leq \frac{1}{2}$ and $\left|\phi_i(t)\right| \leq |\zeta^{(s,r)}(t)| \leq \frac{1}{2}$ where $i \in \mathscr C_s$, we have $|{\boldsymbol{w}}(t)^\top {\boldsymbol{x}}_i| \leq 1$, which falls in the range where $\sigma({\boldsymbol{w}}(t)^\top {\boldsymbol{x}}_i) = \frac{1}{3} \left[{\boldsymbol{w}}(t)^\top {\boldsymbol{x}}_i\right]^3$, which fulfills our assumption in Section \ref{sec:dynamic-model-parameter}. In the following we will assume $\sigma(z) = \frac{1}{3}z^3$.

\item If the conclusion of \cref{lem:initialize-leq,lem:initialize-geq,cor:main} holds, for any $t \leq t_0$ we have \begin{align}
    \left|h({\boldsymbol{x}}_{i_1}) - h\left({\boldsymbol{x}}_{i_2}\right)\right|_r & = \left|\sigma\left({\boldsymbol{w}}_r(t)^\top {\boldsymbol{\mu}}^{(1)} + {\boldsymbol{w}}_r(t)^\top {\boldsymbol{\xi}}_{i_1}\right) - \sigma\left({\boldsymbol{w}}_r(t)^\top {\boldsymbol{\mu}}^{(1)} + {\boldsymbol{w}}_r(t)^\top {\boldsymbol{\xi}}_{i_2}\right)\right|
\\ & \leq \mathcal L' \sigma \left| {\boldsymbol{w}}_r(t)^\top {\boldsymbol{\xi}}_{i_1} - {\boldsymbol{w}}_r(t)^\top{\boldsymbol{\xi}}_{i_2} \right|
\\ & = \mathcal L'\sigma\left( \left|\phi_{i_1}^{(r)}(t)\right| + \left|\phi_{i_2}^{(r)}(t)\right|\right)
\\    & \leq \mathcal L' \sigma\left( \frac{\mathfrak  c}{4}\sqrt{q} + \frac{2 C^{(\phi)} \mathfrak  c^2}{n}\left|\zeta^{(1,r)}(t)\right| \right)
    \\ & \leq \mathcal L' \sigma\left( \frac{\mathfrak  c}{4}\sqrt{q} + \frac{C^{(\phi)} \mathfrak  c^2}{n}\right).
\end{align}

\item Notice that if \cref{cor:main} holds, then the sign of $\zeta^{(s,r)}(t)$, as well as $\zeta^{(s,r)}(t) - \phi_i^{(s)}(t)$, is determined by its initialization $\zeta^{(s,r)}(t)$. Since $\zeta^{(s,r)}(0) \sim \mathcal N(0,\omega^2 \tau^2)$, we with probability at least $0.5$ that $\zeta^{(2,r)}(0) \leq 0$ and since $\omega \geq \frac{2 \sqrt{q}}{\tau}$, from Lemma \ref{lem:gaussian-anti-concentrate} we have $\zeta^{(1,r)}(0) \geq \sqrt{q}$ with probability at least $\frac{1}{2}$. We denote $\mathcal A$ to be the neuron indices $r$ who satisfies $\zeta^{(2,r)}(0) < 0$ and $\zeta^{(1,r)}(0) \geq \sqrt{q}$. For each $r \leq m$, we have $\mathbb P\{r \in \mathcal A\} \geq \frac{1}{4}$. By calculating the concentration of Binomial distribution \citep{concen-survey}, we have $|\mathcal A| \geq \frac{m}{8}$ with probability at least $1 - \delta_4$, where $\delta_4 = \exp(-m/32)$. If $r \in \mathcal A$, then we have \begin{align}
\left|h({\boldsymbol{x}}_{i_1}) - h\left({\boldsymbol{x}}_{i_3}\right)\right|_r & = \left|\sigma\left({\boldsymbol{w}}_r(t)^\top {\boldsymbol{\mu}}^{(1)} + {\boldsymbol{w}}_r(t)^\top {\boldsymbol{\xi}}_{i_1}\right) - \sigma\left({\boldsymbol{w}}_r(t)^\top {\boldsymbol{\mu}}^{(2)} + {\boldsymbol{w}}_r(t)^\top {\boldsymbol{\xi}}_{i_2}\right)\right|
\\ & \geq \sigma\left(\left|\zeta^{(1,r)}(t) + \phi_{i_1}^{(r)}(t)\right|\right).
\\ & \geq \sigma\left(\left|\zeta^{(1,r)}(t)\right| - \left|\phi_{i_1}^{(r)}(t)\right|\right).
\\ & \geq \sigma\left[ \left(1 - \frac{C^{(\phi)} \mathfrak  c^2}{n}\right)\left|\zeta^{(1,r)}(t)\right | - \frac{\mathfrak  c}{8}|\zeta_k(0)|\right]
\\ & \geq \sigma\left(\frac{1}{2} \zeta_k(t)\right)
\\ & \geq \sigma\left( \frac{1}{2} \left[\zeta^{(1,r)}(0)^{-1} - \left(\frac{1}{8} - \mathfrak  c\right)\eta \tau^2 t\right]^{-1} \right)
\\ & = \sigma\left( \frac{1}{2} \left[1 - c^{(t)}\left(\frac{1}{8} - \mathfrak  c\right)\right]^{-1}\left|\zeta^{(1,r)}(0)\right| \right).
\end{align}
Notice that $\forall r \in \mathcal A,  \left|\zeta^{(1,r)}(0)\right| \geq \sqrt{q}$, we have \begin{align}
\left\|h\left({\boldsymbol{x}}_{i_1}\right) - h\left({\boldsymbol{x}}_{i_3}\right)\right\|^2 & = \sum_{r=1}^m \left|h({\boldsymbol{x}}_{i_1}) - h\left({\boldsymbol{x}}_{i_3}\right)\right|_r^2
\\ & \geq \sum_{r \in \mathcal A} \left|h({\boldsymbol{x}}_{i_1}) - h\left({\boldsymbol{x}}_{i_3}\right)\right|_r^2
\\ & \geq \frac{m}{8} \times \sigma\left( \frac{1}{2} \left[1 - c^{(t)}\left(\frac{1}{8} - \mathfrak  c\right)\right]^{-1}\sqrt{q} \right)^2.
\\ & \geq \frac{m}{8} \times \sigma\left( \frac{1}{2} \sqrt{q} \right)^2.
\end{align}
\item To put things together, notice that since $\sigma(z) = \frac 13z^3$, we have $\sigma(a)/\sigma(b) = 3 \sigma(a/b)$. Let $\delta = mt_0 \delta_1 + 2 \delta_2 + m\delta_3 + \delta_4$, with probability at least $1-\delta$ we have \begin{align}
\frac{\left\|h\left({\boldsymbol{x}}_{i_1}\right) - h\left({\boldsymbol{x}}_{i_3}\right)\right\|^2}{\left\|h\left({\boldsymbol{x}}_{i_1}\right) - h\left({\boldsymbol{x}}_{i_2}\right)\right\|^2} &  \geq \frac{3\times \frac{m}{8}}{m \mathcal L'} \times \sigma\left(\frac{\frac{1}{2} \sqrt{q} }{\frac{\mathfrak  c}{4}\sqrt{q} + \frac{ C^{(\phi)} \mathfrak  c^2}{n}}\right)^2
\\ & = \frac{3}{8 \mathcal L'}\sigma\left(\frac{1}{\frac{\mathfrak  c }{2} + \frac{2 C^{(\phi)} \mathfrak  c^2}{n\sqrt{q}}}\right)^2
\\ & \geq \frac{3}{8 \mathcal L'}\sigma\left(\frac{1}{\frac{\mathfrak  c }{2} + \frac{2 C^{(\phi)} \mathfrak  c^2}{n\sqrt{\frac{32\mathcal L n}{\mathfrak  c^2 d^{1/2}}}}}\right)^2
\\ & = \frac{3}{8 \mathcal L'}\sigma\left(\frac{1}{\frac{\mathfrak  c }{2} + \frac{C^{(\phi)}}{\sqrt{8 \mathcal L} } \times \frac{\mathfrak  c^3 \sqrt{d}}{n^{1.5}} }\right)^2.
\end{align}
Notice that $\sigma(z) = \frac{1}{3}z^3$, we have \begin{align}
    \frac{\left\|h\left({\boldsymbol{x}}_{i_1}\right) - h\left({\boldsymbol{x}}_{i_3}\right)\right\|}{\left\|h\left({\boldsymbol{x}}_{i_1}\right) - h\left({\boldsymbol{x}}_{i_2}\right)\right\|} & \geq \sqrt{\frac{1}{24\mathcal L'}} \left[\frac{1}{\frac{\mathfrak  c }{2} + \frac{C^{(\phi)}}{\sqrt{8 \mathcal L} } \times \frac{\mathfrak  c^3 \sqrt{d}}{n^{1.5}}}\right]^3
    \\ & = \Omega\left(\min\left\{\mathfrak  c^{-3}, \mathfrak  c^{-9}n^\frac{9}{2}d^{-\frac{3}{2}}\right\}\right),
\end{align}
and simply taking $C^{(\phi)} = 10$  proves the proposition.

\end{itemize}

\end{proof}

\section{Experiment on Mixture of Gaussian with Coarse Labels}\label{sec:additional-synthetic}

In this section, we reproduce the experiments under the setting of classifying mixture of Gaussian using a 2-layer MLP.

Formally, we create a dataset from a mixture of Gaussian, where the input from the $c$-th cluster is generated from $\mathcal N({\boldsymbol{\mu}}^{(c)}, {\boldsymbol{I}})$, and each cluster mean ${\boldsymbol{\mu}}^{(c)}$ is drawn i.i.d. from $\mathcal N({\boldsymbol{0}}, \sigma^2{\boldsymbol{I}})$. %is the cluster mean of each cluster, which is also generated by Gaussian distribution, and $\mathrm{dinp}$ is the dimensionality of input features, ${\boldsymbol{0}}_{\mathrm{dinp}}$ is a vector with all entries being zero and ${\boldsymbol{I}}_\mathrm{dinp} \in \mathbb R^{\mathrm{dinp} \times \mathrm{dinp}}$ is an identity matrix, $\sigma$ is a positive scalar representing the variance of the cluster means.
The class label of each datapoint is the index of the cluster it belongs to.
The larger $\sigma^2$ is, the larger the separation between each two clusters is, and the more likely it is to observe a fine-grained representation structure when given coarse labels.

We perform the same coarsening process described in \Cref{sec:experiment-setup} (by combining two classes into one super-class) and train a $2$-layer MLP on the coarsely labeled dataset. We measure the significance of the fine-grained structure using the ratio between \{the average squared distance between representations in the same super-class but different sub-classes\} and \{the average squared distance between representations in the same subclass\}, which we call the Mean Squared Distance Ratio (MSDR). Mathematically, it is defined as
\[
\mathrm{MSDR} = \frac{\mathrm{average}_{i\not=j \text{ in same super-class}}\{D_{i,j}\}}{\mathrm{average}_i \{D_{i,i}\}},
\]
where $D_{i,j}$ is defined in \eqref{eqn:mean-square-distance}. A larger $\mathrm{MSDR}$ means that the fine-grained structure is more pronounced, while $\mathrm{MSDR} \approx 1$ indicates no fine-grained structure.

%value of the central dark line and the two side dark lines in the class distance matrix, i.e. the ratio of average squared distance of samplings within the same super-class and original class, which we call Mean Squared Distance Ratio.

By varying training and data-generating parameters, we investigate factors that impact the significance of the fine-grained structure. Specifically, we vary the input dimension, hidden dimension in the network, and weight decay rate, and plot how MSDR scales with $\sigma^2$. The results are shown in \cref{fig:d-inp,fig:d-hid,fig:wd}. Note that in each figure we use two different scales in the x-axes to differentiate the cases of $\sigma^2<2$ and $\sigma^2>2$. 

From the figures, we observe that both input and hidden dimensions exhibit a clear positive correlation with $\mathrm{MSDR}$. On the other hand, the weight decay rate does not have an impact on $\mathrm{MSDR}$ in this setting.

%However, different from the observation made on Coarse CIFAR-10, in this synthetic dataset it seems weight decay rate does not have a clear influence on the  significance of fine-grained structure.

\begin{figure}[h]
\begin{minipage}{0.49\linewidth}
    \centering
    \includegraphics[scale=0.6]{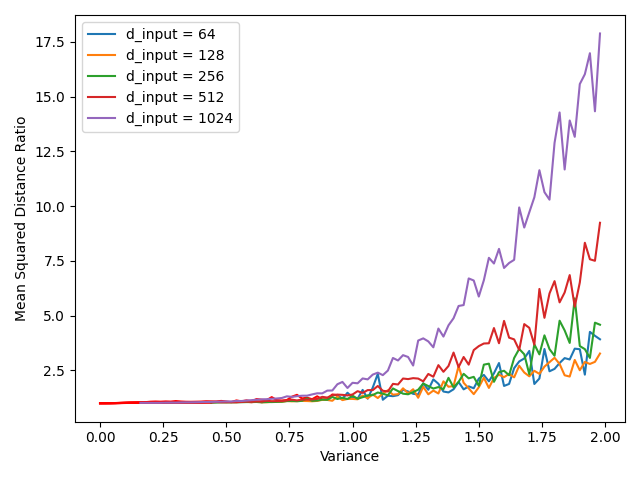}
\end{minipage}
\hfill
\begin{minipage}{0.49\linewidth}
    \centering
    \includegraphics[scale=0.6]{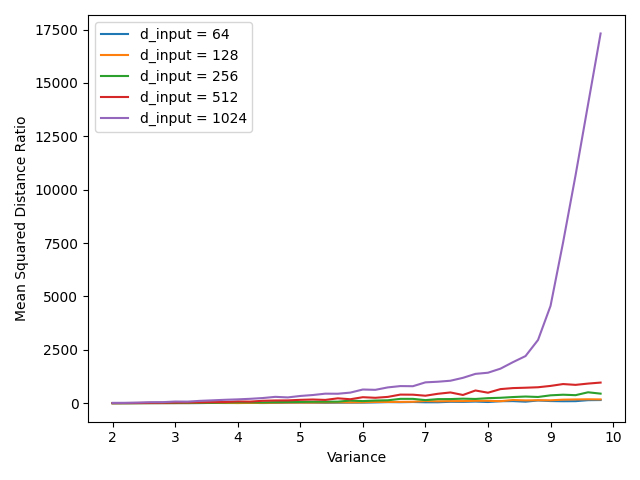}
\end{minipage}
\hfill
    \caption{Mean Squared Distance Ratio vs. variance $\sigma^2$ for different input dimensions. Red lines on the left end are cases where the training accuracy does not reach $100\%$.}
    \label{fig:d-inp}
\end{figure}

\begin{figure}
\begin{minipage}{0.49\linewidth}
    \centering
    \includegraphics[scale=0.6]{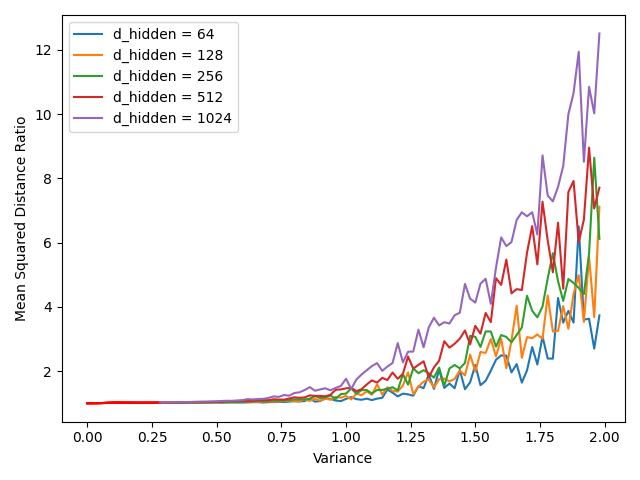}
\end{minipage}
\hfill
\begin{minipage}{0.49\linewidth}
    \centering
    \includegraphics[scale=0.6]{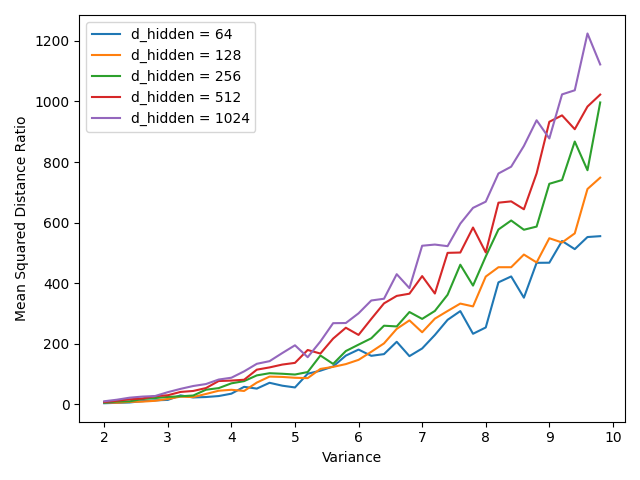}
\end{minipage}
\hfill
    \caption{Mean Squared Distance Ratio vs. variance $\sigma^2$ for different hidden dimensions. Red lines on the left end are cases where the training accuracy does not reach $100\%$.}
    \label{fig:d-hid}
\end{figure}

\begin{figure}[h]
\begin{minipage}{0.49\linewidth}
    \centering
    \includegraphics[scale=0.6]{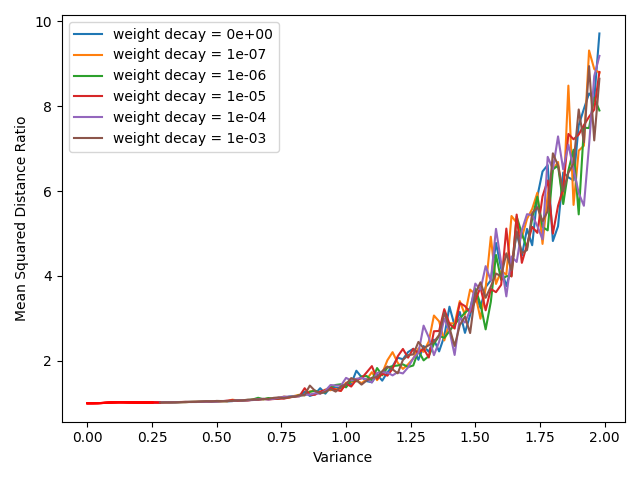}
\end{minipage}
\hfill
\begin{minipage}{0.49\linewidth}
    \centering
    \includegraphics[scale=0.6]{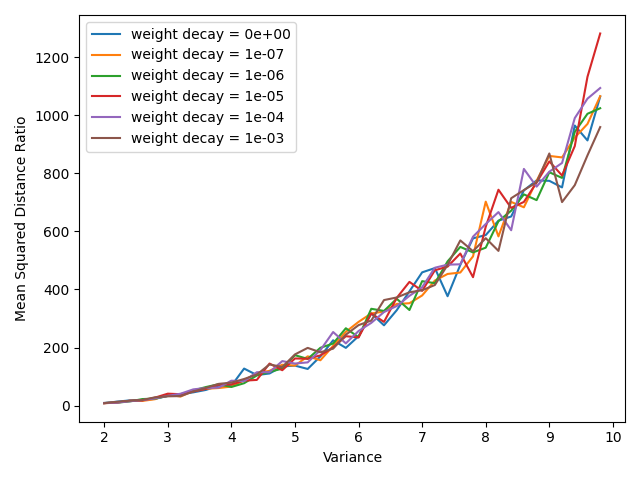}
\end{minipage}
\hfill
    \caption{Mean Squared Distance Ratio vs. variance $\sigma^2$ for different weight decay rates. Red lines on the left end are cases where the training accuracy does not reach $100\%$.}
    \label{fig:wd}
\end{figure}

\paragraph{Training details.}
We generate data from $8$ clusters, each having $500$ samples. We train the model with gradient descent for $1\text{,}000$ steps. The results are averaged over $10$ runs. When varying one hyper-parameter, other hyper-parameters are set to their default values: $d_\mathrm{input} = 512, d_\mathrm{hidden} = 512, \mathrm{weight\ decay} = 0$.

\subsection{Modeling Semantic Similarity} \label{sec:synthetic-similarity}

In \Cref{sec:semantic-similarity}, we showed that the emergence of fine-grained representations depends on the semantic similarity between sub-classes.
Now we take a step in investigating this question by creating ``similar'' and ``dissimilar'' sub-classes in the Gaussian mixture model considered in this section. 

In particular, we use the same data-generating process described above, except that half of the super-classes will be altered so that they consist of ``similar'' sub-classes, and we say that the other super-classes consist of ``dissimilar'' sub-classes. The way to generate similar sub-classes it to first sample $ {\boldsymbol{\mu}} \sim \mathcal N({\boldsymbol{0}}, \sigma^2 {\boldsymbol{I}})$ and then generate two means ${\boldsymbol{\mu}}^{(c)}, {\boldsymbol{\mu}}^{(c')} \sim \mathcal N({\boldsymbol{\mu}}, \tau^2 {\boldsymbol{I}})$.
Therefore, we can vary $\tau^2$ to control the level of similarity between similar sub-classes.

%We further design a variation of the synthetic task described in \cref{sec:additional-synthetic}: we take several super-classes in the dataset and modify their data-generating process as follows: For each selected super-class $k$, first generate a super-class mean $\bar {\boldsymbol{\mu}}_k \sim \mathcal N({\boldsymbol{0}}, \sigma^2 {\boldsymbol{I}})$, then generate the two sub-class means on the top of the super-class mean: \begin{equation} {\boldsymbol{\mu}}_k^{(1)}, {\boldsymbol{\mu}}_k^{(2)} \sim \mathcal N({\boldsymbol{\bar \mu}_k}, \tau^2 {\boldsymbol{I}}), \label{eq:sep} \end{equation}
%and the data points under these two sub-classes are generated with ${\boldsymbol{\mu}}_k^{(1)}$ and ${\boldsymbol{\mu}}_k^{(2)}$.

We use default hyper-parameters described above, fix $\sigma^2 = 4$, and vary $\tau^2$.  \cref{fig:sep} shows the Mean Squared Distance Ratio for similar and dissimilar subclasses, respectively.
We see that fine-grained structure within a super-class does require sufficient dissimilarity between its sub-classes, which agrees with our observation from \Cref{sec:semantic-similarity} on CIFAR-100.

\begin{figure}[ht]
\begin{minipage}{0.49\linewidth}
    \centering
    \includegraphics[scale=0.4]{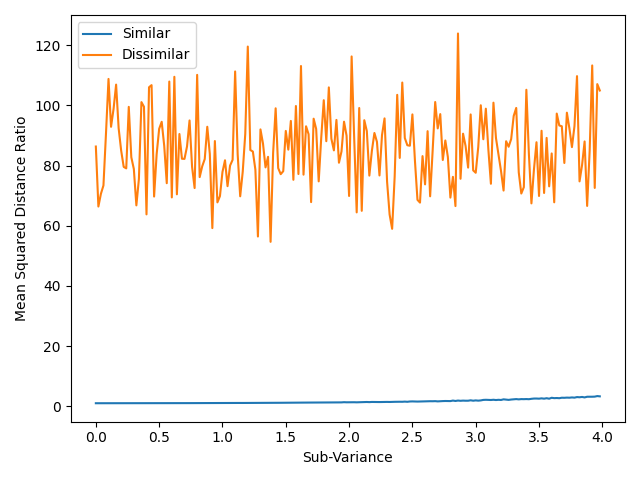}
\end{minipage}
\hfill
\begin{minipage}{0.49\linewidth}
    \centering
    \includegraphics[scale=0.4]{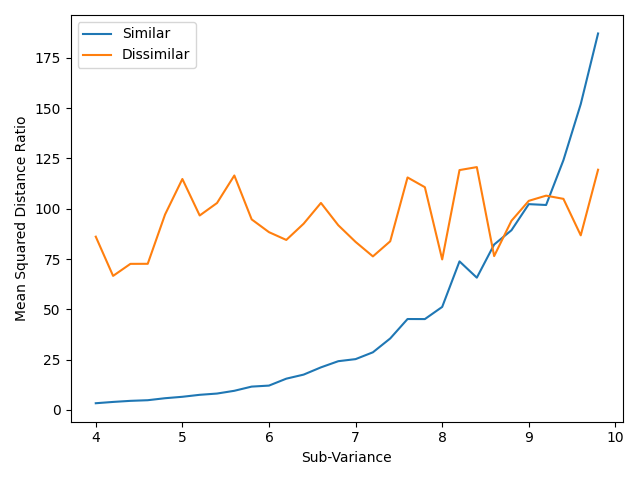}
\end{minipage}
\hfill
     \caption{Mean Squared Distance Ratio vs. sub-variance $\tau^2$. The setting is described in \Cref{sec:synthetic-similarity}.}
    \label{fig:sep}
\end{figure}

\section{Complete Coarse CIFAR-10 Experiment Results}\label{sec:complete-cifar-5}

In the following sections, we provide extended experiment results. As mentioned in \Cref{sec:experiment-setup}, we use learning rate in $\{10^{-1} , 10^{-2} , 10^{-3}\}$ and weight decay rate in $\{5\times 10^{-3} , 5\times 10^{-4} , 5 \times 10^{-5}\}$. Generally, the results will be shown in a $3 \times 3$ table, of which each grid represents the result of one hyper-parameter combination, with each row has the same learning rate and each column has the same weight decay rate.

In this section, we repeat the experiments in \Cref{sec:experiment-on-coarse-cifar-10} with all learning rate and weight-decay rate combinations.  Firstly, we present the training statistics (accuracy, loss) of all hyper-parameters in \Cref{fig:train_statistics_acc,fig:train_statistics_loss} as an reference. It can be observed that all hyper-parameter groups achieved very low training error except the first one (weight decay = $5\times 10^{-3}$, learning rate = $10^{-1}$). In fact, the last two hyper-parameter combinations (learning rate = $10^{-3}$, weight decay $\in \{5 \times 10^{-4} , 5 \times 10^{-5}\}$) didn't achieve exactly $0$ training error (their training error are $<0.5\%$ but not exactly $0$), and the other $6$ hyper-parameter combinations all achieved exactly $0$ training error.

\begin{figure}[h]
    \begin{minipage}{0.49\linewidth}
        \centering
        \includegraphics[scale=0.30]{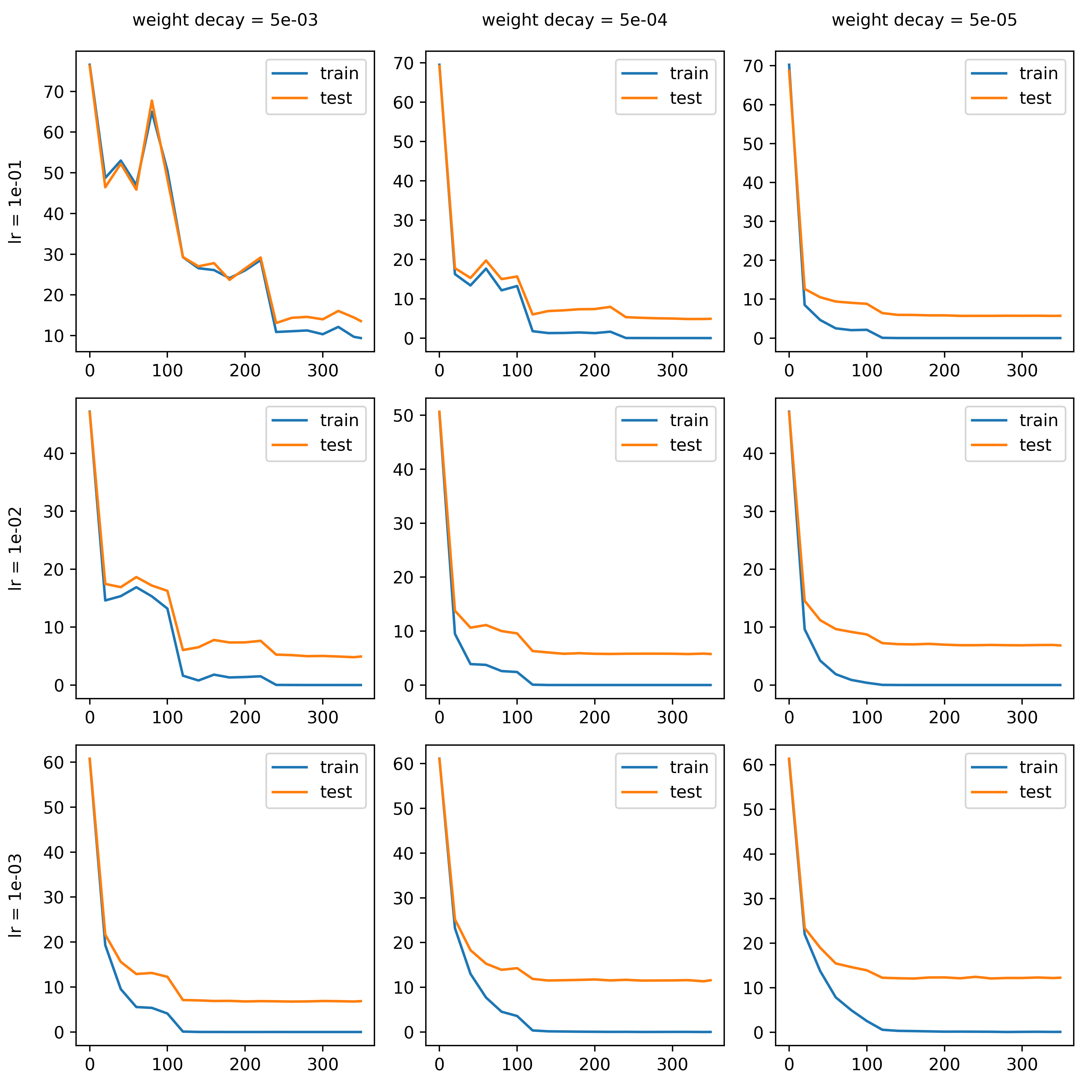}
        \caption{Training and test error during training.}
        \label{fig:train_statistics_acc}
    \end{minipage}
    \hfill
    \begin{minipage}{0.49\linewidth}
        \centering
        \includegraphics[scale=0.30]{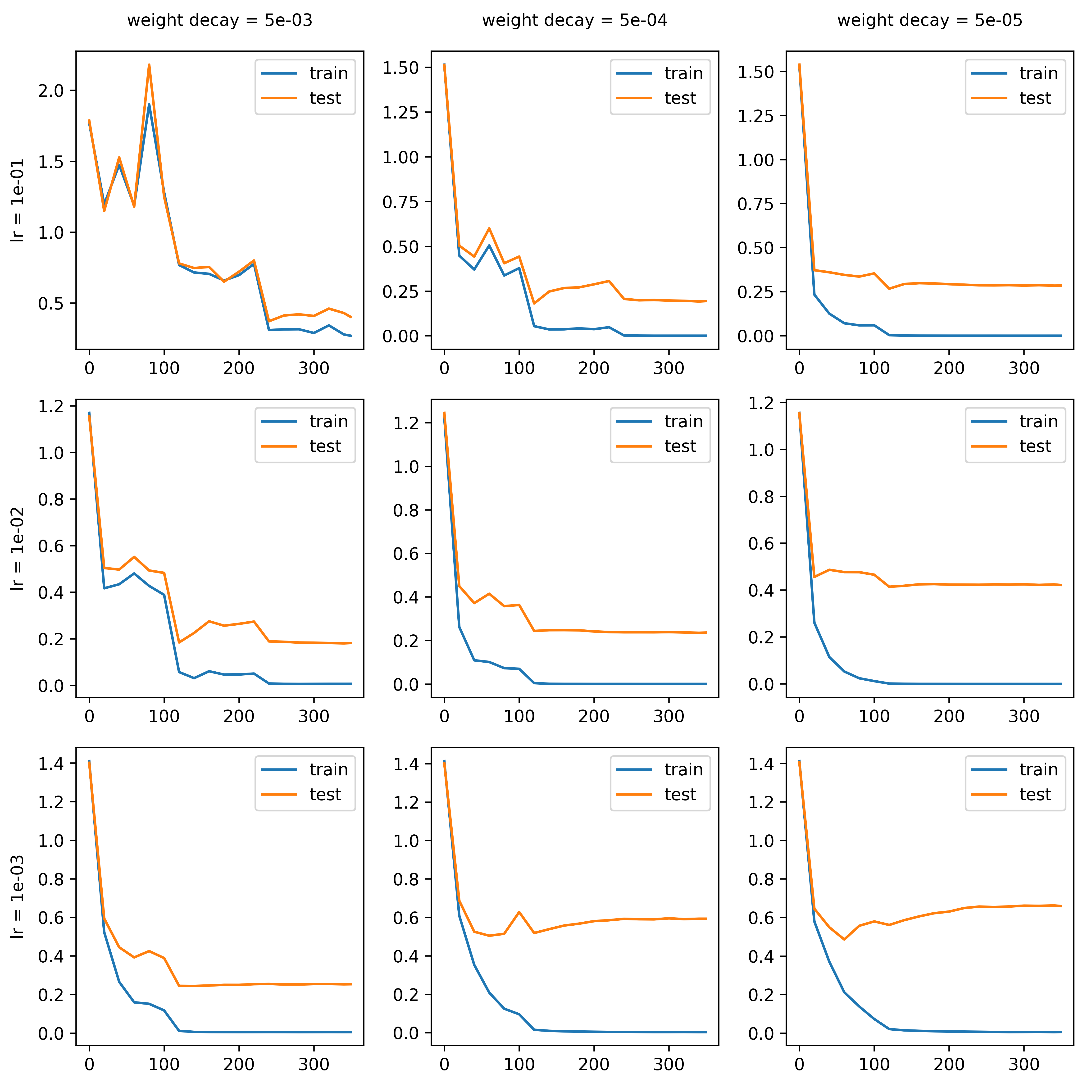}
        \caption{Training and test loss during training.}
        \label{fig:train_statistics_loss}
    \end{minipage}
\end{figure}

\subsection{Class Distance}

Here we present the visualization of the heatmap of class distance matrix $D$ which is defined in \Cref{sec:class-distance}. We choose $4$ epochs to show the trend during training. The results are presented in \Cref{fig:distance-cifar-5-20,fig:distance-cifar-5-120,fig:distance-cifar-5-240,fig:distance-cifar-5-350}, whose epoch numbers are $20,120,240$ and $349$ respectively.

\begin{figure}
    \begin{minipage}{0.49\linewidth}
        \centering
        \includegraphics[scale=0.30]{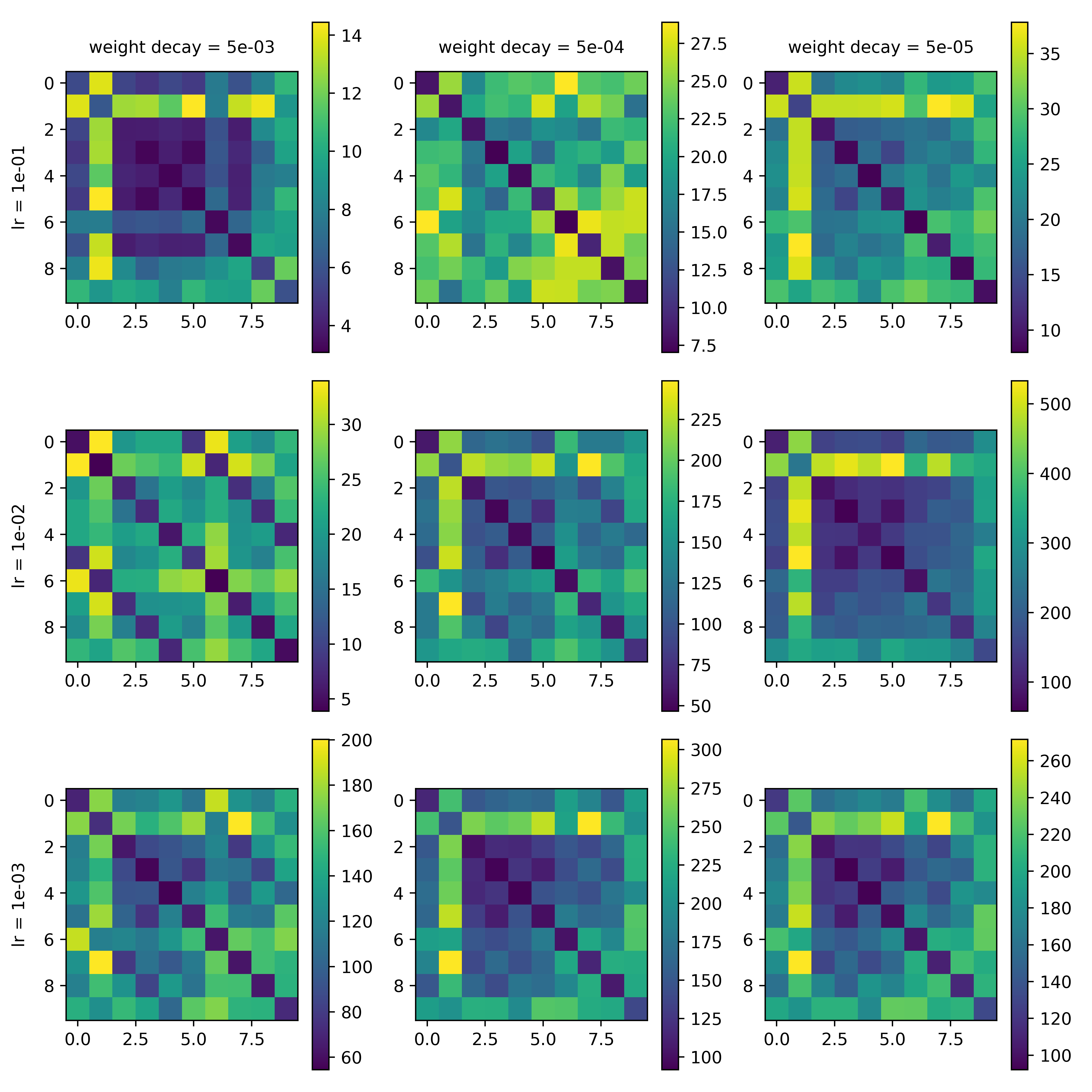}
        \caption{The heatmaps of class distance matrices of different hyper-parameter combinations at epoch $20$.}
        \label{fig:distance-cifar-5-20}
    \end{minipage}
    \hfill
    \begin{minipage}{0.49\linewidth}
        \centering
        \includegraphics[scale=0.30]{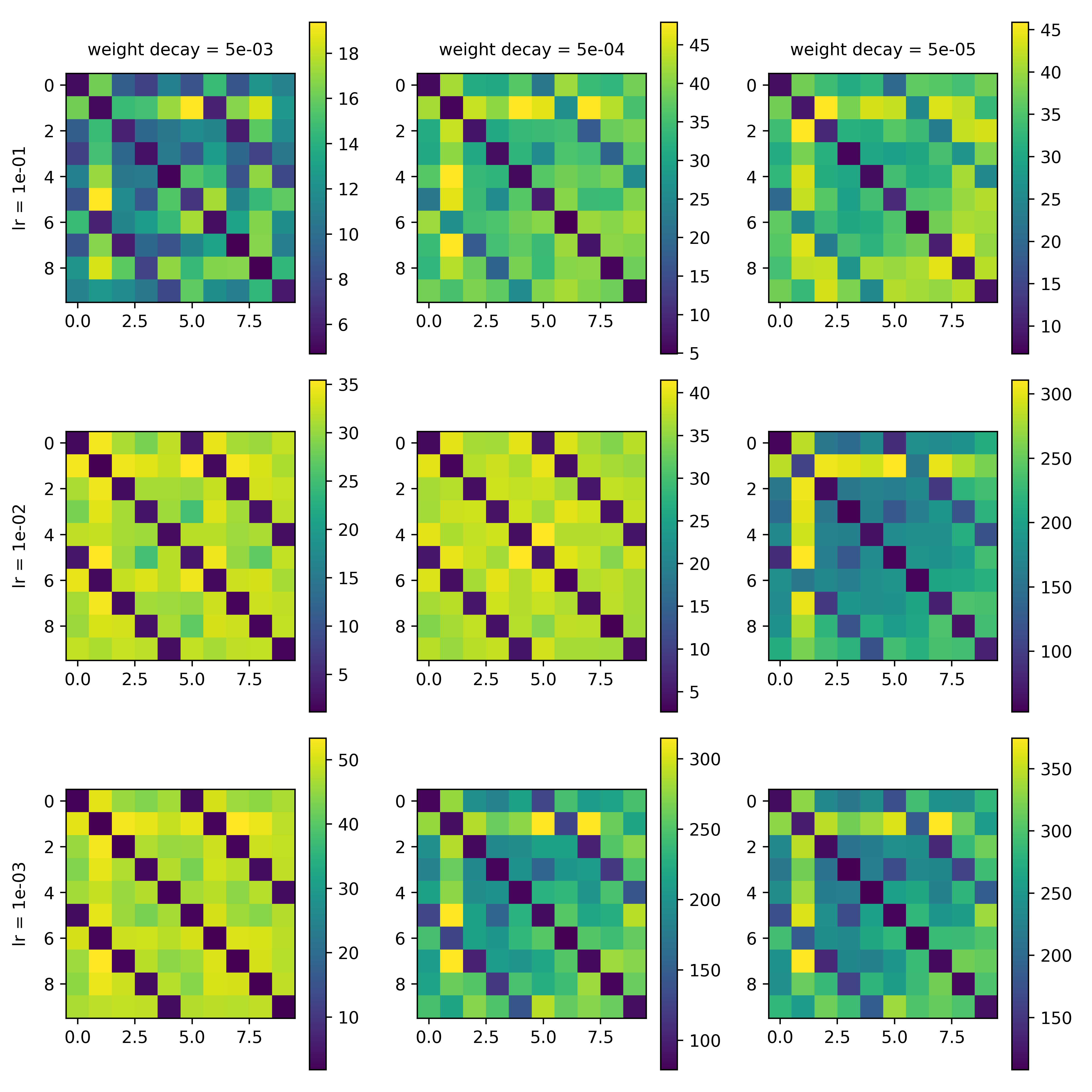}
        \caption{The heatmaps of class distance matrices of different hyper-parameter combinations at epoch $120$.}
        \label{fig:distance-cifar-5-120}
    \end{minipage}
\end{figure}

\begin{figure}
    \begin{minipage}{0.49\linewidth}
        \centering
        \includegraphics[scale=0.30]{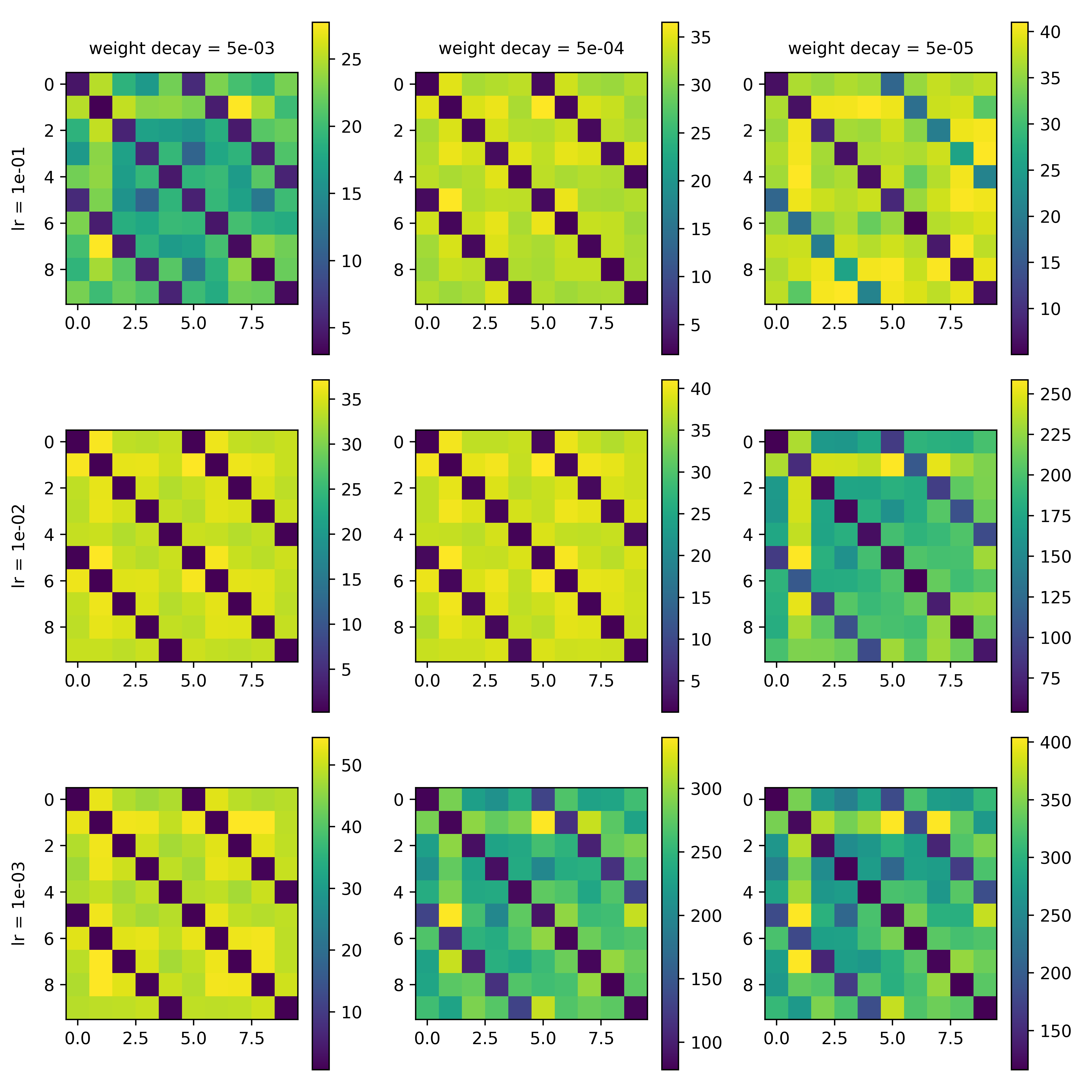}
        \caption{The heatmaps of class distance matrices of different hyper-parameter combinations at epoch $240$.}
        \label{fig:distance-cifar-5-240}
    \end{minipage}
    \hfill
    \begin{minipage}{0.49\linewidth}
        \centering
        \includegraphics[scale=0.30]{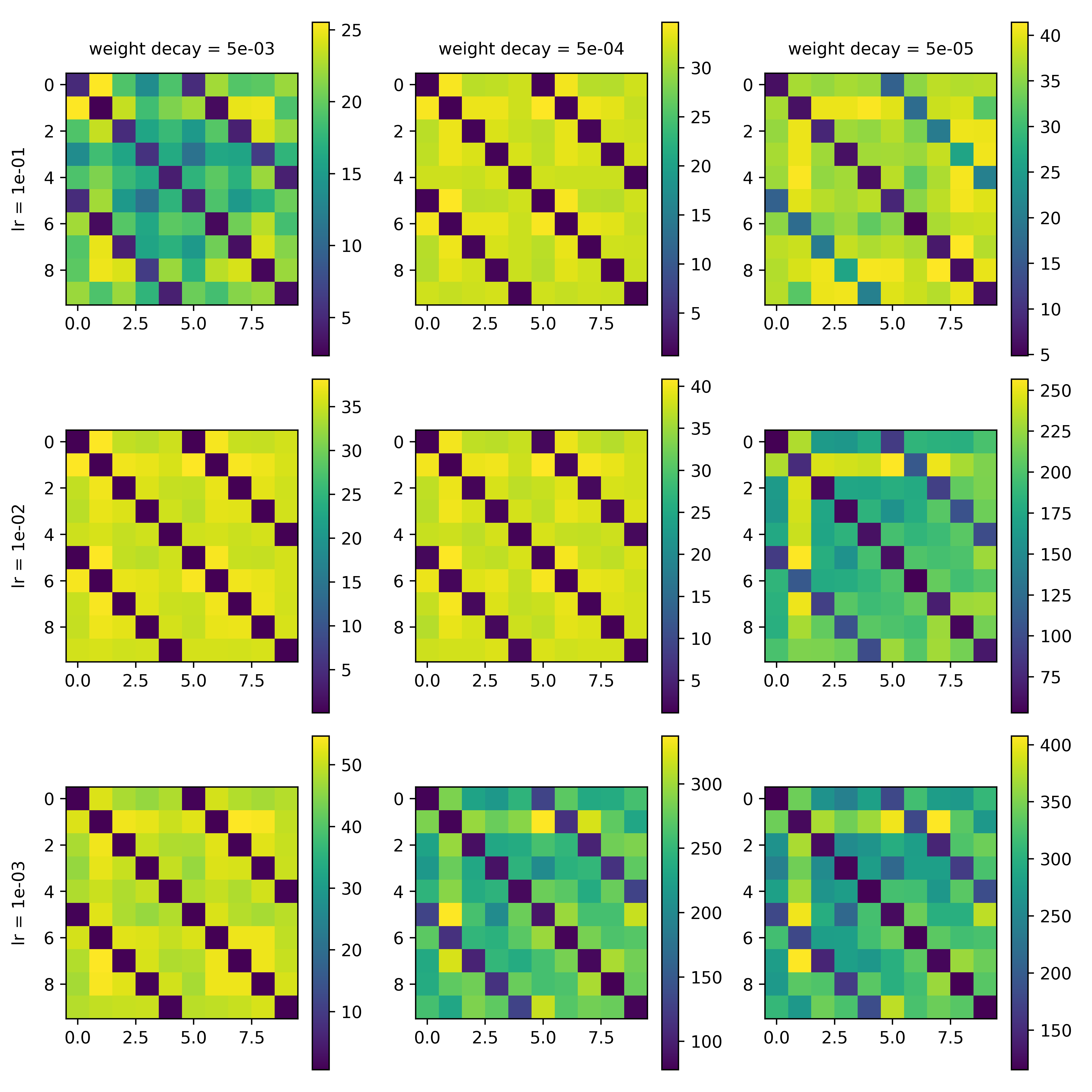}
        \caption{The heatmaps of class distance matrices of different hyper-parameter combinations at epoch $349$.}
        \label{fig:distance-cifar-5-350}
    \end{minipage}
\end{figure}

\subsection{Visualization}

In this section, we present the t-SNE visualization result of ResNet-18 on Coarse CIFAR-10 in \Cref{fig:visual-cifar-5-123,fig:visual-cifar-5-456,fig:visual-cifar-5-789}. The results are divided into three groups, each of which has the same learning rate and the format of each group is the same as \Cref{fig:tsne-cifar-5-349-group2}.

\begin{figure}
    \centering
    \begin{minipage}{1\linewidth}
\centering
        \includegraphics[scale=0.12]{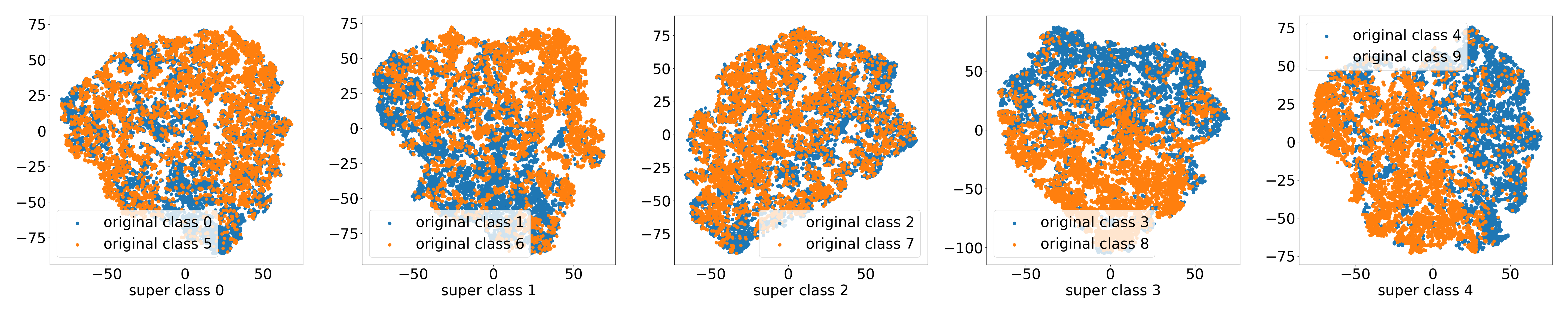}
    \end{minipage}
    \begin{minipage}{1\linewidth}
\centering
        \includegraphics[scale=0.12]{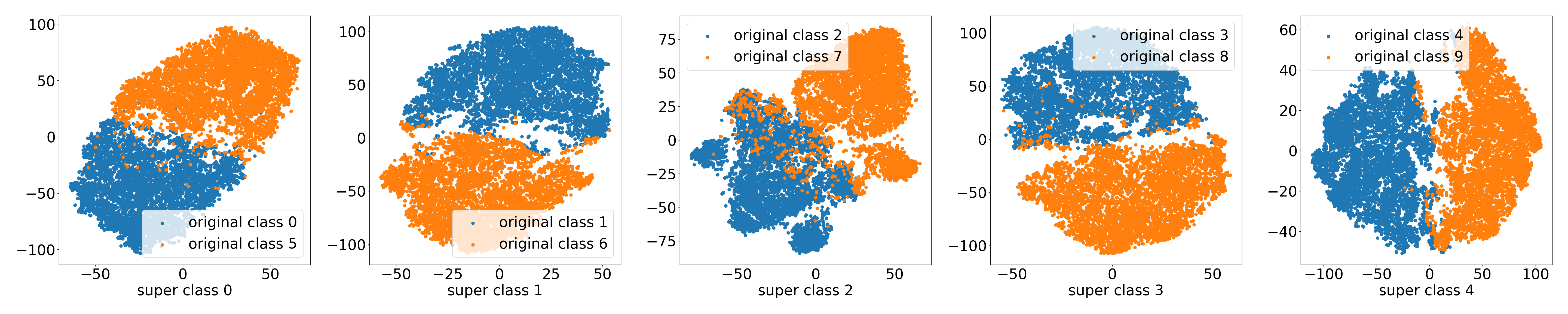}
    \end{minipage}
    \begin{minipage}{1\linewidth}\centering
        \includegraphics[scale=0.12]{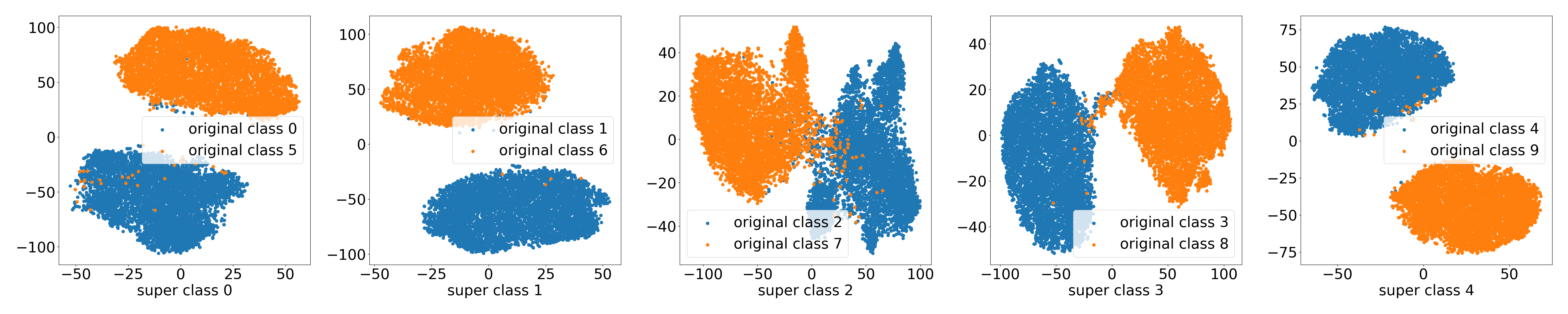}
    \end{minipage}
    \caption{Visualization of last layer representations of ResNet-18 trained on Coarse CIFAR-10 with learning rate = $0.1$. Each row represents a hyper-parameter combination and each column represents a super-class. The weight decay rates from top to bottom are $5\times 10^{-3},5\times 10^{-4},5\times 10^{-5}$.}
    \label{fig:visual-cifar-5-123}
\end{figure}
\begin{figure}
\centering
    \begin{minipage}{1\linewidth}\centering
        \includegraphics[scale=0.12]{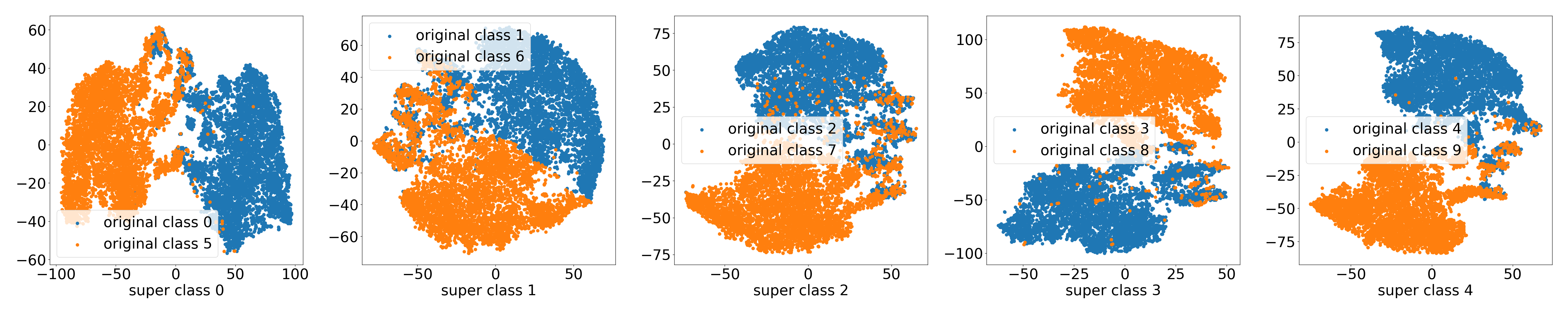}
    \end{minipage}
    \begin{minipage}{1\linewidth}\centering
        \includegraphics[scale=0.12]{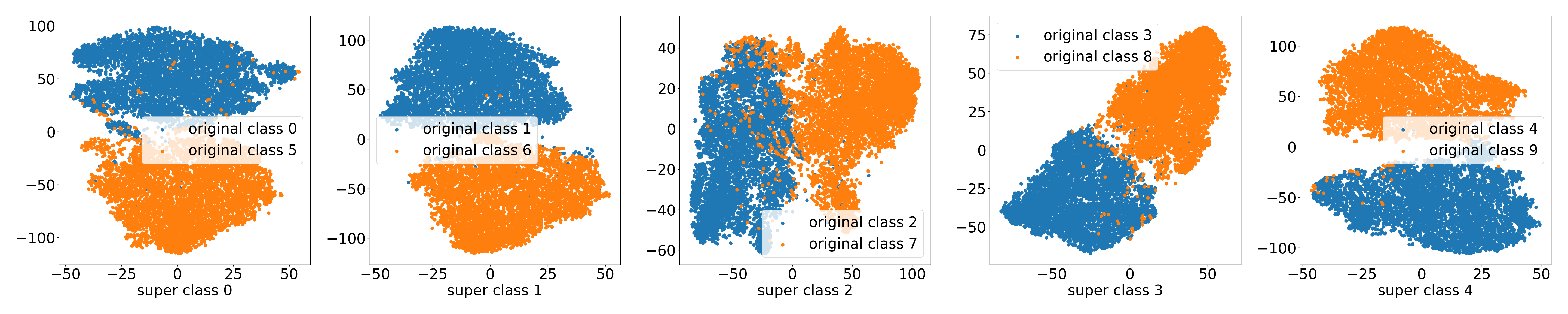}
    \end{minipage}
    \begin{minipage}{1\linewidth}\centering
        \includegraphics[scale=0.12]{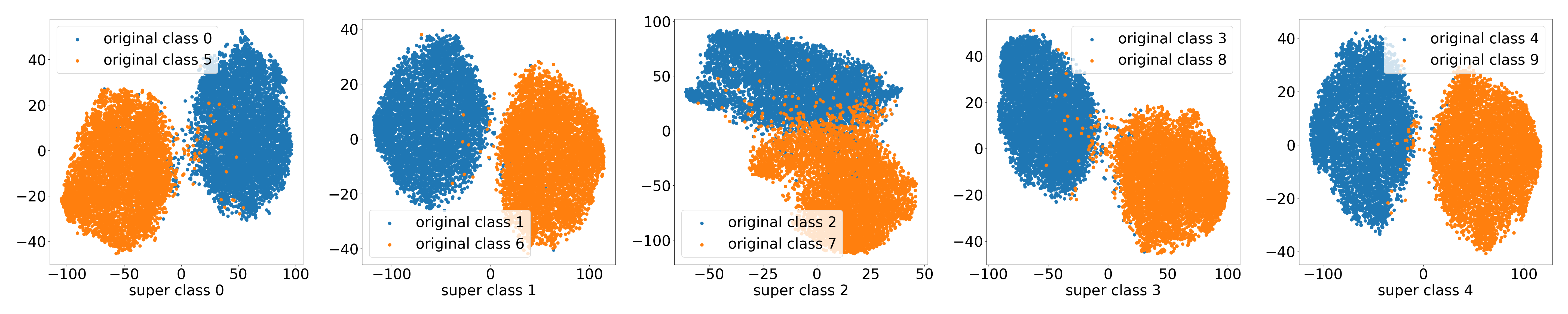}
    \end{minipage}
    \caption{Visualization of last layer representations of ResNet-18 trained on Coarse CIFAR-10 with learning rate = $0.01$. Each row represents a hyper-parameter combination and each column represents a super-class. The weight decay rates from top to bottom are $5\times 10^{-3},5\times 10^{-4},5\times 10^{-5}$.}
    \label{fig:visual-cifar-5-456}
\end{figure}
\begin{figure}
\centering
    \begin{minipage}{1\linewidth}\centering
        \includegraphics[scale=0.12]{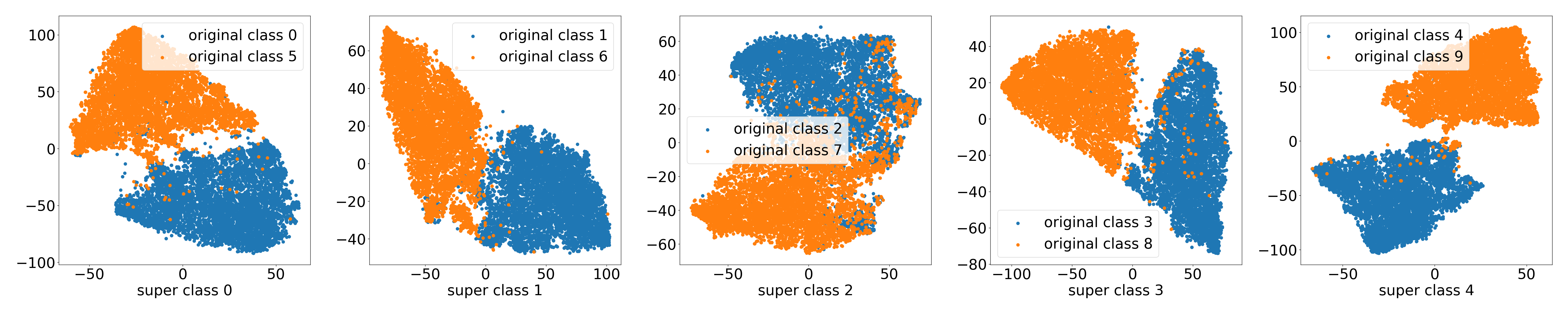}
    \end{minipage}
    \begin{minipage}{1\linewidth}\centering
        \includegraphics[scale=0.12]{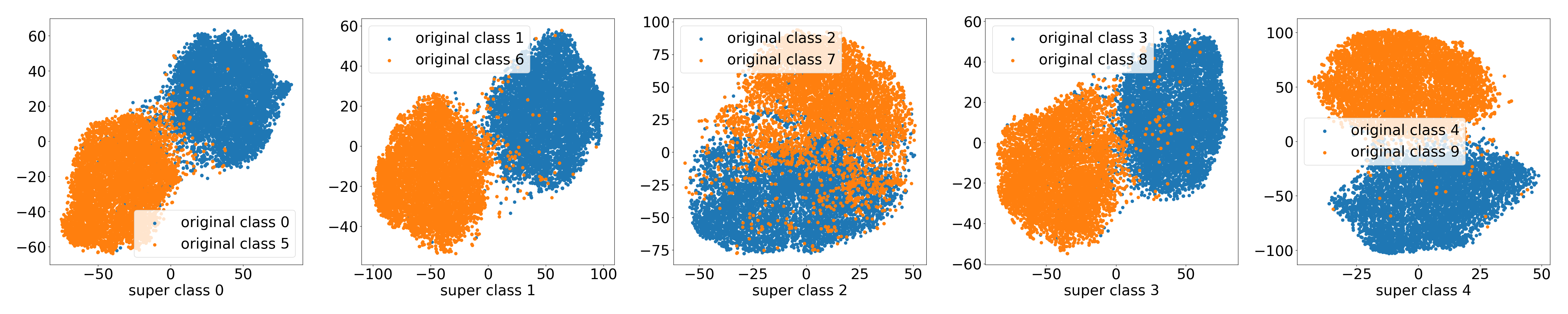}
    \end{minipage}
    \begin{minipage}{1\linewidth}\centering
        \includegraphics[scale=0.12]{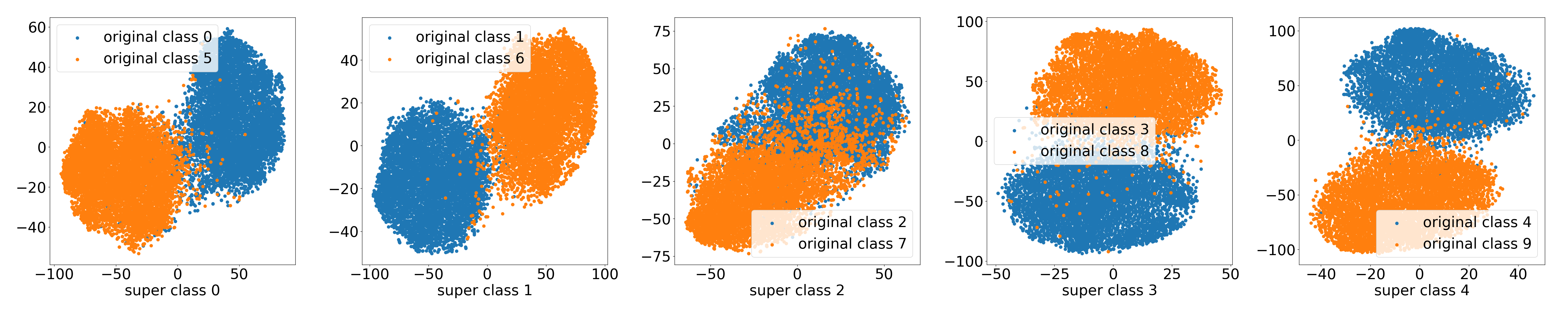}
    \end{minipage}
    \caption{Visualization of last layer representations of ResNet-18 trained on Coarse CIFAR-10 with learning rate = $0.001$. Each row represents a hyper-parameter combination and each column represents a super-class. The weight decay rates from top to bottom are $5\times 10^{-3},5\times 10^{-4},5\times 10^{-5}$.}
    \label{fig:visual-cifar-5-789}
\end{figure}

\subsection{Cluster-and-Linear-Probe}

The Cluster-and-Linear-Probe test results of ResNet-18 trained on Coarse CIFAR-10 with all hyper-parameter combinations are presented in \Cref{fig:cluster-and-train-cifar-5}. 

\begin{figure}
    \begin{minipage}{0.49\linewidth}
        \centering
        \includegraphics[scale=0.30]{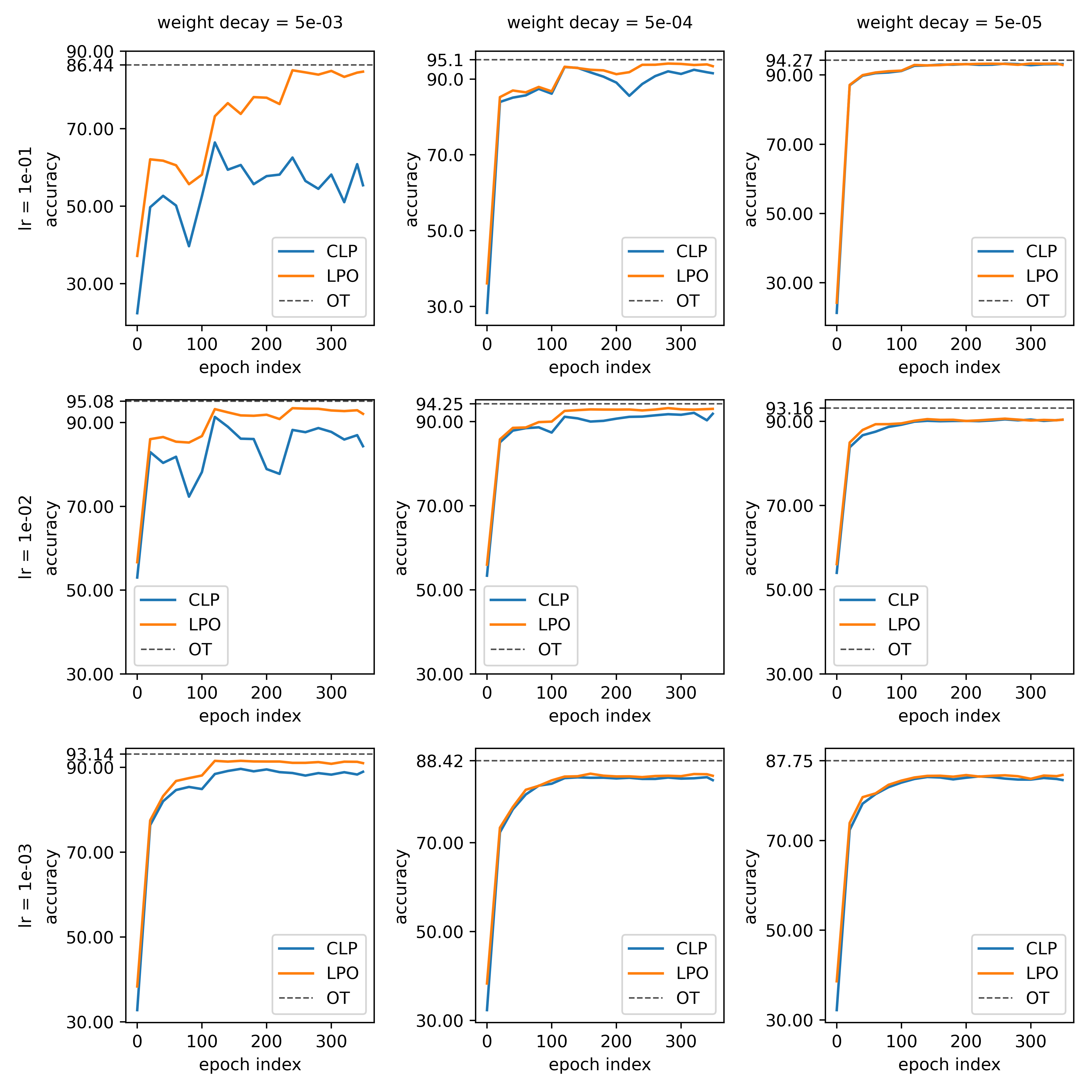}
        \caption{The result of Cluster-and-Linear-Probe test. In the figure``CLP'' refers to Cluster-and-Linear-Probe, ``LPO'' refers to linear probe with original labels and ``OT' refers to the test set accuracy of model trained on original CIFAR-10.}
        \label{fig:cluster-and-train-cifar-5}
    \end{minipage}
    \hfill
    \begin{minipage}{0.49\linewidth}
        \centering
        \includegraphics[scale=0.30]{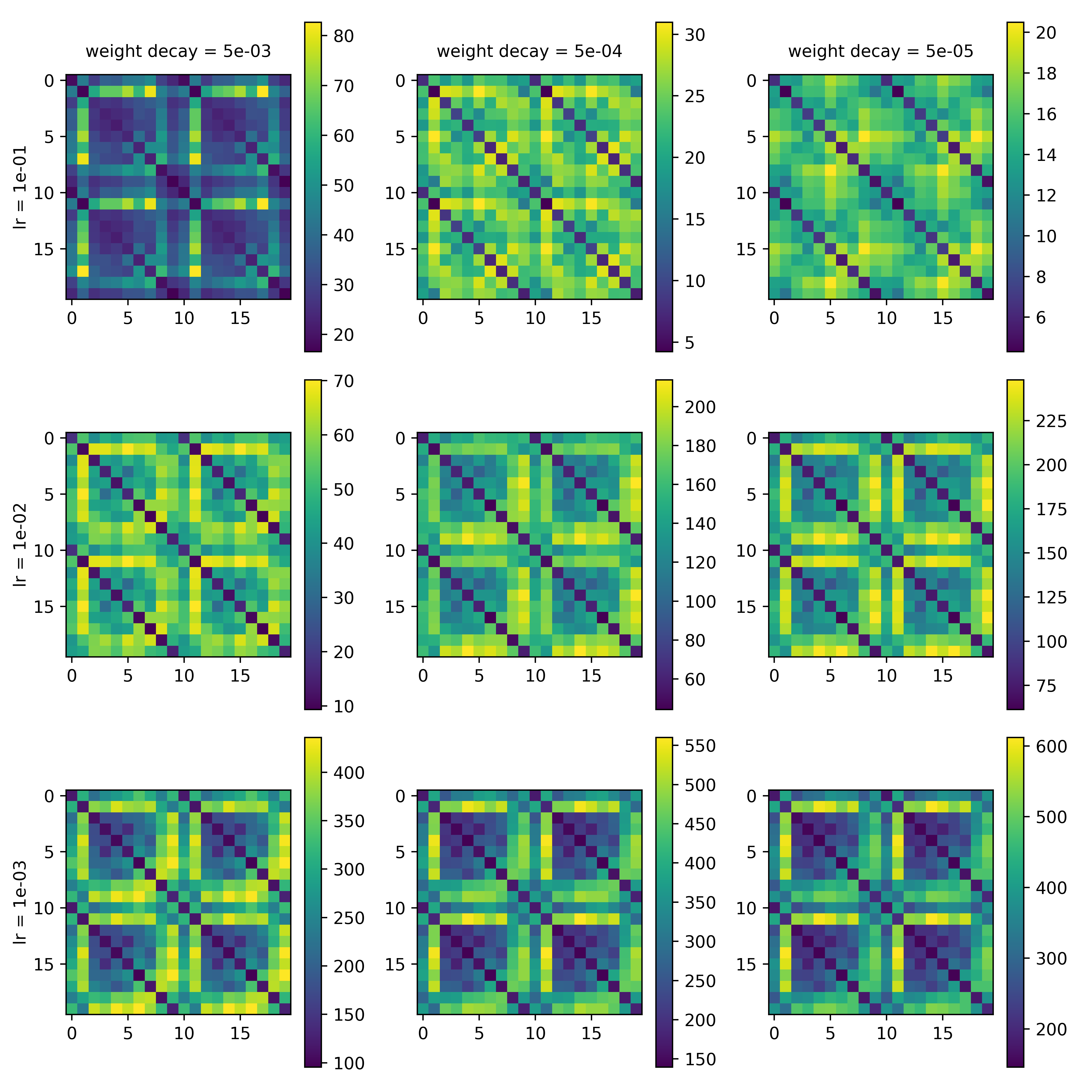}
        \caption{The heatmaps of class distance matrices of different hyper-parameter combinations on Fine CIFAR-10 at epoch $20$.}
        \label{fig:distance-cifar-20-20}
    \end{minipage}

\end{figure}

\section{Complete Class Distance Result of Fine CIFAR-10}\label{sec:complete-cifar-20}

In this section, we provide the visualization of the class distance matrix of Fine CIFAR-10 with all hyper-parameter combinations, which has been partially displayed in \Cref{sec:a-refined-dataset}. As before, multiple epochs during training are selected to display a evolutionary trend of the class distance matrices. The results are presented in \Cref{fig:distance-cifar-20-20,fig:distance-cifar-20-200,fig:distance-cifar-20-350}, whose epoch numbers are $20,200$ and $350$ respectively.

\begin{figure}
    \begin{minipage}{0.49\linewidth}
        \centering
        \includegraphics[scale=0.30]{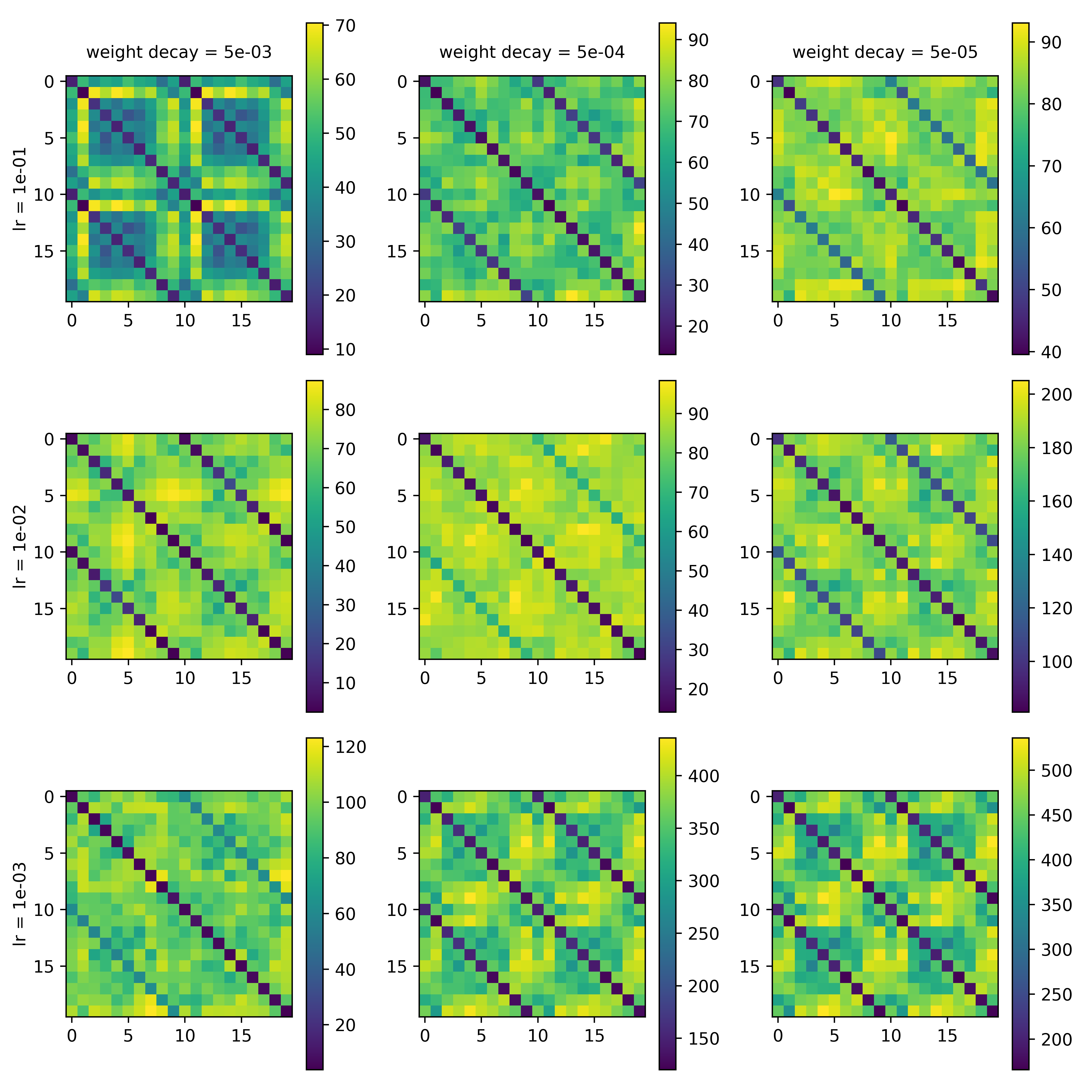}
        \caption{The heatmaps of class distance matrices of different hyper-parameter combinations on Fine CIFAR-10 at epoch $200$.}
        \label{fig:distance-cifar-20-200}
    \end{minipage}
    \hfill
    \begin{minipage}{0.49\linewidth}
        \centering
        \includegraphics[scale=0.30]{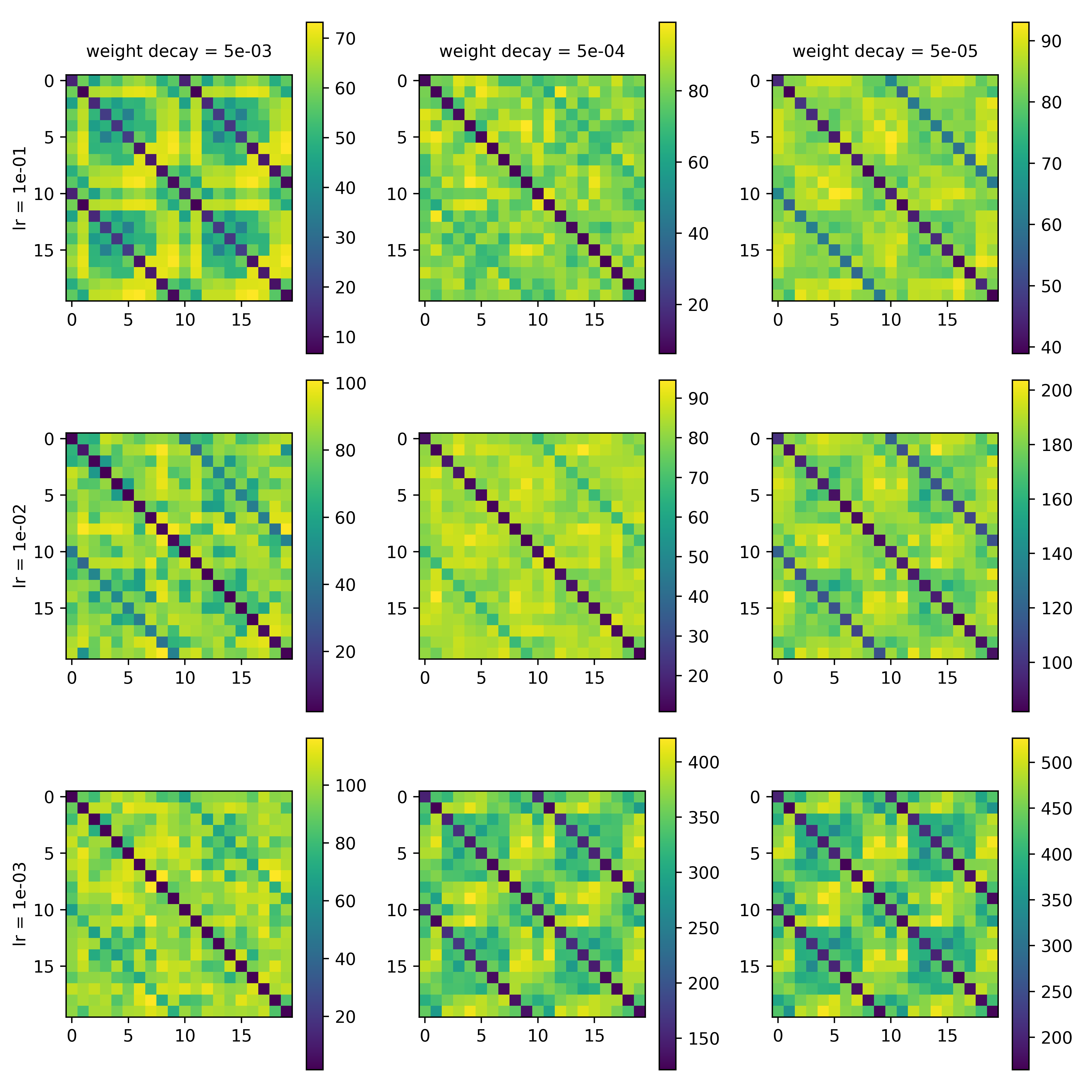}
        \caption{The heatmaps of class distance matrices of different hyper-parameter combinations on Fine CIFAR-10 at epoch $350$.}
        \label{fig:distance-cifar-20-350}
    \end{minipage}
\end{figure}

\section{Experiment of ResNet-18 on Coarse CIFAR-100}\label{sec:complete-rand-cifar-100}

In this section, we report extended experiment results on Coarse CIFAR-100. The same with the case of Coarse CIFAR-10, we construct CIFAR-100 through the label coarsening process described in \Cref{sec:experiment-setup} and choose $\tilde C = 20$, so that every $5$ original classes are merged into one super-class. We repeat most of experiments in \Cref{sec:experiment-on-coarse-cifar-10}.

\subsection{Class Distance}

The heatmaps of distance matrices are presented in \Cref{fig:distance-cifar-100-20,fig:distance-cifar-100-200,fig:distance-cifar-100-350}, whose epoch number are $20,200$ and $350$ respectively.

\begin{figure}
    \begin{minipage}{0.49\linewidth}
        \centering
        \includegraphics[scale=0.30]{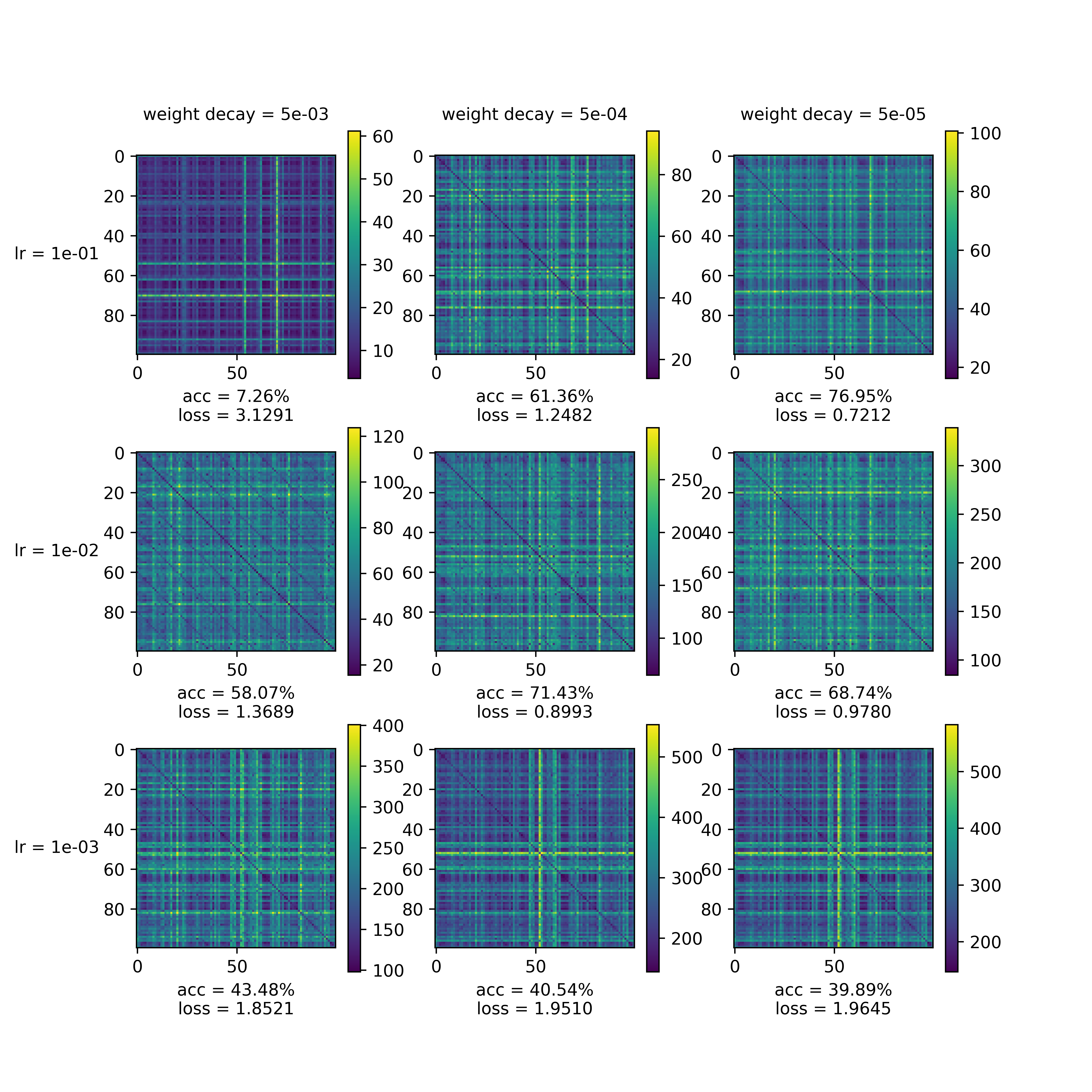}
        \caption{The heatmaps of class distance matrices of different hyper-parameter combinations on Coarse CIFAR-100 at epoch $20$.}
        \label{fig:distance-cifar-100-20}
    \end{minipage}
    \hfill
    \begin{minipage}{0.49\linewidth}
        \centering
        \includegraphics[scale=0.30]{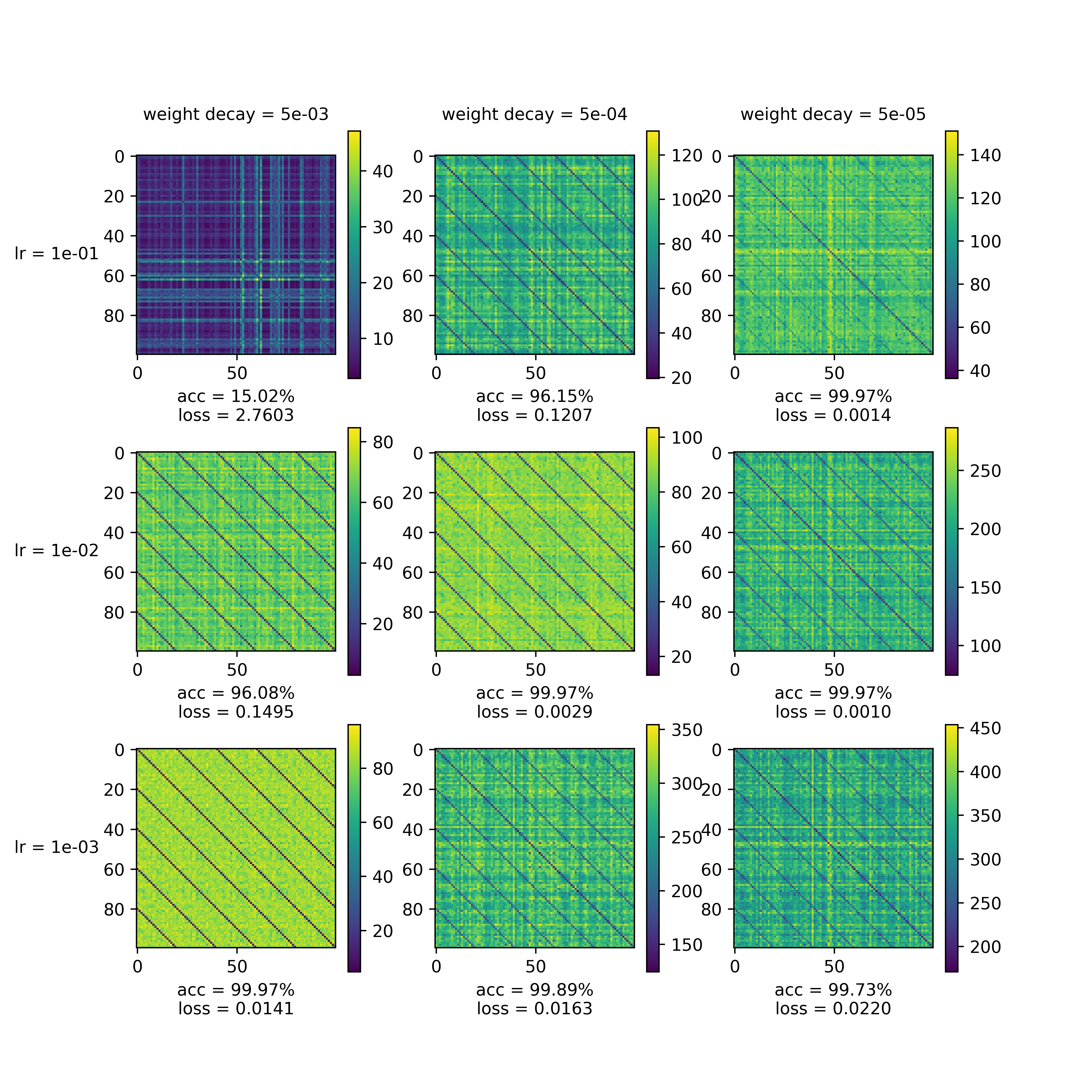}
        \caption{The heatmaps of class distance matrices of different hyper-parameter combinations on Coarse CIFAR-100 at epoch $200$.}
        \label{fig:distance-cifar-100-200}
    \end{minipage}
\end{figure}

\subsection{Visualization}

In this section, we present the t-SNE visualization result of ResNet-18 on Coarse CIFAR-100. We only put the result for learning rate = 0.01 and weight decay = 1e-4 here as a demonstration in \Cref{fig:visual-cifar-100-1}.

\begin{figure}
    \centering
    \includegraphics[scale=0.15]{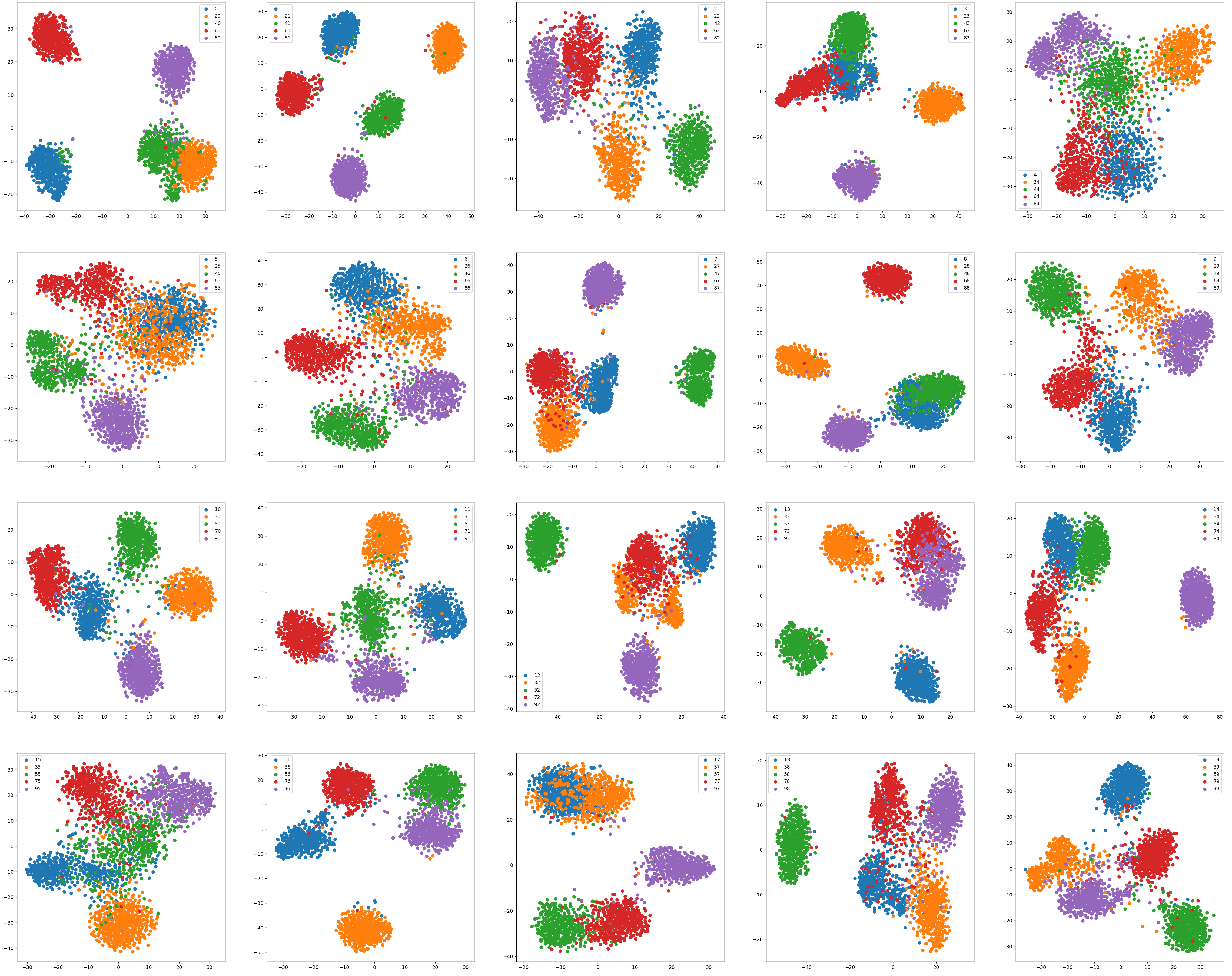}
    \caption{Visualization of last layer representations of ResNet-18 trained on Coarse CIFAR-100. Each grid represents a super-class and each color in a grid represents a sub-class.}
    \label{fig:visual-cifar-100-1}
\end{figure}

\subsection{Cluster-and-Linear-Probe}

Notice that we omit the visualization result of Coarse CIFAR-100 since there are too many figures. We present the Cluster-and-Linear-Probe results to reflect the clustering property of last-layer representations learned on Coarse CIFAR-100. The CLP results are presented in \Cref{fig:cluster-and-train-cifar-100-rand}.

\begin{figure}
    \begin{minipage}{0.49\linewidth}
        \centering
        \includegraphics[scale=0.30]{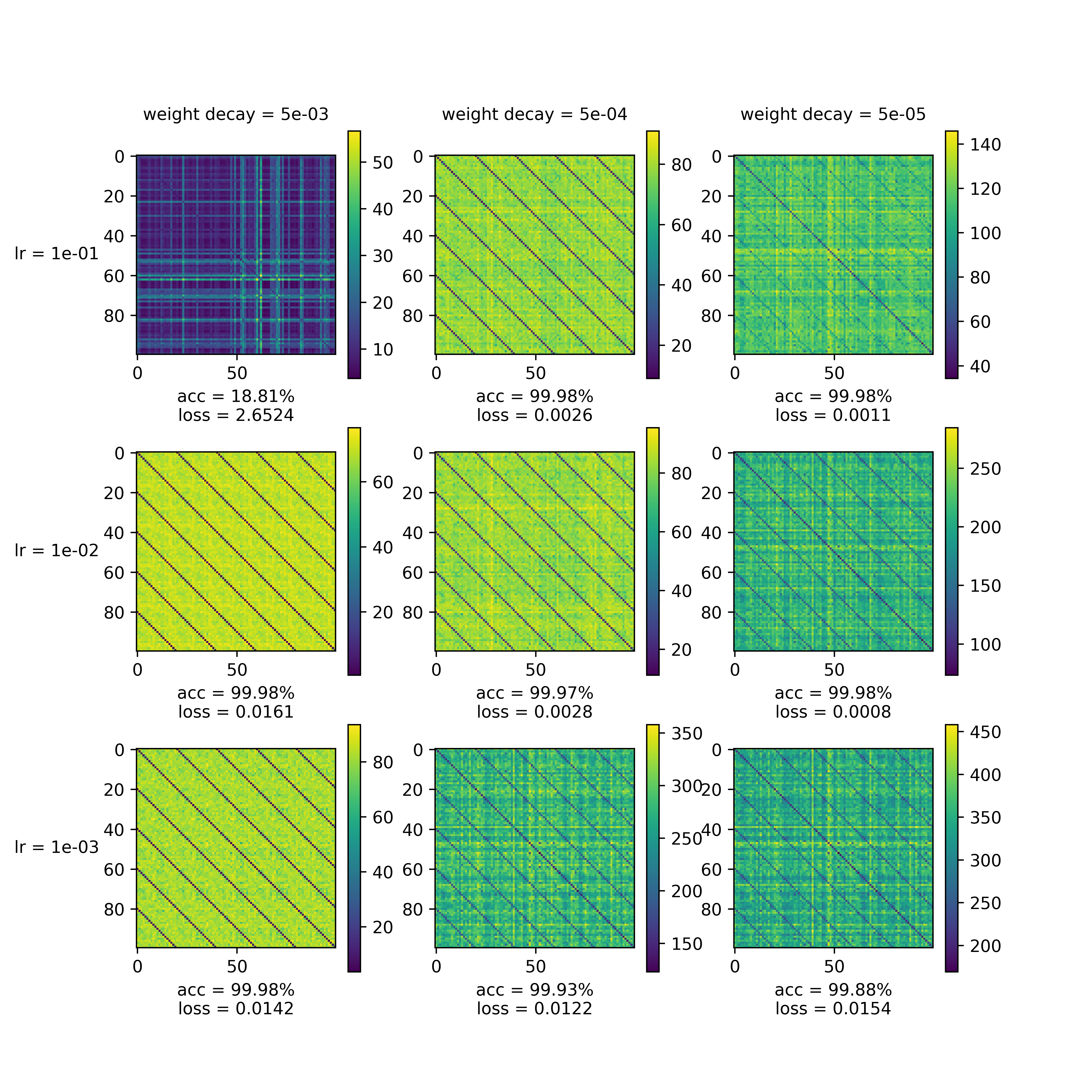}
        \caption{The heatmaps of class distance matrices of different hyper-parameter combinations on Coarse CIFAR-100 at epoch $350$.}
        \label{fig:distance-cifar-100-350}
    \end{minipage}
    \hfill
    \begin{minipage}{0.49\linewidth}
        \centering
        \includegraphics[scale=0.30]{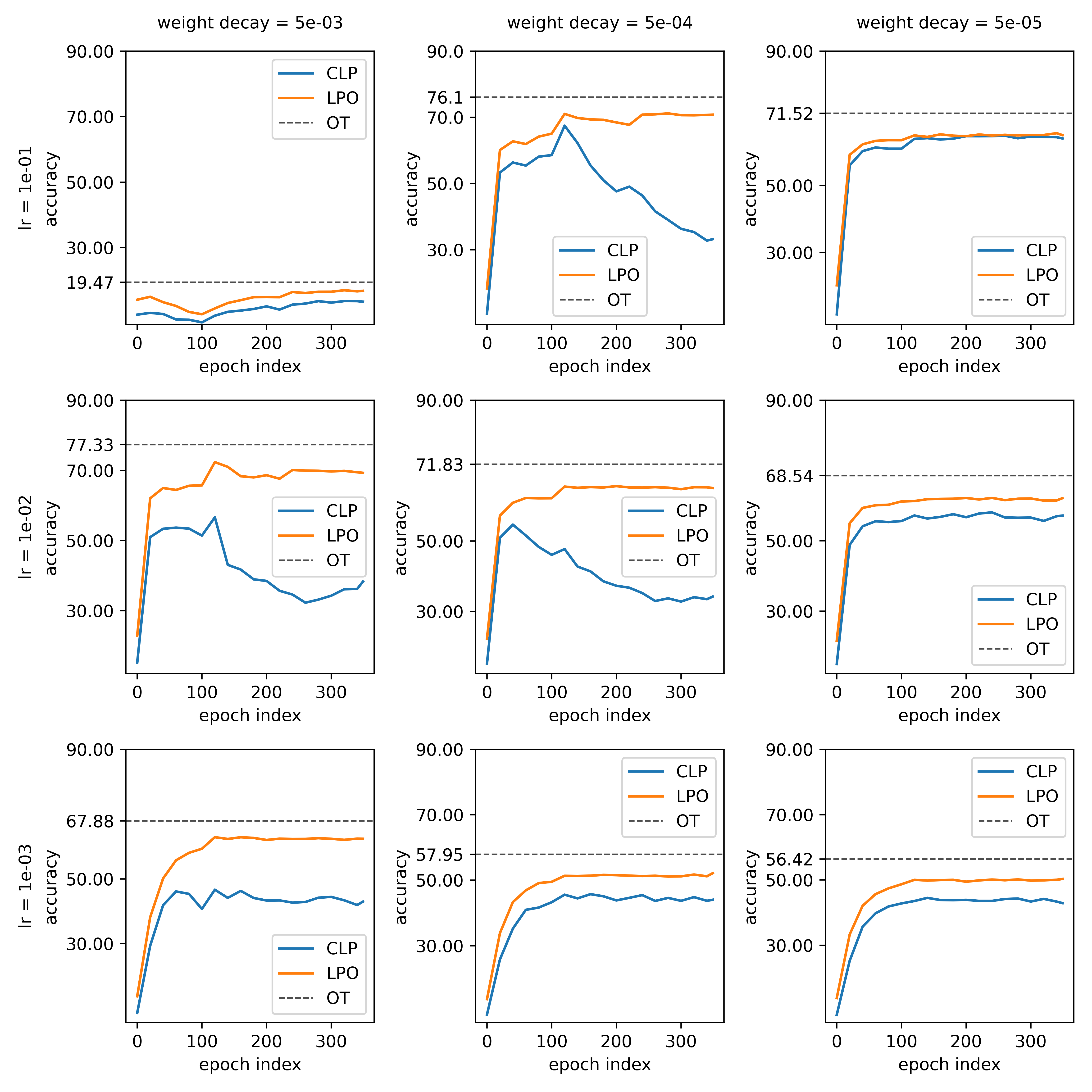}
        \caption{The result of Cluster-and-Linear-Probe test. In the figure``CLP'' refers to Cluster-and-Linear-Probe, ``LPO'' refers to linear probe with original labels and ``OT' refers to the test set accuracy of model trained on original CIFAR-100.}
        \label{fig:cluster-and-train-cifar-100-rand}
    \end{minipage}
\end{figure}

\section{Random Coarse CIFAR-10}\label{sec:complete-rand-cifar-10}

In this section, as mentioned in \Cref{sec:experiment-setup}, we make our experiment more complete by performing a random combination of labels on CIFAR-10 rather than using a determined coarsening process as in the main paper. The dataset construction is almost the same as the process of assigning coarse labels described in \Cref{sec:experiment-setup}, except here we randomly shuffle the class indices before coarsening them. 

The class distance matrices of three difference epochs are shown in \Cref{fig:distance-cifar-5-rand-20,fig:distance-cifar-5-rand-200,fig:distance-cifar-5-rand-350}. From the results we can see, although there are no longer three dark lines, for each row there are generally two dark blocks, represents the original classes belongs to the same super-class, and the same observations in \Cref{sec:experiment-on-coarse-cifar-10} can still be made here.

\begin{figure}
    \begin{minipage}{0.49\linewidth}
        \centering
        \includegraphics[scale=0.30]{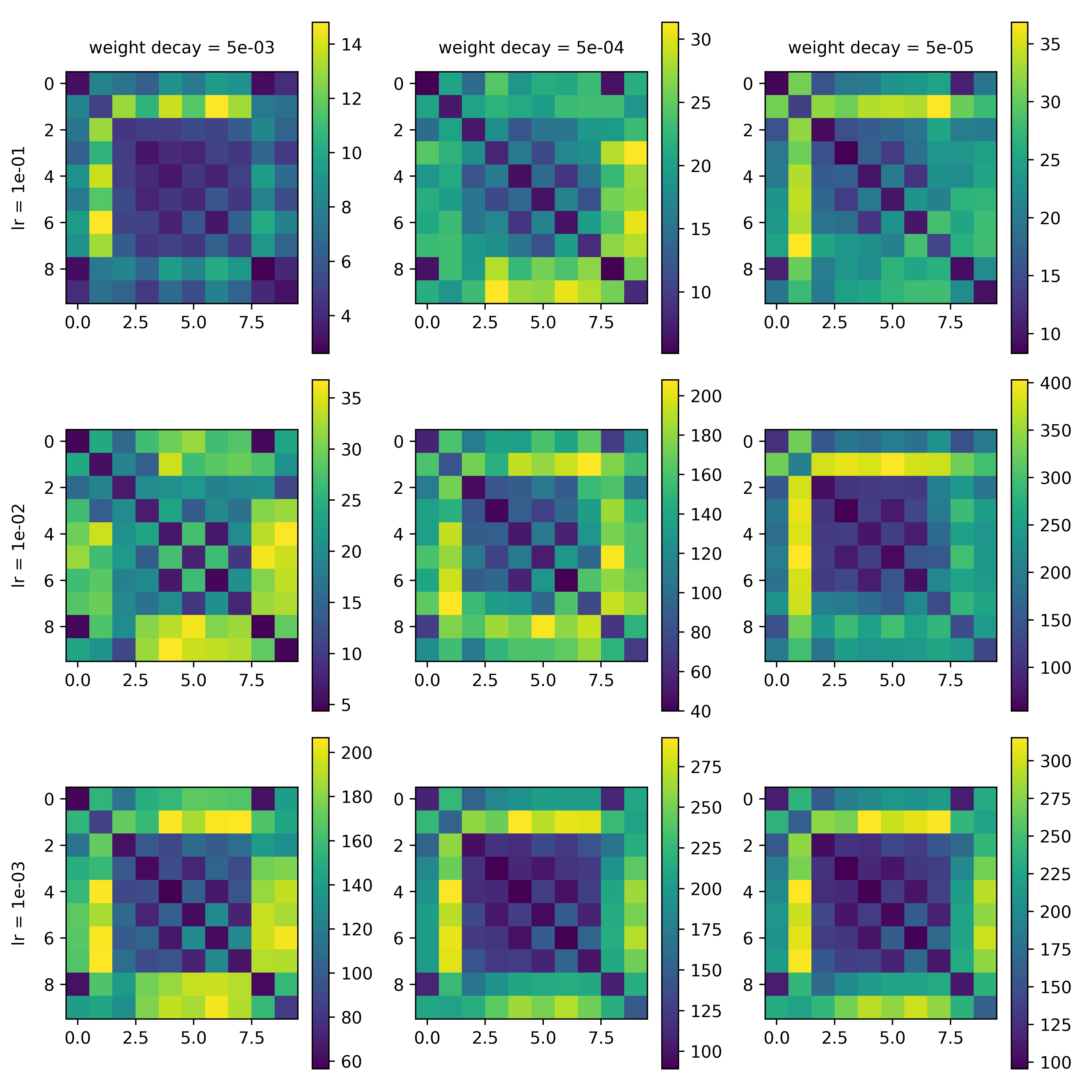}
        \caption{The heatmaps of class distance matrices of different hyper-parameter combinations on Random Coarse CIFAR-10 at epoch $20$.}
        \label{fig:distance-cifar-5-rand-20}
    \end{minipage}
    \hfill
    \begin{minipage}{0.49\linewidth}
        \centering
        \includegraphics[scale=0.30]{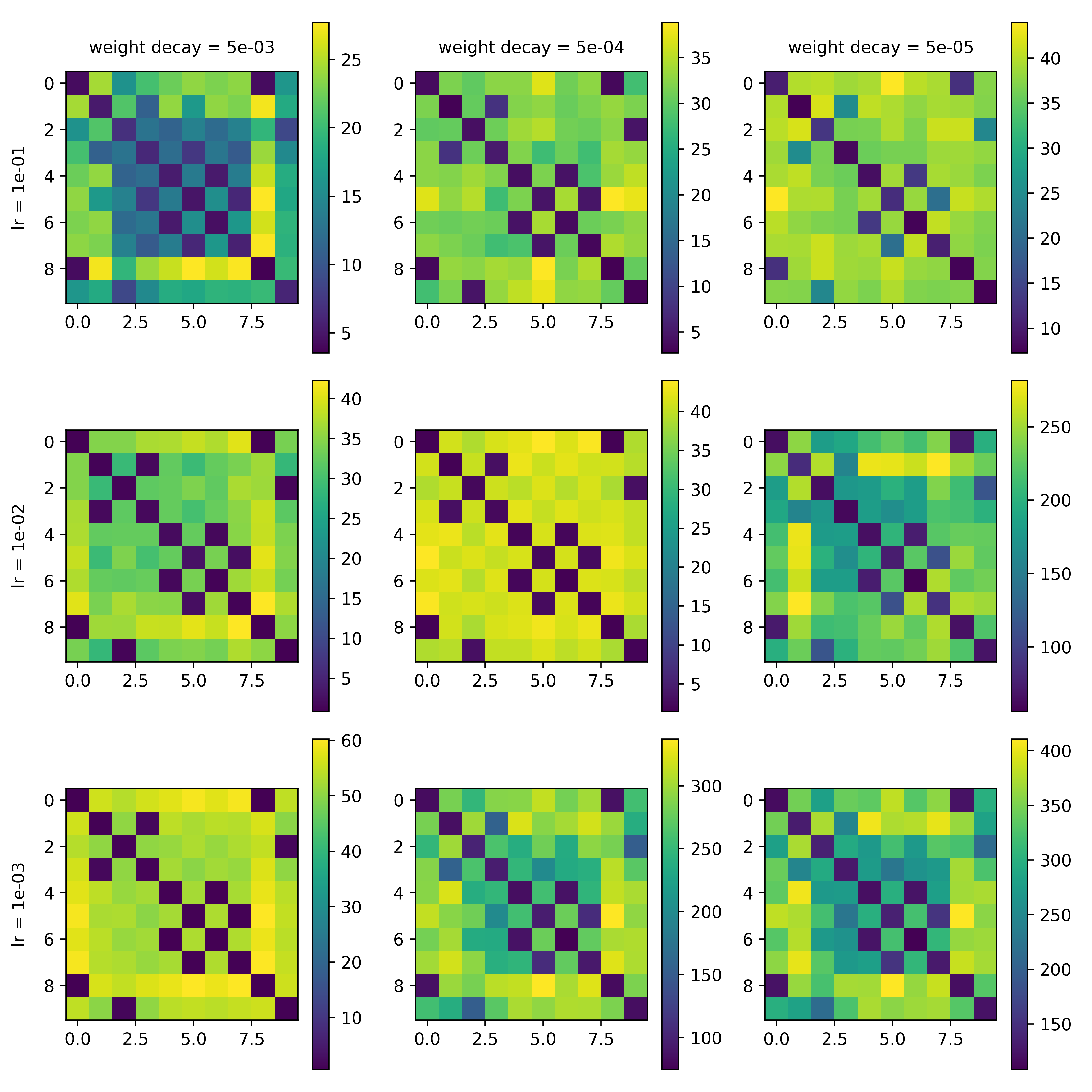}
        \caption{The heatmaps of class distance matrices of different hyper-parameter combinations on Random Coarse CIFAR-10 at epoch $200$.}
        \label{fig:distance-cifar-5-rand-200}
    \end{minipage}
\end{figure}

\begin{figure}
    \begin{minipage}{0.49\linewidth}
        \centering
        \includegraphics[scale=0.30]{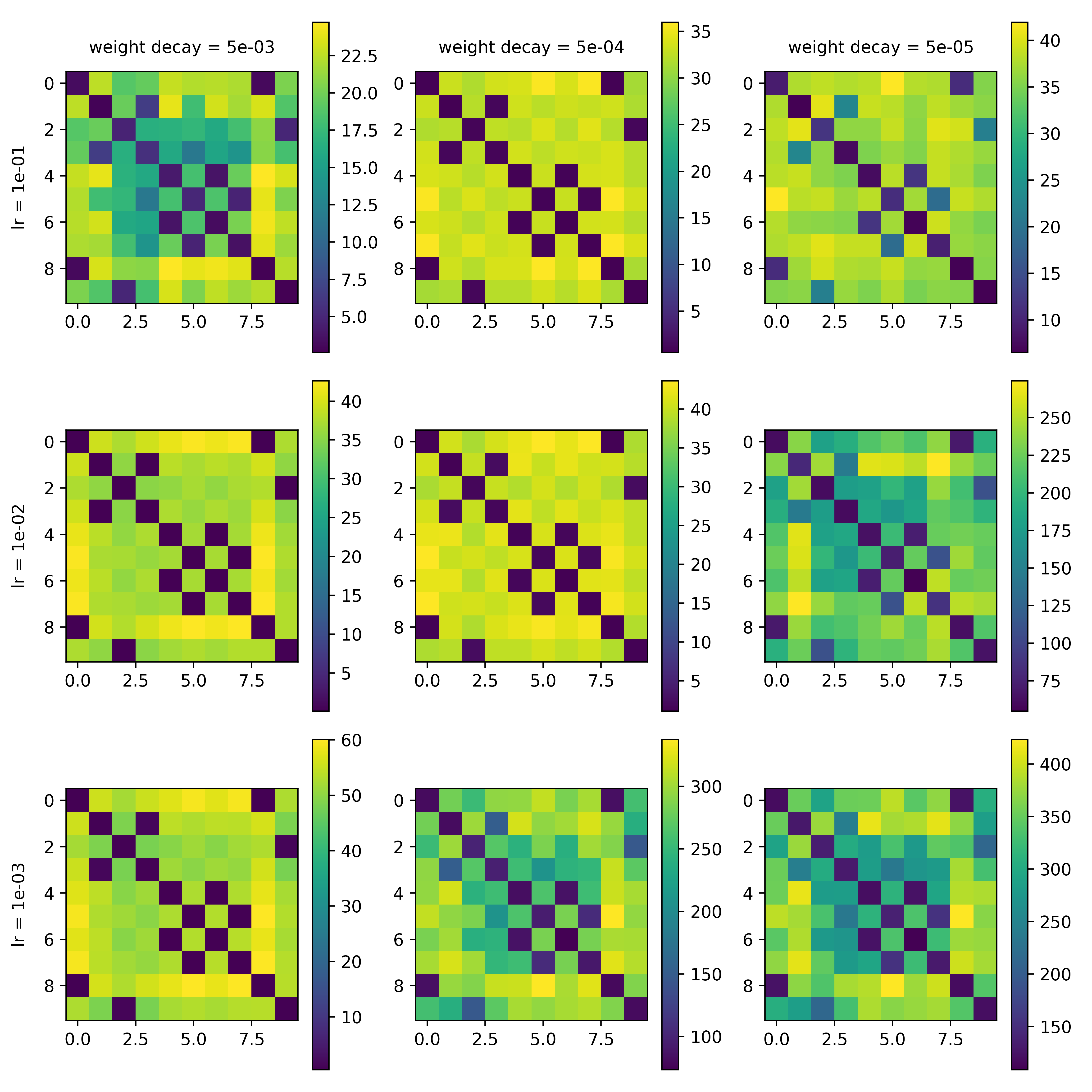}
        \caption{The heatmaps of class distance matrices of different hyper-parameter combinations on Random Coarse CIFAR-10 at epoch $350$.}
        \label{fig:distance-cifar-5-rand-350}
    \end{minipage}
    \hfill
    \begin{minipage}{0.49\linewidth}
        \centering
        \includegraphics[scale=0.30]{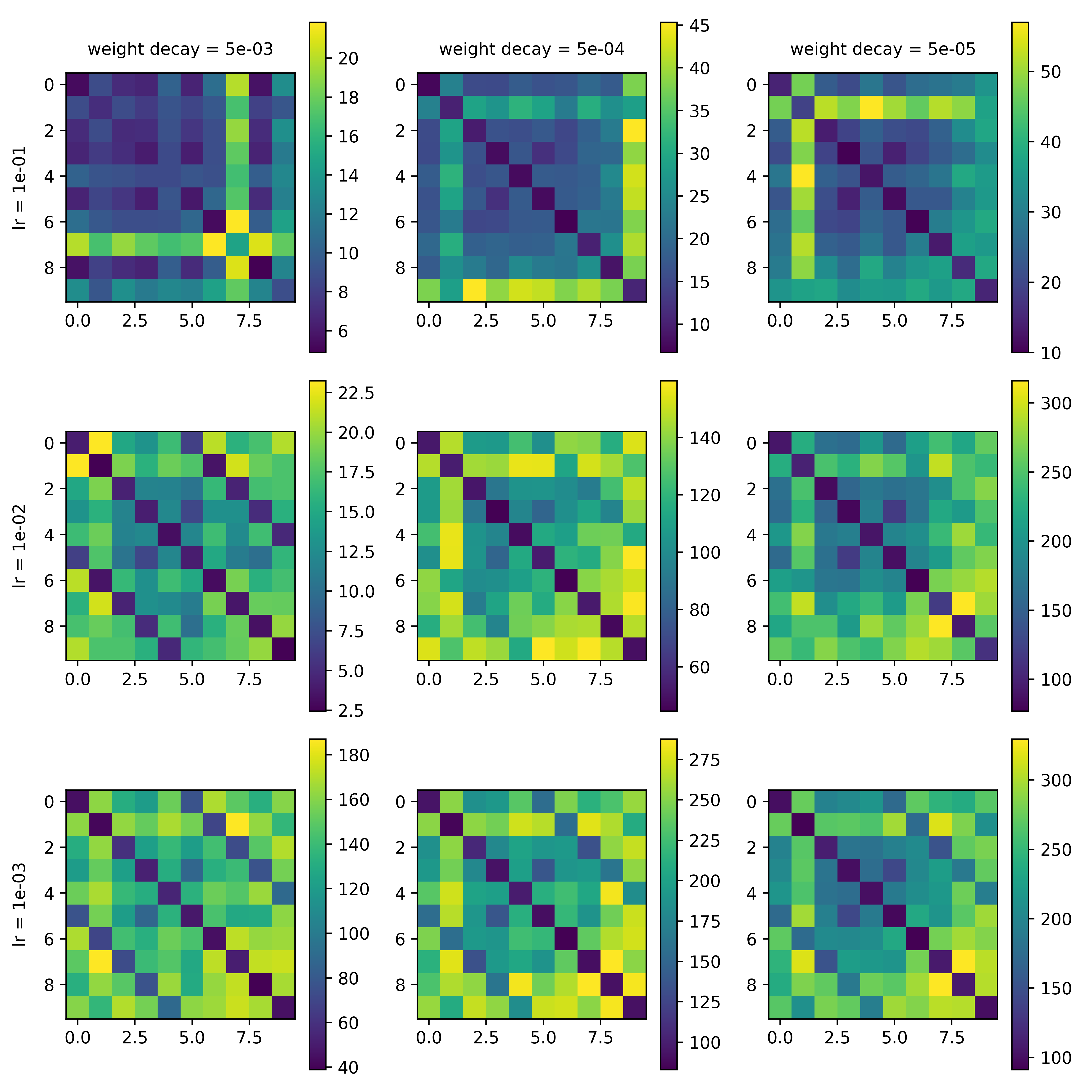}
        \caption{The heatmaps of class distance matrices of different hyper-parameter combinations with DenseNet-121 on Coarse CIFAR-10 at epoch $20$.}
        \label{fig:distance-cifar-5-densenet-20}
    \end{minipage}
\end{figure}

\section{Experiment with DenseNet}\label{sec:densenet}

We also perform our experiments with different neural network structures for completeness. In this section, we show the result with DenseNet-121 on Coarse CIFAR-10. The experiments with DenseNet is supportive to our observations in the main paper.

\subsection{Class Distance}

The class distance matrices of three epochs during training are presented in \Cref{fig:distance-cifar-5-densenet-20,fig:distance-cifar-5-densenet-200,fig:distance-cifar-5-densenet-350}.

\begin{figure}
    \begin{minipage}{0.49\linewidth}
        \centering
        \includegraphics[scale=0.30]{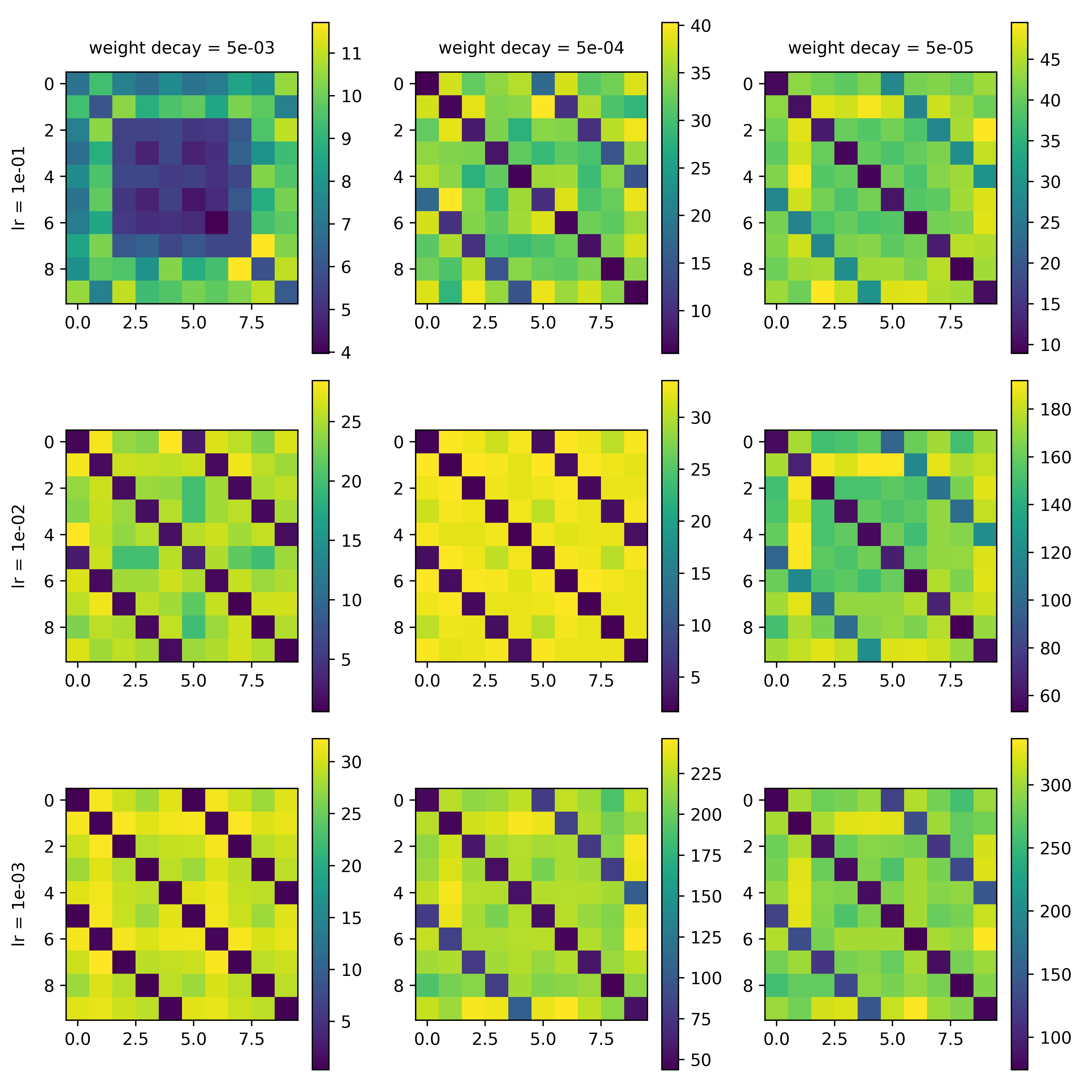}
        \caption{The heatmaps of class distance matrices of different hyper-parameter combinations with DenseNet-121 on Coarse CIFAR-10 at epoch $20$.}
        \label{fig:distance-cifar-5-densenet-200}
    \end{minipage}
    \hfill
    \begin{minipage}{0.49\linewidth}
        \centering
        \includegraphics[scale=0.30]{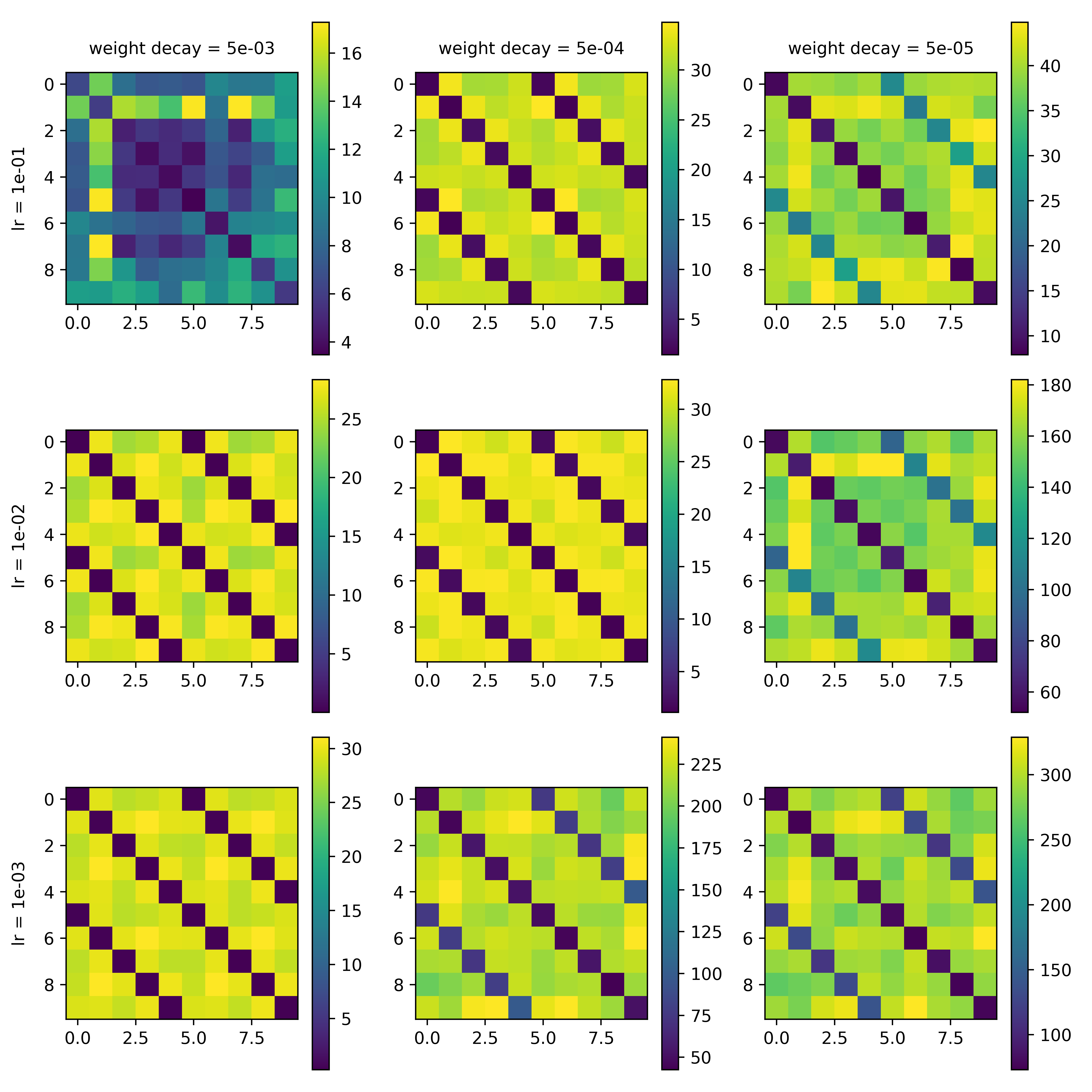}
        \caption{The heatmaps of class distance matrices of different hyper-parameter combinations with DenseNet-121 on Coarse CIFAR-10 at epoch $20$.}
        \label{fig:distance-cifar-5-densenet-350}
    \end{minipage}
\end{figure}

\subsection{Cluster-and-Linear-Probe}

The Cluster-and-Linear-Probe test results are presented in \Cref{fig:cluster-and-train-cifar-5-densenet}.

\begin{figure}
    \begin{minipage}{0.49\linewidth}
        \centering
        \includegraphics[scale=0.30]{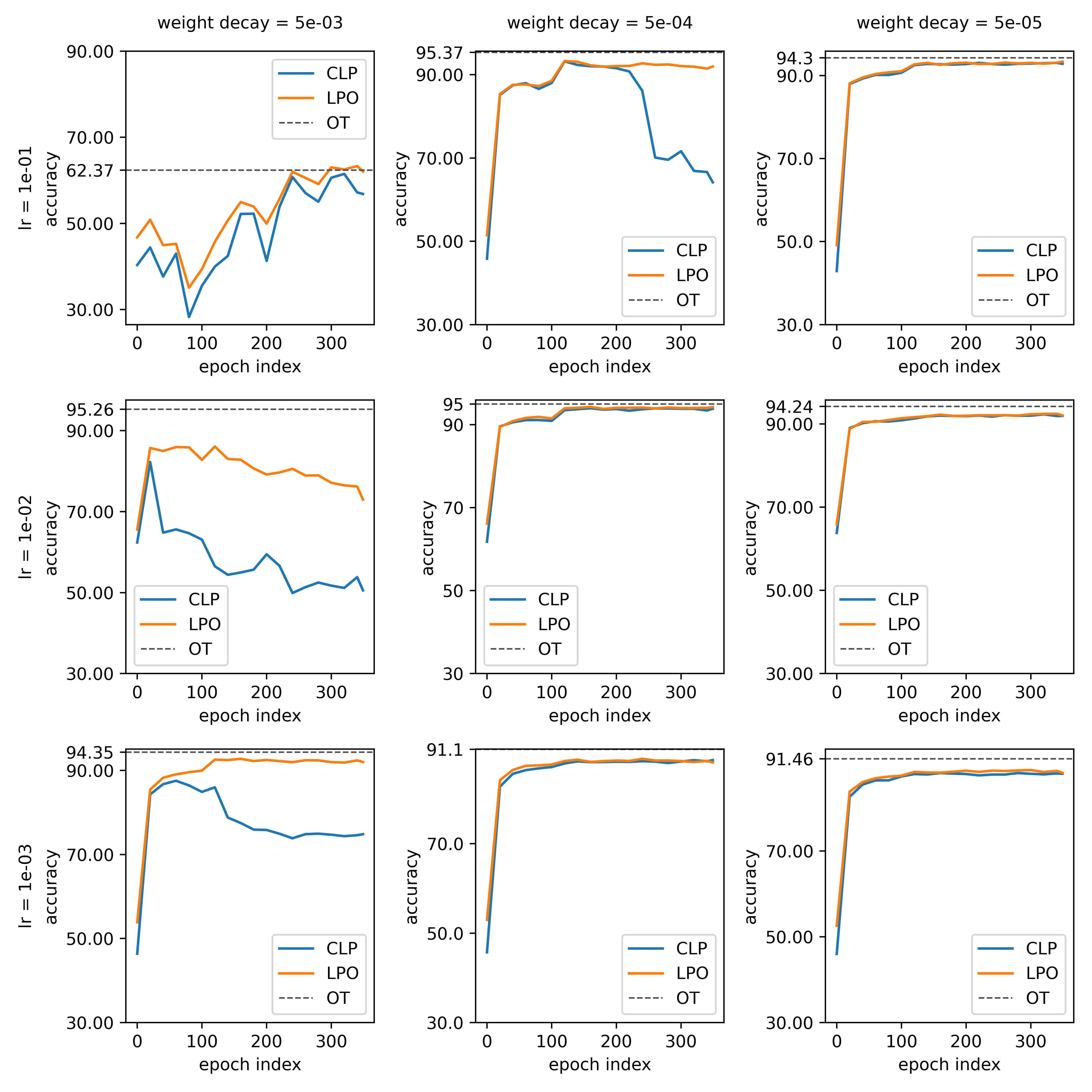}
        \caption{The result of Cluster-and-Linear-Probe test with DenseNet-121. In the figure``CLP'' refers to Cluster-and-Linear-Probe, ``LPO'' refers to linear probe with original labels and ``OT' refers to the test set accuracy of model trained on original CIFAR-10.}
        \label{fig:cluster-and-train-cifar-5-densenet}
    \end{minipage}
    \hfill
    \begin{minipage}{0.49\linewidth}
        \centering
        \includegraphics[scale=0.30]{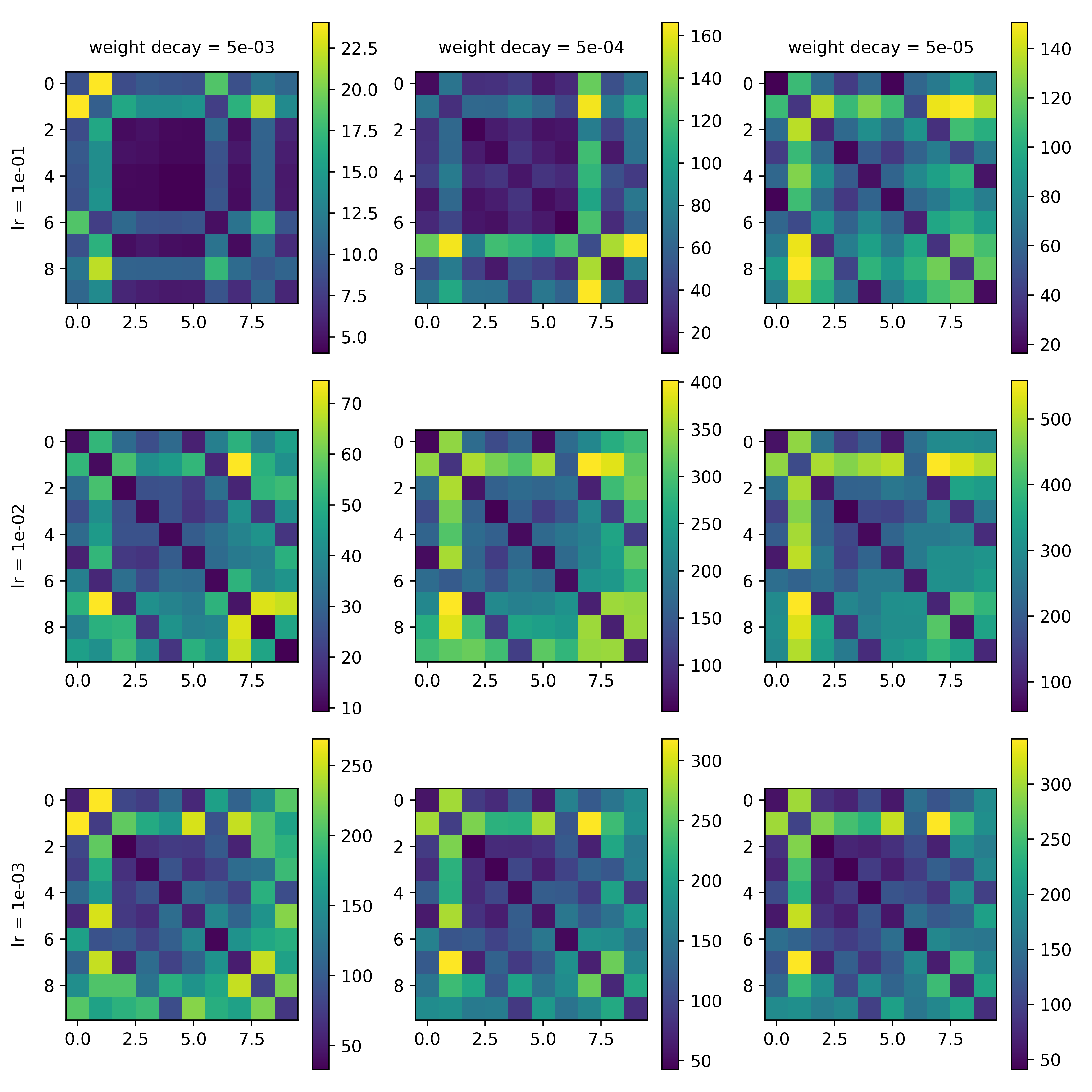}
        \caption{The heatmaps of class distance matrices of different hyper-parameter combinations with VGG-18 on Coarse CIFAR-10 at epoch $20$.}
        \label{fig:distance-cifar-5-vgg-20}
    \end{minipage}
\end{figure}

\section{Experiment with VGG}\label{sec:vgg}

We also extend our experiments to VGG-18. Interestingly, VGG to some extent is a counter example of the observations made in the main paper: it only displays Neural Collapse, and can not distinguish different original classes within one super-class, even in an early stage of training. The reason why VGG is abnormal requires further exploration.

The class distance matrices of three epochs with VGG during training are shown in \Cref{fig:distance-cifar-5-vgg-20,fig:distance-cifar-5-vgg-200,fig:distance-cifar-5-vgg-350}. It can be observed that the three dark lines appears almost at the same time and always be of nearly the same darkness. This represents the trend predicted by Neural Collapse (\Cref{fig:distance-intro} (a)), but rejects the prediction made by (\Cref{fig:distance-intro} (b)).

\begin{figure}
    \begin{minipage}{0.49\linewidth}
        \centering
        \includegraphics[scale=0.30]{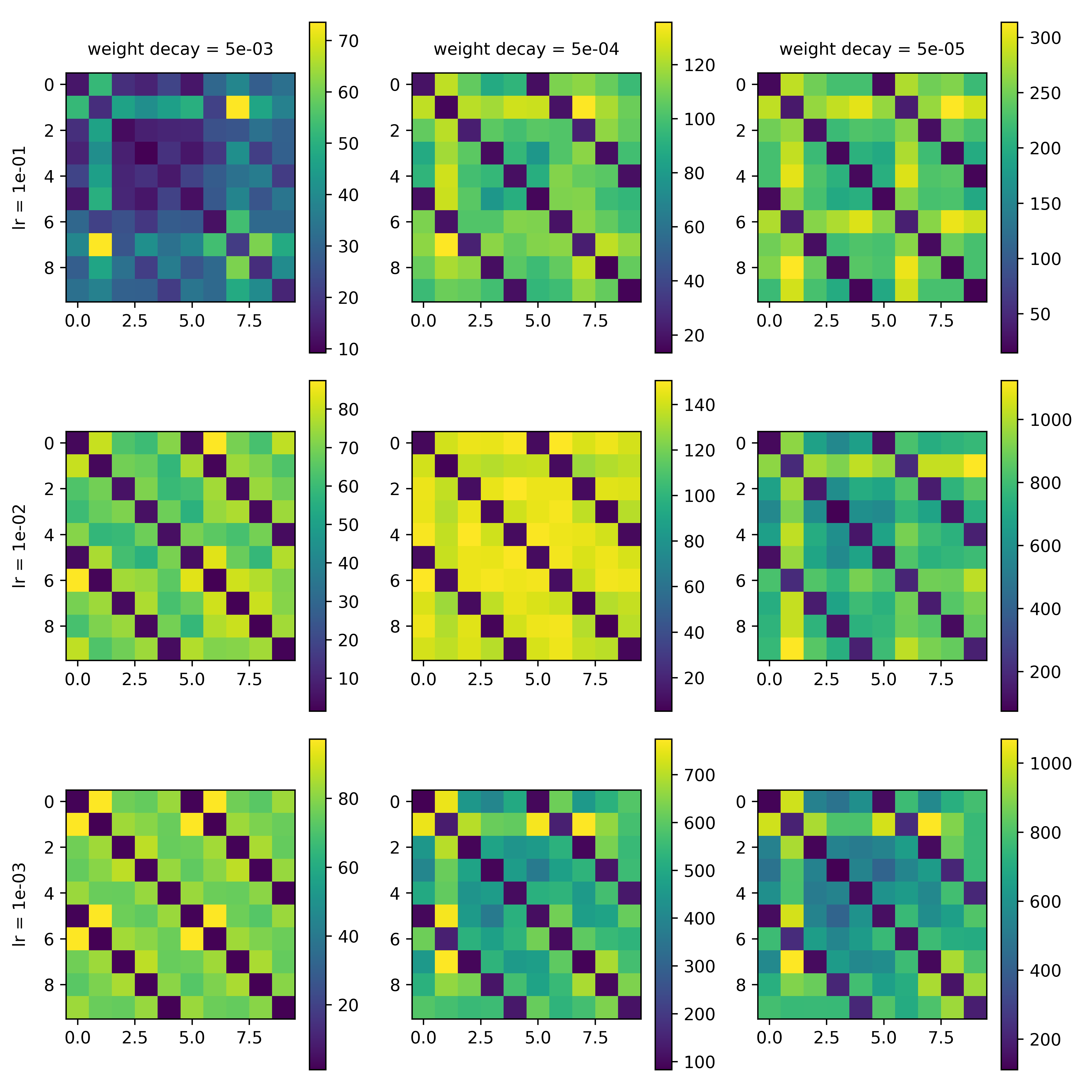}
        \caption{The heatmaps of class distance matrices of different hyper-parameter combinations with VGG-18 on Coarse CIFAR-10 at epoch $200$.}
        \label{fig:distance-cifar-5-vgg-200}
    \end{minipage}
    \hfill
    \begin{minipage}{0.49\linewidth}
        \centering
        \includegraphics[scale=0.30]{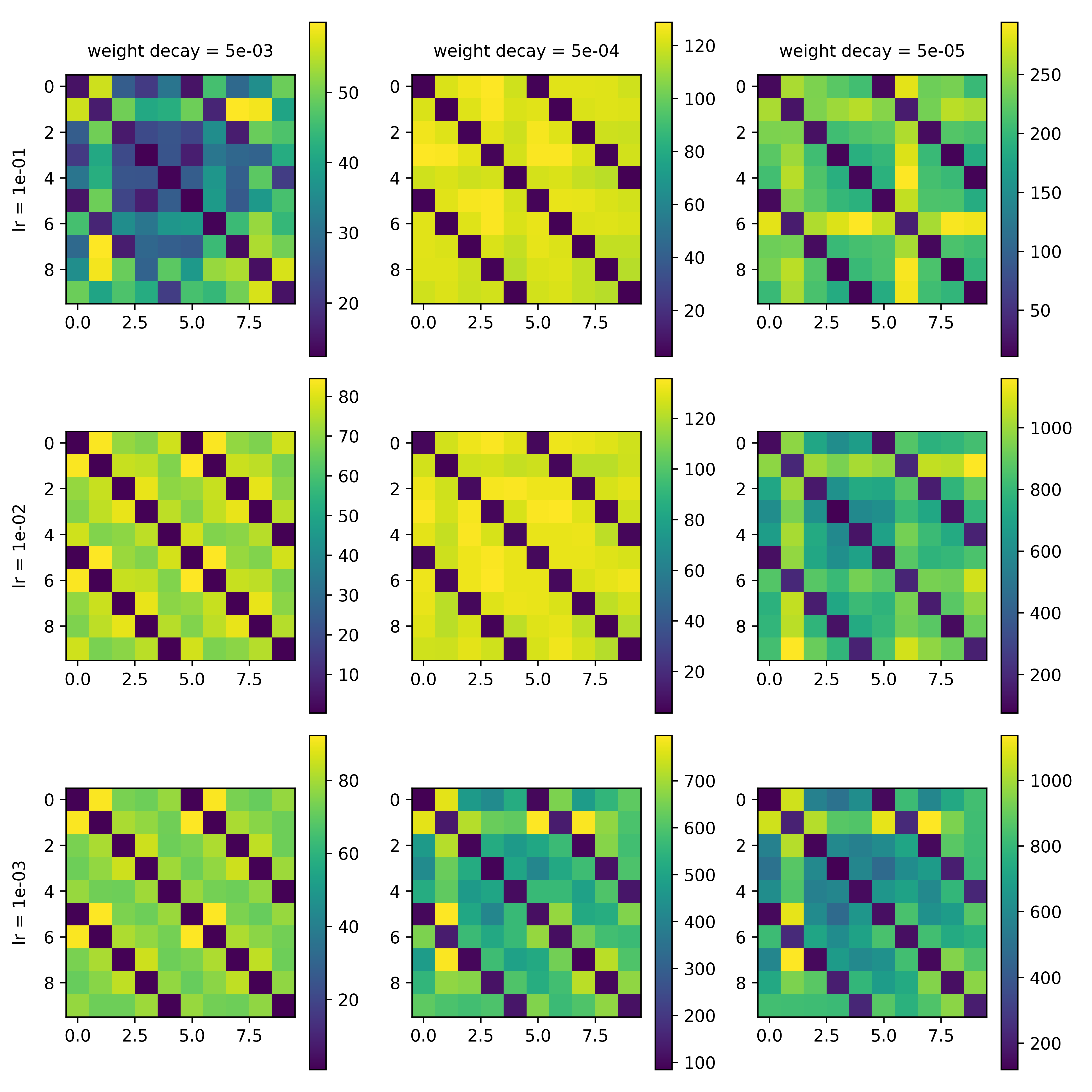}
        \caption{The heatmaps of class distance matrices of different hyper-parameter combinations with VGG-18 on Coarse CIFAR-10 at epoch $350$.}
        \label{fig:distance-cifar-5-vgg-350}
    \end{minipage}
\end{figure}

\end{document}